%% file: iclr2025_hcfl.tex
\documentclass{article}

\usepackage{iclr2025_conference,times}
\definecolor{darkblue}{rgb}{0.0, 0.0, 0.55}
\definecolor{myblue}{rgb}{0,0.45,0.74}
\definecolor{myred}{rgb}{0.85,0.33,0.1}
\usepackage[hyperindex,breaklinks,colorlinks,citecolor=darkblue]{hyperref} 

\usepackage{stfloats}
\usepackage[utf8]{inputenc}
\usepackage[T1]{fontenc}
\usepackage{booktabs}
\usepackage{amsfonts}       
\usepackage{nicefrac}       
\usepackage{microtype}      
\usepackage{xcolor}
\usepackage[flushleft]{threeparttable}
\usepackage{color}
\usepackage{mathtools}
\usepackage{transparent}

\usepackage{url}
\usepackage{float}
\usepackage{graphicx}
\usepackage{multirow}
\usepackage{xspace}
\usepackage{enumitem}
\usepackage[font=small]{caption}
\usepackage{diagbox}
\usepackage{wrapfig}
\usepackage[toc, page, header]{appendix}
\usepackage{subcaption}

\input{configuration/lins_macros}
\usepackage{amssymb}
\usepackage{pifont}

\input{configuration/lins_algo}

\usepackage{xspace}  
\usepackage{bm}

\newcommand{\cpA}{\textrm{ASCP}\xspace}
\newcommand{\cpB}{\textrm{CSCP}\xspace}
\newcommand{\tierB}{Cluster Weights Calculation\xspace}
\newcommand{\tierA}{Cluster Formulations\xspace}
\newcommand{\algfednew}{\textsc{HCFL}$^{+}$\xspace} 
\newcommand{\algframework}{\textsc{HCFL}\xspace}

\newcommand{\Rmnum}[1]{\uppercase\expandafter{\romannumeral #1}}

\title{Enhancing Clustered Federated Learning: Integration of Strategies and Improved Methodologies}

\author{Yongxin Guo$^{1}${\quad\,}Xiaoying Tang$^{1,2,3,}$\thanks{Corresponding author.}{\quad\,}Tao Lin$^{4,5}$\\
$^1$School of Science and Engineering, The Chinese University of Hong Kong, Shenzhen 518172, China\\
$^2$Shenzhen Institute of Artificial Intelligence and Robotics for Society (AIRS), Shenzhen, China\\
$^3$Guangdong Provincial Key Laboratory of Future Networks of Intelligence, Shenzhen, China\\
$^4$School of Engineering, Westlake University\\
$^5$Research Center for Industries of the Future, Westlake University
}

\begin{document}

\iclrfinalcopy

\maketitle

\begin{abstract}
  Federated Learning (FL) is an evolving distributed machine learning approach that safeguards client privacy by keeping data on edge devices. However, the variation in data among clients poses challenges in training models that excel across all local distributions. Recent studies suggest clustering as a solution to address client heterogeneity in FL by grouping clients with distribution shifts into distinct clusters. Nonetheless, the diverse learning frameworks used in current clustered FL methods create difficulties in integrating these methods, leveraging their advantages, and making further enhancements.
  To this end, this paper conducts a thorough examination of existing clustered FL methods and introduces a four-tier framework, named \algframework, to encompass and extend the existing approaches. Utilizing the \algframework, we identify persistent challenges associated with current clustering methods in each tier and propose an enhanced clustering method called \algfednew to overcome these challenges. Through extensive numerical evaluations, we demonstrate the effectiveness of our clustering framework and the enhanced components. Our code is available at \url{https://github.com/LINs-lab/HCFL}.
\end{abstract}

\input{main}


\bibliography{main}
\bibliographystyle{iclr2025_conference}

\clearpage
\appendix


\input{appendix}

\end{document}

%% file: configuration/lins_macros.tex
\usepackage{amsmath,amsthm,amssymb}
\newtheorem{theorem}{Theorem}[section]
\newtheorem{lemma}[theorem]{Lemma}
\newtheorem{corollary}[theorem]{Corollary}
\newtheorem{definition}[theorem]{Definition}

\newtheorem{remark}[theorem]{Remark}
\newtheorem{assumption}{Assumption}

\providecommand{\norm}[1]{\left\lVert#1\right\rVert}

\providecommand{\R}{\mathbb{R}} %

\providecommand{\E}{{\mathbb E}}
\providecommand{\E}[1]{{\mathbb E}\left.#1\right. }        %
\providecommand{\Eb}[1]{{\mathbb E}\left[#1\right] }       %

\providecommand{\derive}[2]{\frac{{\partial}{#1}}{\partial #2} }  %

\DeclareMathOperator*{\argmin}{arg\,min}
\DeclareMathOperator*{\argmax}{arg\,max}

\providecommand{\0}{\mathbf{0}}
\providecommand{\1}{\mathbf{1}}

\renewcommand{\gg}{\mathbf{g}}

\providecommand{\pp}{\mathbf{p}}
\providecommand{\qq}{\mathbf{q}}

\providecommand{\vv}{\mathbf{v}}

\providecommand{\xx}{\mathbf{x}}
\providecommand{\yy}{\mathbf{y}}
\providecommand{\zz}{\mathbf{z}}

\providecommand{\mA}{\mathbf{A}}
\providecommand{\mB}{\mathbf{B}}

\providecommand{\mD}{\mathbf{D}}
\providecommand{\mE}{\mathbf{E}}
\providecommand{\mF}{\mathbf{F}}
\providecommand{\mG}{\mathbf{G}}
\providecommand{\mH}{\mathbf{H}}
\providecommand{\mI}{\mathbf{I}}

\providecommand{\mP}{\mathbf{P}}
\providecommand{\mQ}{\mathbf{Q}}
\providecommand{\mR}{\mathbf{R}}
\providecommand{\mS}{\mathbf{S}}

\providecommand{\mV}{\mathbf{V}}
\providecommand{\mW}{\mathbf{W}}
\providecommand{\mX}{\mathbf{X}}
\providecommand{\mY}{\mathbf{Y}}
\providecommand{\mZ}{\mathbf{Z}}

\providecommand{\mTheta}{\boldsymbol{\Theta}}
\providecommand{\mOmega}{\boldsymbol{\Omega}}

\providecommand{\mphi}{\boldsymbol{\phi}}

\providecommand{\mnu}{\boldsymbol{\nu}}

\providecommand{\mtheta}{\boldsymbol{\theta}}

\providecommand{\cA}{\mathcal{A}}
\providecommand{\cB}{\mathcal{B}}
\providecommand{\cC}{\mathcal{C}}
\providecommand{\cD}{\mathcal{D}}
\providecommand{\cE}{\mathcal{E}}
\providecommand{\cF}{\mathcal{F}}

\providecommand{\cL}{\mathcal{L}}

\providecommand{\cN}{\mathcal{N}}

\providecommand{\cP}{\mathcal{P}}

\providecommand{\cR}{\mathcal{R}}
\providecommand{\cS}{\mathcal{S}}
\providecommand{\cT}{\mathcal{T}}

\newenvironment{talign}
{\align}
{\endalign}
\newenvironment{talign*}
{\csname align*\endcsname}
{\endalign}

%% file: configuration/lins_algo.tex
\usepackage{algorithm}
\usepackage[noend]{algpseudocode}

\newcommand{\mycaptionof}[2]{\captionof{#1}{#2}}

\errorcontextlines\maxdimen

\makeatletter
\newcommand*{\algrule}[1][\algorithmicindent]{\makebox[#1][l]{\hspace*{.5em}\thealgruleextra\vrule height \thealgruleheight depth \thealgruledepth}}%
\newcommand*{\thealgruleextra}{}
\newcommand*{\thealgruleheight}{.75\baselineskip}
\newcommand*{\thealgruledepth}{.25\baselineskip}

\newcount\ALG@printindent@tempcnta
\def\ALG@printindent{%
	\ifnum \theALG@nested>0
	\ifx\ALG@text\ALG@x@notext
	\else
		\unskip
		\addvspace{-1pt}
		\ALG@printindent@tempcnta=1
		\loop
		\algrule[\csname ALG@ind@\the\ALG@printindent@tempcnta\endcsname]%
		\advance \ALG@printindent@tempcnta 1
		\ifnum \ALG@printindent@tempcnta<\numexpr\theALG@nested+1\relax
			\repeat
		\fi
	\fi
}%
\usepackage{etoolbox}
\patchcmd{\ALG@doentity}{\noindent\hskip\ALG@tlm}{\ALG@printindent}{}{\errmessage{failed to patch}}
\makeatother

\newbox\statebox
\newcommand{\myState}[1]{%
	\setbox\statebox=\vbox{#1}%
	\edef\thealgruleheight{\dimexpr \the\ht\statebox+1pt\relax}%
	\edef\thealgruledepth{\dimexpr \the\dp\statebox+1pt\relax}%
	\ifdim\thealgruleheight<.75\baselineskip
		\def\thealgruleheight{\dimexpr .75\baselineskip+1pt\relax}%
	\fi
	\ifdim\thealgruledepth<.25\baselineskip
		\def\thealgruledepth{\dimexpr .25\baselineskip+1pt\relax}%
	\fi
	\State #1%
	\def\thealgruleheight{\dimexpr .75\baselineskip+1pt\relax}%
	\def\thealgruledepth{\dimexpr .25\baselineskip+1pt\relax}%
}

%% file: main.tex
\section{Introduction}
Federated Learning (FL) is a privacy-focused distributed machine learning approach. In FL, the server shares the model with clients for local training, and the clients send parameter updates back to the server. The clients will not share their raw data with servers, ensuring privacy. However, the non-iid client data distribution leads to significant performance drops for FL algorithms~\citep{mcmahan2016communication,li2018federated,karimireddy2020scaffold,karimireddyscaffold2019}.
To address data heterogeneity, traditional FL focuses on training a single global model that performs well across all local distributions~\citep{li2021model, li2018federated, tang2022virtual, guo2022fedaug}. However, relying solely on a global model may not adequately handle the heterogeneous client distributions. As a remedy, clustered FL methods have been proposed to group clients into different clusters based on their local distributions~\footnote{In this study, we address the issue of supervised clustered FL, which may differ from the unsupervised clustered FL examined by ~\citet{ding2023horizontal,qiao2024federated}.}. Numerous studies have demonstrated the superiority of clustered FL methods over single-model FL approaches~\citep{long2023multi, sattler2020byzantine, ghosh2020efficient, marfoq2021federated, guo2023fedconceptem}.
\looseness=-1


\begin{figure*}[!t]
    \centering
    \includegraphics[width=.9\textwidth]{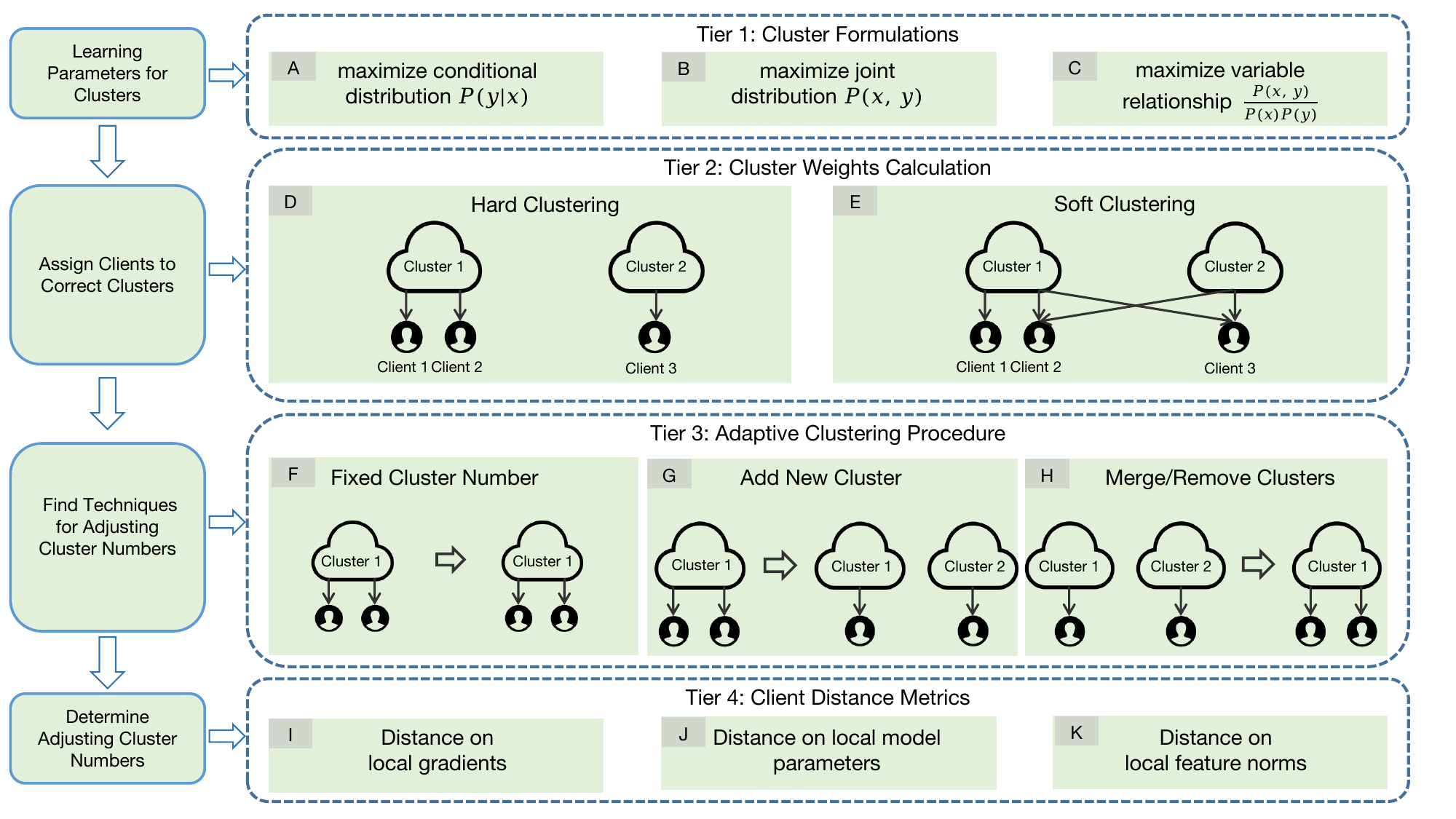}
    \caption{\small
        \textbf{Overview of the \algframework.}
        The \algframework encompasses the existing clustered FL algorithms through the design of four tiers,
        including \textit{cluster formulations}, which maximize conditional distribution, joint distribution, or variable relationships;
        \textit{cluster weights calculation}, including soft clustering and hard clustering; \textit{adaptive clustering procedure}, including using a predefined number of clusters, automatically adding new clusters, or merge and remove existing clusters; \textit{client distance metrics}, including using distance on clients' local gradients, clients' local model parameters, or clients' local feature norms. The four tiers collaborate to form a comprehensive clustered FL learning process, as shown in the left part of the figure. For instance, CFL can be described by the A, D, G, and J, while A, E, and F cover FedEM.\looseness=-1
    }
    \vspace{-2em}
    \label{fig:diverse-learning-system}
\end{figure*}

\textbf{Diverse learning frameworks pose challenges on enhancing the clustered FL.}
Despite the success of current clustered FL methods, the use of diverse learning frameworks poses challenges in integrating different algorithms, gathering their advantages, and achieving further improvements.
For instance, FedEM~\citep{marfoq2021federated} excels in addressing complex mixture distribution scenarios and performs admirably on challenging tasks. However, it necessitates a predefined number of clusters, constraining its practicality.
In contrast, adaptive clustering techniques such as CFL~\citep{sattler2020byzantine} can autonomously determine the number of clusters. Nonetheless, CFL cannot be seamlessly integrated with soft clustering methods like FedEM, thereby limiting its effectiveness in handling complex mixture distribution tasks.


\textbf{Consolidating existing methods as a solution.}
To tackle these challenges, we believe there is a need to develop a holistic learning framework for supervised clustered FL methods, allowing us to seamlessly combine their advantages.
In this paper, we introduce \algframework, a holistic clustered FL algorithm framework incorporating (1) a unified clustering objective function that handles both soft and hard clustering and (2) a unified, four-tier clustering procedure paradigm~\footnote{This study focuses on data heterogeneity. While other techniques may enhance privacy and communication efficiency, they are not directly related to our approach.} (as shown in Figure~\ref{fig:diverse-learning-system}).

The \algframework framework enables the flexible combination of existing techniques within each tier, unlocking new benefits beyond the mere recovery of traditional methods. For example, FedRC~\cite{guo2023fedconceptem} demonstrates strong generalization performance but cannot automatically determine the number of clusters, whereas CFL~\cite{sattler2020clustered} excels at determining the number of clusters but lacks generalization. The \algframework framework allows for the integration of the strengths of both FedRC and CFL, achieving both superior generalization and personalization performance (see Table~\ref{tab:performance-beta-02}). Additionally, enabling both cluster removal and addition yields comparable—or even better—performance than baseline algorithms, with a significantly reduced number of clusters (see Table~\ref{tab:performance-beta-02}).

\textbf{Enhancing clustered FL methods by improved methodologies.}
In light of the \algframework, we have identified the remaining challenges within each tier that were previously overlooked by existing clustered FL, as illustrated in Figure~\ref{fig:challenges} in Section~\ref{sec:remain-challenges}. We then introduce \algfednew, an enhanced algorithm designed to tackle these remaining challenges.
Numerical results confirm that \algfednew effectively extends existing methods, achieving a superior balance between personalization and generalization while delivering strong performance.
We summarize the contribution of this paper as follows:


\begin{itemize}[leftmargin=12pt,nosep]
    \item We introduce \algframework, a holistic framework for clustered FL that encompass the existing methods. The \algframework represents a feasible approach in achieving the integration of benefits from existing methods by adjusting the techniques at each tier.\looseness=-1
    \item We identify four remaining challenges within each tier of \algframework, and introduce an improved algorithm called \algfednew to address these challenges.
    \item Extensive experiments on different datasets (CIFAR10, CIFAR100, and Tiny-Imagenet) and various architectures (MobileNet-V2 and ResNet18) demonstrate the effectiveness of our framework and the improved components of \algfednew.\looseness=-1
\end{itemize}

\section{Related Works}

In the field of Federated Learning, FedAvg serves as the de-facto algorithm, employing local Stochastic Gradient Descent (local SGD) techniques~\citep{mcmahan2016communication,lin2020dont} to reduce communication costs and protect client privacy. However, FL faces significant challenges due to distribution shifts among clients, which can hinder the performance of FL algorithms~\citep{li2018federated,wang2020federated,karimireddy2020scaffold,jiang2023test,guo2021towards}.
To address these challenges, researchers have introduced clustered FL algorithms to enhance the performance of FL algorithms.\looseness=-1

Clustered FL groups clients based on their local data distribution, addressing the distribution shift problem. Most methods employ hard clustering with a fixed number of clusters, grouping clients by measuring their similarities~\citep{ghosh2020efficient, long2023multi, wang2022accelerating, stallmann2022towards,ma2022convergence}. However, hard clustering may not adequately capture complex relationships between local distributions, and soft clustering paradigms have been proposed to address this issue~\citep{marfoq2021federated, wu2023personalized, ruan2022fedsoft, guo2023fedconceptem,reisser2021federated,ruan2022fedsoft}.
In this paper, we propose a generalized formulation for clustered FL that encompasses current methods and improves them by addressing issues related to intra-client inconsistency and efficiency.\looseness=-1

Another line of research focuses on automatically determining the number of clusters. Current methods utilize hierarchical clustering~\citep{sattler2020byzantine, sattler2020clustered, zhao2020cluster, briggs2020federated, zeng2023stochastic, duan2021fedgroup, duan2021flexible}, which measures client dissimilarity using model parameters or local gradient distances. Some papers enhance these distance metrics by employing various techniques, such as eigenvectors~\citep{yan2023clustered} and local feature norms~\citep{wei2023edge}. FEDCOLLAB~\citep{bao2023optimizing} quantifies client similarity through client discriminators. However, the requirement for discriminators between every client pair in FEDCOLLAB hinders scalability for cross-device scenarios with numerous clients.
In this paper, we concentrate on cross-device settings, introducing a holistic adaptive clustering framework enabling cluster splitting and merging. We also present enhanced weight updating for soft clustering and finer distance metrics for various clustering principles. For further discussions on related works, please refer to Appendix~\ref{sec:related-works-appendix}.
\looseness=-1



\section{\algframework: Revisiting and Extending Clustered FL Methods} \label{sec:revisit}
Current clustered FL methods typically employ diverse learning frameworks. As a result, existing methods often face challenges in gathering the advantages of different algorithms for potential enhancements.
To address this issue, as shown in Figure~\ref{fig:diverse-learning-system}, we introduce the \algframework, consisting of four tiers designed to tackle the primary tasks of clustering methods:
(1) Cluster Learning and Assignment (tiers 1 and 2): Identify which clients should belong to the same clusters and learns parameters for each cluster.
(2) Cluster Number Determinant (tiers 3 and 4): Decide the number of clusters.
As a result, the four tiers of the \algframework form a comprehensive learning process (Algorithm~\ref{alg:algorithm-framework-general}), enabling flexible improvements and the integration of advantages from different algorithms (Table~\ref{tab:performance-beta-02}).\looseness=-1

\subsection{Tiers 1 \& 2: The \tierA  and \tierB }
We introduce the first two tiers: \tierA and \tierB. \tierA defines the objective functions of the clustering methods, aiming to learn the underlying distributions of each cluster.
\tierB  orthogonally helps find the suitable clusters for each client, whereas hard clustering assigns each client to one cluster, while soft clustering allows clients to contribute to multiple clusters.
We propose the following optimization framework to encompass these two tiers.


\textbf{Optimization framework of clustered FL methods.}
The clustered FL methods can be expressed as a dual-variable optimization problem that maximizes $\cL(\mTheta, \mOmega)$, with $K$ clusters and $M$ data sources represented as $\cD_1, \cdots, \cD_M$:
\begin{small}
    \begin{talign}
        \label{equ:general-objective}
        \textstyle
        & \cL(\mTheta, \mOmega) = \frac{1}{N} \sum_{i=1}^{M} \sum_{j=1}^{N_i} \log \left( \sum_{k=1}^{K} \omega_{i;k} \cL_k(\xx_{i,j}, y_{ij}; \mtheta_k) \right) \,, \nonumber \\
        & \text{s.t.}
        \quad
        \sum_{k=1}^{K} \omega_{i;k} = 1, \forall i  \,,
    \end{talign}
\end{small}%
where $N = \sum_{i=1}^{M} N_i$, $N_i := |\cD_i|$, $(\xx_{i,j}, y_{ij})$ are sampled from $\cD_i$, and $\cL_k(\xx_{i,j}, y_{ij}; \mtheta_k)$ is the local objective function.
The parameters to be optimized are clustering weights $\mOmega = [\omega_{1;1}, \cdots, \omega_{M, K}]$, and model parameters $\mTheta = [\mtheta_1, \cdots, \mtheta_K]$.

\textbf{Tier 1: Incorporate existing \tierA .}
The existing methods employ clustering to address diverse tasks, which results in the proposal of various formulations. Our Algorithm~\ref{alg:algorithm-framework-general} encompasses the existing clustering formulations by selecting $\cL_k(\xx_{i,j}, y_{ij}; \mtheta_k)$ as follows:
\begin{itemize}[leftmargin=12pt,nosep]
    \item $\cP_{\mtheta_k}(y_{i,j} | \xx_{i,j})$.
          Most existing methods~\citep{marfoq2021federated,ghosh2020efficient,long2023multi} can be recovered using this conditional distribution (likelihood functions).
    \item $\cP_{\mtheta_k} (\xx_{i,j}, y_{i,j})$.
          FedGMM~\citep{wu2022motley} uses joint probability.~\footnote{In FedGMM~\citep{wu2023personalized}, $\mtheta_k$ is split into $[\mtheta_{k_1}, \mnu_{k_2}]$, and it uses $\cL_k(\xx_{i,j}, y_{i,j}; \mtheta_k) = \cP_{\mtheta_{k_1}}(y_{i,j}|\xx_{i,j})\cP_{\mnu_{k_2}}(\xx_{i,j})$ to model the joint probability.}
    \item $\frac{\cP_{\mtheta_k}(\xx_{i,j},y_{i,j})}{\cP_{\mtheta_k}(\xx_{i,j})\cP_{\mtheta_k}(y_{i,j})}$.
          FedRC~\citep{guo2023fedconceptem} relies on correlations between variables $\xx$ and $y$. \looseness=-1
\end{itemize}


\textbf{Tier 2: Incorporate existing \tierB .}
Various methods employ distinct mechanisms for calculating clustering weights $\omega_{i;k}$.
The choice of $\omega_{i;k}$, with either binary values $\omega_{i;k} \in \{0, 1\}$ or continuous values $\omega_{i;k} \in [0, 1]$, characterizes the dynamic clustering procedure.
\begin{itemize}[leftmargin=12pt,nosep]
    \item Hard clustering methods employ binary values $\omega_{i;k} \in {0, 1}$.
          In these methods, $\omega_{i;k}$ is determined using heuristic techniques, such as parameter distance~\citep{long2023multi, zeng2023stochastic, sattler2020byzantine}, local loss function values~\citep{ghosh2020efficient} or label distributions~\citep{diao2024exploiting}.\looseness=-1
    \item Soft clustering approaches permit $\omega_{i;k} \in [0, 1]$, determined by maximizing $\cL(\mTheta, \mOmega)$\citep{marfoq2021federated, guo2023fedconceptem, wu2023personalized}, or by normalizing local loss values\citep{ruan2022fedsoft}. Soft clustering methods do not assume separated clients' local distributions and can thus handle complex scenarios, such as mixture distributions~\citep{marfoq2021federated, wu2023personalized}.
          \looseness=-1
\end{itemize}

\subsection{Tiers 3 \& 4: The Adaptive Clustering Procedure and Distance Metrics}
\label{sec:holistic-framework}

Tiers 3 and 4 illustrate the techniques for Cluster Number Determination. In detail, the adaptive clustering procedures automatically adjust the number of clusters, while distance metrics control the adaptive clustering procedures, determining whether clusters should split or merge.
The \algframework allows for different techniques at each tier, enhancing flexibility in choosing the optimal adaptive clustering methods or converting methods that rely on fixed cluster numbers to adaptive ones.
\looseness=-1

\textbf{Tier 3: Adaptive clustering procedures demonstrate how to modify cluster numbers.}
To automatically determine the number of clusters, current approaches can be categorized into two orthogonal methods:
(1) Splitting clusters to increase the number of clusters~\citep{sattler2020byzantine, sattler2020clustered}.
(2) Merging clusters to reduce the number of clusters~\citep{zeng2023stochastic}.
We unify these approaches at tier 3.
\looseness=-1


\textbf{Tier 4: Client distance metrics dictate when cluster numbers should be adjusted.}
The client's distances are utilized to determine whether the current number of clusters should be adjusted. For instance, when the distance within a cluster is large, the cluster will divide into sub-clusters. Conversely, if the distances between two clusters are small, these two clusters should be merged. Existing clustering methods use various metrics such as cosine similarity of local gradients~\citep{sattler2020byzantine}, gradients from a globally shared network~\citep{zeng2023stochastic}, and local feature norms~\citep{wei2023edge}.
\looseness=-1




\begin{figure*}
    \centering
    \includegraphics[width=.9\textwidth]{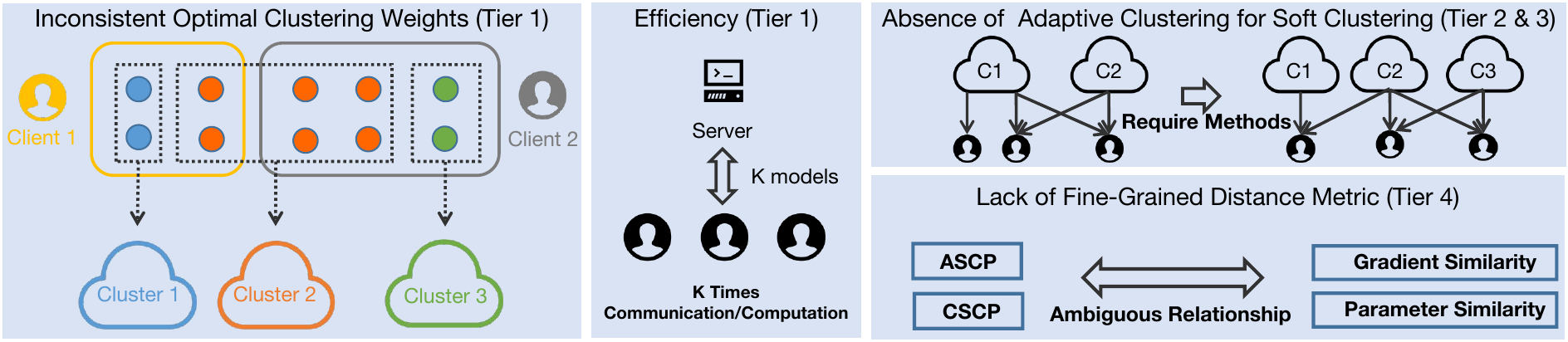}
    \caption{\small
        \textbf{Remaining challenges in clustered FL methods.}
        We identify four key issues in clustered FL algorithms: (1) inconsistent intra-client clustering weights, (2) efficiency concerns, (3) the absence of adaptive clustering for soft clustering methods, and (4) the lack of fine-grained distance metrics for various clustering principles. Clustering principles \cpA and \cpB differ in their approach as follows: \cpA assigns clients with any shifts into different clusters, while \cpB only assigns clients with concept shifts to different clusters."
    }
    \vspace{-1.5em}
    \label{fig:challenges}
\end{figure*}

\section{\algfednew: Tackling Remaining Challenges in Clustered FL}
\label{sec:method}

Section~\ref{sec:revisit} introduces a holistic clustering framework with four tiers to encompass existing methods. However, each tier still presents challenges that current methods cannot address. In this section, we outline four key remaining challenges in Figure~\ref{fig:challenges} and introduce \algfednew to tackle them. Due to space constraints, we summarize the improved algorithm in Algorithm~\ref{alg:algorithm-framework}.

While the \algfednew employs standard FL training process without introducing extra privacy concerns compared to traditional methods. Existing privacy-preserving techniques, like secure aggregation, can be integrated as plugins. We plan to explore this possibility further in the future work.

\subsection{Remaining Challenges of the Clustering in FL}
\label{sec:remain-challenges}

In this subsection, we identify four remaining challenges of the \algframework. We categorize these challenges by tiers in the \algframework, as shown in Figure~\ref{fig:challenges}. The details are provided below.
\looseness=-1

\textbf{Challenges on tier 1: Inconsistent intra-client clustering weights and efficiency concerns.}
These challenges can be addressed by improving the clustering formulations.
\begin{itemize}[leftmargin=12pt,nosep]
    \item \textit{Inconsistent intra-client clustering weights.}
          Existing approaches use the same clustering weights $\omega_{i;k}$ for all the samples belonging to client $i$~\citep{sattler2020byzantine,ghosh2020efficient,marfoq2021federated,guo2023fedconceptem}. However, they overlook cases where the optimal clustering weights of different samples within the same client can be inconsistent, implying that $\omega_{i,j_1;k} \neq \omega_{i,j_2;k}$ for certain samples $(\xx_{i,j_1}, y_{i, j_1})$ and $(\xx_{i,j_2}, y_{i, j_2})$.
          See our example here\footnote{
              In real-world scenarios, clients may have varying local data distributions, but they can still share certain data, like common knowledge or location information. While client-specific data may belong to different clusters, shared data should be in the same clusters, leading to inconsistent intra-client clustering weights.
          }.
          \looseness=-1
    \item \textit{Efficiency.} The current clustered FL methods~\citep{marfoq2021federated,long2023multi,guo2023fedconceptem} require K-fold higher communication or computation costs, hindering overall algorithm efficiency during deployment.
\end{itemize}

\begin{algorithm}[!t]
    \small
    \begin{algorithmic}[1]
        \small
        \Require{Number of communication rounds $T$, initial number of clusters $K^{0}$, initial parameters $\mphi^{0}$, and $\mTheta^{0}$.}
        \Ensure{Number of clusters $K^{T}$, trained parameters $\mphi^{T}$, and $\mTheta^{T}$.}
        \For{$t = 0, \cdots, T-1$}
        \myState{Sample a subset of clients $\cS^{t}$, and send $\mTheta^{t+1}$ to the clients.}
        \For{Client $i$ in $\cS^{t}$}
        \myState{Do local SGD by solving~\eqref{equ:general-objective} for $\tau$ local epochs..\Comment{Tiers 1 and 2}}
        \myState{Upload local model updates $\gg_{i;k}^{t+1}$ to the server.}
        \EndFor
        \myState{$\mtheta_{k}^{t+1} = \mtheta_{k}^{t} - \eta_g \sum_{i \in \cS^{t}} \gg_{i;k}^{t+1}$, $\forall k$.}
        \myState{Calculate distance matrix $\mD^{t}$, and $\mD_k^{t}$, $\forall k$. \Comment{Tier 4}}
        \If{Detect cluster $k_{s}$ need to be split }\Comment{Tier 3}
        \myState{Split cluster $k_{s}$ into sub-clusters $\cS_{s, 1}$ and $\cS_{s, 2}$.}
        \myState{$\mtheta_{k_s}^{t+1} = \mtheta_{k}^{t} - \eta_g \sum_{i \in \cS_{s, 1}} \gg_{i;k}^{t+1}$.}
        \myState{$\mtheta_{K^{t} + 1}^{t+1} = \mtheta_{k}^{t} - \eta_g \sum_{i \in \cS_{s, 2}} \gg_{i;k}^{t+1}$.}
        \myState{Update $\omega_{i;k}$ for corresponding clients.}
        \EndIf
        \If{Detect cluster $k_{d}$ need to be deleted}\Comment{Tier 3}
        \myState{Delete cluster $k_{d}$.}
        \myState{Update $\omega_{i;k}$ for corresponding clients.}
        \EndIf
        \myState{Update $K^{t+1}$ by the current number of clusters.}

        \EndFor
    \end{algorithmic}
    \mycaptionof{algorithm}{\small Algorithm Framework of \algframework.}
    \label{alg:algorithm-framework-general}
\end{algorithm}

\textbf{Challenges on tiers 2 \& 3: The absence of adaptive clustering for soft clustering methods.}
Current adaptive clustering methods primarily address hard clustering~\citep{sattler2020byzantine,sattler2020clustered,zeng2023stochastic}.
Hence, there exists a gap between research and practice, as there is a need to automatically determine the number of clusters for soft clustering methods~\citep{marfoq2021federated,guo2023fedconceptem}.
\looseness=-1

\textbf{Challenges on tier 4: Lack of fine-grained distance metrics for various clustering principles.}
The clustering principles determine which clients should be assigned to the same clusters. Existing clustering methods may use different clustering principles, as described by \cpA and \cpB below: \looseness=-1
\begin{itemize}[leftmargin=12pt,nosep]
    \item \cpA (Any Shift Type Clustering Principle): clients with any distribution shifts are placed into separate clusters~\citep{marfoq2021federated, wu2023personalized}.
    \item \cpB (Concept Shift Only Clustering Principle): only clients with concept shifts are assigned to separate clusters~\citep{guo2023fedconceptem}.
\end{itemize}
As discussed in Section~\ref{sec:holistic-framework}, client distances determine whether the current number of clusters should be changed, aligning with the role of clustering principles: if the current number of clusters cannot meet the requirements of the clustering principles, the cluster number should be adjusted.
Consequently, we advocate for distance metrics to be closely tied to distribution shifts, ultimately aligning with clustering principles. Unfortunately, existing distance metrics, such as those based on local gradients or local model parameters~\citep{sattler2020byzantine, zeng2023stochastic, long2023multi, yan2023clustered}, cannot establish a clear link to distribution shifts. As a result, current methods struggle to satisfy diverse and detailed clustering principles.
\looseness=-1

\subsection{Improve Tier1: Inconsistency and Efficiency Aware Objective Functions}
\label{sec:inconsistency-aware-objective}

To address tier 1 challenges, specifically, (i) inconsistent intra-client clustering weights and (ii) efficiency, we propose an extension of the objective function (Eq.~\eqref{equ:general-objective}), which is defined as $\mathcal{L}(\boldsymbol{\phi}, \boldsymbol{\Theta}, \boldsymbol{\Omega}, \tilde{\boldsymbol{\Omega}})$ and includes the parameters $\boldsymbol{\phi}$, $\boldsymbol{\Theta}$, $\boldsymbol{\Omega}$, and $\tilde{\boldsymbol{\Omega}}$.
\looseness=-1
\begin{itemize}[leftmargin=12pt,nosep]
    \item \textit{Shared feature extractor $\mphi$, and cluster-specific predictors $\mTheta = [\mtheta_1, \cdots, \mtheta_K]$.}
          Dividing the feature extractor $\mphi$ and the predictors $\{ \mtheta_k \}$ reduces communication and computation costs since the predictors are lightweight architectures, like linear classifier layers.
          \looseness=-1
    \item \textit{Sample-wise clustering weights $\mOmega = [\omega_{1,1;1} \cdots, \omega_{M, N_M; K}]$
          for enhanced training stage and client-wise clustering weights $\tilde{\mOmega} = [\tilde{\omega}_{1;1} \cdots, \tilde{\omega}_{M;K}]$ for testing stage}~\footnote{Experiments on the effectiveness of sample-wise clustering weights in Figures~\ref{fig:fedem-test-mu} and~\ref{fig:fedrc-test-mu}.}.
          We employ sample-specific clustering weights ($\omega_{i,j;k}$) during training to ensure that data samples from the same clients can contribute to different cluster models, resolving the issue of inconsistent intra-client clustering weights. Furthermore, during testing, when test-time label information is unavailable, we utilize client-specific weights ($\tilde{\omega}_{i;k}$) for each client and cluster.
          \looseness=-1
    \item \textit{Enhance the optimization of $\omega_{i,j;k}$ by regularizing the distance between $\tilde{\omega}_{i;k}$ and $\omega_{i,j;k}$.}
          Motivated by the intuition that ``if data from the same clients have similar distributions, the corresponding clustering weights should be similar'',
          we encourage $\tilde{\omega}_{i;k}$ and $\omega_{i,j;k}$ to be close to each other.
          \looseness=-1
\end{itemize}

The following objective function is designed to meet our requirements.
\begin{small}
    \begin{talign}
        \textstyle
        & \cL(\mphi, \mTheta, \mOmega, \tilde{\mOmega}) = \underbrace{\frac{1}{N} \sum_{i=1}^{M} \sum_{j=1}^{N_i} \log \left( \sum_{k=1}^{K} \omega_{i,j;k} \cL_k(\xx_{i,j}, y_{ij}) \right)}_{\cA_1}
        - \underbrace{\mu \sum_{i=1}^{M} \sum_{j=1}^{N_i} \left( \sum_{k=1}^{K} \tilde{\omega}_{i;k} \log \frac{\tilde{\omega}_{i;k}}{\omega_{i,j;k}} \right)}_{\cA_2} \label{equ:fedias-objective}                      \\
        & \text{s.t.}
        \sum_{k=1}^{K} \omega_{i,j;k} = 1, \, \forall i, j  \,,  \sum_{k=1}^{K} \tilde{\omega}_{i;k} = 1, \, \forall i \,,
        \tilde{\mOmega} = \argmin_{\tilde{\mOmega}} \left| \max_{\mOmega} \cL(\mphi, \mTheta, \mOmega, \tilde{\mOmega}) -  \cL(\mphi, \mTheta, \tilde{\mOmega}, \tilde{\mOmega}) \right| \,, \label{equ:obtain-cluster-weights}
    \end{talign}
\end{small}%
where $\cL_k(\xx_{i,j}, y_{ij})$ is the simplified notation of $\cL_k(\xx_{i,j}, y_{ij}; \mphi, \mtheta_k)$. $\cA_1$ term is extended from~\eqref{equ:general-objective} by using the global shared feature extractor $\mphi$ and the sample-wise weights $\omega_{i,j;k}$.
$\cA_2$ focuses on regularizing the difference between the sample-wise clustering weights $\omega_{i,j;k}$ and the client-wise clustering weights $\tilde{\omega}_{i;k}$.
The $\mu$ controls the strength of this regularization.
We obtain $\tilde{\omega}_{i;k}$ by solving~\eqref{equ:obtain-cluster-weights}, where we aim to minimize the impact of replacing $\omega_{i,j;k}$ with $\tilde{\omega}_{i;k}$.


\textbf{Optimization of the proposed objective function.}
Different from heuristic methods used in most studies to optimize~\eqref{equ:fedias-objective}~\footnote{IFCA~\citep{ghosh2020efficient} sets $\omega_{i,j;k_{i, \text{min}}} = 1, \forall j$ when $k_{i, \text{min}} = \argmin_k \E_{D_i} \left[f_{i;k}(\xx_{i,j}, y_{i,j}, \mphi, \mtheta_k)\right]$, where $f_{i;k}$ is the local loss function. FeSEM~\citep{long2023multi} sets $k_{i, \text{min}} = \argmin_k \norm{\mtheta_k - \mtheta_i}_2$, where $\mtheta_k, \mtheta_i$ represents the model parameters of cluster $k$ and client $i$, respectively.\looseness=-1},
we aim to introduce a more interpretable approach by maximizing the objective functions (Eq.~\eqref{equ:fedias-objective}). Detailed proof refer to Appendix~\ref{sec:Proof of EM steps}.
Specifically, we can update $\tilde{\omega}_{i;k}$, $\omega_{i,j;k}$, $\mtheta_k$, and $\mphi$ by~\eqref{equ:uptate-gamma}--\eqref{equ:update-mphi}.
\begin{small}
    \begin{talign}
        \textstyle
        \gamma_{i,j;k}^{t+1}       & = \frac{ \omega_{i,j;k}^{t} \cL_k^{t}(\xx_{i,j}, y_{i,j})}{\sum_{n=1}^{K} \omega_{i,j;n}^{t} \cL_k^{t}(\xx_{i,j}, y_{i,j})} \,, \label{equ:uptate-gamma}
        \tilde{\gamma}_{i,j;k}^{t+1} = \frac{ \tilde{\omega}_{i;k}^{t} \cL_k^{t}(\xx_{i,j}, y_{i,j})}{\sum_{n=1}^{K} \omega_{i;n}^{t} \cL_k^{t}(\xx_{i,j}, y_{i,j})} \,,  \\
        \tilde{\omega}_{i;k}^{t+1} & = \frac{1}{N_i} \sum_{j=1}^{N_i} \tilde{\gamma}_{i,j;k}^{t+1} \,,                                                                                                                                                                                      \omega_{i,j;k}^{t+1}  = \frac{\gamma_{i,j;k}^{t+1}}{1 + \mu N} + \frac{\mu N}{1 + \mu N} \tilde{\omega}_{i;k}^{t+1}  = \tilde{\mu} \gamma_{i,j;k}^{t+1} + (1 - \tilde{\mu}) \tilde{\omega}_{i;k}^{t+1} \, , \label{equ:update-omega} \\
        \mtheta_k^{t+1}            & = \mtheta_k^{t} - \eta \sum_{i=1}^{M} \sum_{j=1}^{N_i} \frac{\gamma_{i,j;k}^{t+1}}{\cL_k^{t}(\xx_{i,j}, y_{i,j})} \nabla_{\mtheta_k} \cL_k^{t}(\xx_{i,j}, y_{i,j}) \, ,   \label{equ:update-theta}                                                                                                                                                                                                                                                                 \\
        \mphi^{t+1}                & = \mphi^{t} - \eta \sum_{i=1}^{M} \sum_{j=1}^{N_i} \sum_{k=1}^{K} \frac{\gamma_{i,j;k}^{t+1}}{\cL_k^{t}(\xx_{i,j}, y_{i,j})} \nabla_{\mphi} \cL_k^{t}(\xx_{i,j}, y_{i,j}) \label{equ:update-mphi}\, ,
    \end{talign}
\end{small}%

where $\cL_k^{t}(\xx_{i,j}, y_{ij})$ is the simplified notation of local objective function $\cL_k(\xx_{i,j}, y_{ij}; \mphi^{t}, \mtheta_k^{t})$, $\gamma_{i,j;k}$ and $\tilde{\gamma}_{i,j;k}$ are intermediate results for calculating $\omega_{i,j;k}$ and $\tilde{\omega}_{i;k}$.
More detailed proofs can be found in Appendix~\ref{sec:Proof of EM steps}.
$\tilde{\mu} = \frac{1}{1 + \mu N}$ serves as a hyperparameter to control the strength of the penalty term in Equation~\eqref{equ:fedias-objective}. We provide the theoretical analysis on linear representation learning case in Appendix~\ref{sec:Theoretical Study on Linear Representation Case}.\looseness=-1

\subsection{Improve Tiers 2 \& 3: Adaptive Clustering for Soft Clustering Paradigms}
\label{sec:cover-soft-cluster}
Given the limitations of existing adaptive clustering methods, we have extended the clustering weight update mechanisms to incorporate soft clustering and have verified its effectiveness in Figures~\ref{fig:val-scwu} and~\ref{fig:test-scwu}.
The overall process is summarized in Algorithms~\ref{alg:add-cluster} and~\ref{alg:remove-cluster}.
\looseness=-1
In Algorithm~\ref{alg:add-cluster}, the clustering weights are adjusted after splitting cluster $k$ into two sub-clusters, denoted by $k_1$ and $k_2$. Then we set $\omega_{i,j,k_1} = \omega_{i,j,k_2} = \omega_{i,j,k} / 2$ for all $i$ and $j$.
In Algorithm~\ref{alg:remove-cluster}, the clustering weights are updated when removing cluster $k$. For all $k^{'} \neq k$, we modify $\omega_{i,j;k^{'}}$ as $\omega_{i,j;k^{'}} = \frac{\omega_{i,j;k^{'}}}{\sum_{n \neq k} \omega_{i,j;n}}$.\looseness=-1

We use the hyperparameter $\rho$ to control cluster splitting.
As evidenced in Table~\ref{tab:performance-beta-02}, a higher $\rho$ results in fewer clusters, signifying enhanced generalization but reduced personalization.
In detail, the cluster $k$ will split if the following condition is met:
\begin{small}
    \begin{talign}
        \textstyle
        \max (\mD_k) - \operatorname{mean} (\mD_k) \ge \rho \,,
    \end{talign}
\end{small}%
where $\mD_k$ is the distance matrix of cluster $k$.
We identify the need for cluster removal when the cluster no longer receives the highest clustering weights from any clients.
Additional details about the enhanced adaptive process can be found in Algorithm~\ref{alg:algorithm-framework}.

\looseness=-1

\subsection{Improve Tier4: Fine-Grained Distance Metric Design} \label{sec:interpretable-distance-metric}
Due to the page limitations, we include most of the details about the method design and practical implements in Appendix~\ref{sec:algorithms}.
As discussed in Section~\ref{sec:remain-challenges}, various algorithms may group clients into different clusters based on different clustering principles.
Therefore, in this section, we design the following fine-grained distance metrics for these different clustering principles. In detal, we have $\mD_{i,j}^{k} = $
\looseness=-1
\begin{small}
    \begin{talign}
        \textstyle
        \begin{split}
            \left \{            \begin{array}{ll}
                \max \left\{ d_c, d_{lf} \right \}
                \E_{D_i} \left[\tilde{\cL}_k^{t}(\zz, y) \right]
                \E_{D_j} \left[\tilde{\cL}_k^{t}(\zz, y) \right] \label{equ:theory-distance} \, , & \text{\cpA} \, , \\
                d_c
                \E_{D_i} \left[\tilde{\cL}_k^{t}(\zz, y) \right]
                \E_{D_j} \left[\tilde{\cL}_k^{t}(\zz, y) \right] \, ,                             & \text{\cpB} \, ,
            \end{array}
            \right.
        \end{split}
    \end{talign}
\end{small}%
where

\begin{small}
    \begin{align*}
        d_c \!=\! \max_{y} \left\{ \text{dist} \left( \E_{D_i} \left[\cP(\zz | \xx, y;\mphi) \right], \E_{D_j} \left[\cP(\zz| \xx, y; \mphi) \right] \right) \right\} \, ,
        d_{lf} \!=\! \text{dist} \left( \E_{D_i} \left[\cP(\zz | \xx;\mphi) \right], \E_{D_j} \left[\cP(\zz| \xx; \mphi) \right] \right)
    \end{align*}
\end{small}

where $\tilde{\cL}_k^{t}(\zz, y)$ is the simplified notation of local objective function $\tilde{\cL}_k^{t}(\zz, y; \mtheta_k)$ given extracted feature $\zz$, $\text{dist}$ is the cos-similarity.
The distances above become large only when the following conditions occur together:
(1) Large values of $d_c$ indicate concept shifts between clients $i$ and $j$;
(2) Large $d_{lf}$ indicate significant feature and label distribution differences.
(3) Large values of $\E_{D_i} \left[\tilde{\cL}_k(\zz, y;\mtheta_k) \right] \E_{D_j} \left[\tilde{\cL}_k(\zz, y;\mtheta_k) \right]$ indicate incorrect clustering weights with high confidence.
The effectiveness of the above distance metrics design is evidenced in Table~\ref{tab:ablation-adaptive}.

\paragraph{Approximation of the distance metrics in practice.}
When calculating the distance metrics (Equation~\eqref{equ:theory-distance}) in practice, to avoid training extra generative networks and transmitting more data between servers and clients, we substitute $\tilde{\omega}_{i;k}$ for $\tilde{\cL}_k(\zz, y;\mtheta_k)$ since $\tilde{\omega}_{i;k}$ is positively correlated with $\tilde{\cL}_k(\zz, y;\mtheta_k)$~\citep{marfoq2021federated,guo2023fedconceptem}. Additionally, we approximate $\E_{D_i} \left[\cP(\zz | \xx, y;\mphi) \right]$ and $\E_{D_i} \left[\cP(\zz | \xx, y;\mphi) \right]$ using feature prototypes.
The prototypes are defined by the following equation:
\begin{small}
    \begin{align}
        \tilde{d}_c = Dist(\mP_{c, i}, \mP_{c, j}) \, ,
        \tilde{d}_{lf} = Dist(\mP_{lf, i}, \mP_{lf, j}) \, ,
    \end{align}
\end{small}
where
\begin{small}
    \begin{align}
        \textstyle
        \mP_{c, i} \in \R^{d \times C}
                               & = [\frac{1}{N_{i, 1}} \sum_{j=1}^{N_i} \mathbf{1}_{y_{i,j}=1} g(\xx_{i,j}, \mphi), \cdots, \frac{1}{N_{i, C}} \sum_{j=1}^{N_i} \mathbf{1}_{y_{i,j}=C} g(\xx_{i,j}, \mphi)] \, ,
        \label{equ:get-the-prototypes}                                                                                                                                                                           \\
        \mP_{lf, i} \in \R^{d} & = \frac{1}{N_{i}} \sum_{j=1}^{N_i} g(\xx_{i,j}, \mphi) \, ,
        \label{equ:get-the-mean-prototypes}
    \end{align}
\end{small}
$N_{i,c} = \sum_{j=1}^{N_i} \mathbf{1}_{y_{i,j}=c}$, $g(\xx_{i,j}, \mphi)$ is the function parameterized by $\mphi$,
$Dist$ is a function to measure the distance between prototypes, which we use the cosine similarity as an example in this paper.

\looseness=-1

\section{Numerical Results}
In this section, we evaluate the performance of \algfednew and other clustered FL methods.
Additional experiment results, including hyper-parameter ablation studies, different model architectures, additional scenarios, efficiency comparision, and clusteirng quality illustration can be found in Appendix~\ref{sec:Additional Experiment Results}.\looseness=-1


\begin{table*}[!t]
    \centering
    \caption{\small
        \textbf{Performance of the adaptive clustering methods} on CIFAR10, CIFAR100, and Tiny-Imagenet datasets.
        For each algorithm, we present the best Validation and Test accuracies averaged over three trials.
        For clustering methods that require a fixed number of clusters, we set $K = 3$.
        The hyperparameters $\text{tol}_1$, $\text{tol}_2$, $\alpha^{*}(0)$, $\tau$, and $\rho$ in adaptive clustering methods govern the balance between personalization and generalization, as well as the cluster number. For instance, lower $\tau$ in StoCFL or lower $\rho$ in \algfednew indicate improved personalization and reduced generalization.
        $K^{T}$ denotes the cluster number in the final training round, where a larger $K^{T}$ suggests improved personalization but reduced generalization and communication efficiency.
        We emphasize the best results in \textbf{bold} and the worst results in \textcolor{blue}{blue}.
    }
    \resizebox{.9\textwidth}{!}{

        \begin{tabular}{l c c c c c c c c c c c c}
            \toprule
            \multirow{2}{*}{Algorithm}                   & \multicolumn{3}{c}{CIFAR10, $\beta = 0.2$}                     & \multicolumn{3}{c}{CIFAR100, $\beta = 0.2$}                    & \multicolumn{3}{c}{Tiny-Imagenet, $\beta = 0.2$}                                                                                                                                                                                                                                                                               \\
            \cmidrule(lr){2-4} \cmidrule(lr){5-7} \cmidrule(lr){8-10}
                                                         & Val                                                            & Test                                                           & $K^{T}$                                          & Val                                                            & Test                                                          & $K^{T}$ & Val                                                   & Test                                                           & $K^{T}$ \\
            \midrule
            \textbf{Pre-defined $K$}                                                                                                                                                                                                                                                                                                                                                                                                                                                                                        \\
            FedAvg                                       & $49.19$ \small{\transparent{0.5} $\pm 2.15$}                   & $45.42$ \small{\transparent{0.5} $\pm 2.42$}                   & $1$                                              & $\textcolor{blue}{26.01}$ \small{\transparent{0.5} $\pm 1.15$} & $27.87$ \small{\transparent{0.5} $\pm 2.12$}                  & $1$     & $38.83$ \small{\transparent{0.5} $\pm 0.20$}          & $39.07$ \small{\transparent{0.5} $\pm 0.44$}                   & $1$     \\
            FeSEM                                        & $45.30$ \small{\transparent{0.5} $\pm 0.40$}                   & $29.01$ \small{\transparent{0.5} $\pm 0.79$}                   & $3$                                              & $26.37$ \small{\transparent{0.5} $\pm 0.64$}                   & $24.50$ \small{\transparent{0.5} $\pm 0.28$}                  & $3$     & $37.10$ \small{\transparent{0.5} $\pm 0.80$}          & $30.00$ \small{\transparent{0.5} $\pm 2.02$}                   & $3$     \\
            IFCA                                         & $\textcolor{blue}{34.46}$ \small{\transparent{0.5} $\pm 2.06$} & $\textcolor{blue}{23.18}$ \small{\transparent{0.5} $\pm 2.55$} & $3$                                              & $26.99$ \small{\transparent{0.5} $\pm 3.89$}                   & $26.20$ \small{\transparent{0.5} $\pm 1.56$}                  & $3$     & $38.52$ \small{\transparent{0.5} $\pm 0.30$}          & $29.92$ \small{\transparent{0.5} $\pm 0.30$}                   & $3$     \\
            FedEM                                        & $66.49$ \small{\transparent{0.5} $\pm 0.69$}                   & $53.64$ \small{\transparent{0.5} $\pm 1.61$}                   & $3$                                              & $29.75$ \small{\transparent{0.5} $\pm 0.47$}                   & $24.18$ \small{\transparent{0.5} $\pm 0.03$}                  & $3$     & $42.00$ \small{\transparent{0.5} $\pm 0.74$}          & $39.25$ \small{\transparent{0.5} $\pm 0.31$}                   & $3$     \\
            FedRC                                        & $63.65$ \small{\transparent{0.5} $\pm 2.95$}                   & $59.41$ \small{\transparent{0.5} $\pm 0.19$}                   & $3$                                              & $34.56$ \small{\transparent{0.5} $\pm 0.79$}                   & $\mathbf{37.62}$ \small{\transparent{0.5} $\pm 0.16$}         & $3$     & $38.93$ \small{\transparent{0.5} $\pm 0.18$}          & $39.73$ \small{\transparent{0.5} $\pm 0.04$}                   & $3$     \\

            \midrule
            \textbf{Adaptive $K$}                                                                                                                                                                                                                                                                                                                                                                                                                                                                                           \\
            CFL                                                                                                                                                                                                                                                                                                                                                                                                                                                                                                             \\
            $\quad \text{tol}_1 =0.4, \text{tol}_2 =1.6$ & $61.55$ \small{\transparent{0.5} $\pm 1.74$}                   & $46.88$ \small{\transparent{0.5} $\pm 0.35$}                   & $6$                                              & $35.05$ \small{\transparent{0.5} $\pm 0.35$}                   & $24.84$ \small{\transparent{0.5} $\pm 2.50$}                  & $4$     & $37.41$ \small{\transparent{0.5} $\pm 1.87$}          & $30.25$ \small{\transparent{0.5} $\pm 0.55$}                   & $3$
            \\
            $\quad \text{tol}_1 =0.4, \text{tol}_2 =0.8$ & $65.06$ \small{\transparent{0.5} $\pm 3.34$}                   & $45.74$ \small{\transparent{0.5} $\pm 4.01$}                   & $9$                                              & $36.98$ \small{\transparent{0.5} $\pm 3.37$}                   & $22.00$ \small{\transparent{0.5} $\pm 1.88$}                  & $5$     & $40.36$ \small{\transparent{0.5} $\pm 3.55$}          & $28.82$ \small{\transparent{0.5} $\pm 0.71$}                   & $4$
            \\
            $\quad \text{tol}_1 =0.2, \text{tol}_2 =0.8$ & $58.92$ \small{\transparent{0.5} $\pm 2.09$}                   & $55.02$ \small{\transparent{0.5} $\pm 0.97$}                   & $4$                                              & $37.73$ \small{\transparent{0.5} $\pm 7.68$}                   & $31.47$ \small{\transparent{0.5} $\pm 0.09$}                  & $3$     & $35.74$ \small{\transparent{0.5} $\pm 0.57$}          & $34.41$ \small{\transparent{0.5} $\pm 1.92$}                   & $1$
            \\
            ICFL                                                                                                                                                                                                                                                                                                                                                                                                                                                                                                            \\
            $\quad \alpha^{*} (0) =0.85$                 & $77.59$ \small{\transparent{0.5} $\pm 0.04$}                   & $57.38$ \small{\transparent{0.5} $\pm 1.91$}                   & $98$                                             & $52.73$ \small{\transparent{0.5} $\pm 1.03$}                   & $32.77$ \small{\transparent{0.5} $\pm 0.28$}                  & $100$   & $64.72$ \small{\transparent{0.5} $\pm 0.30$}          & $34.73$ \small{\transparent{0.5} $\pm 0.39$}                   & $87$    \\
            $\quad \alpha^{*} (0) =0.98$                 & $60.58$ \small{\transparent{0.5} $\pm 1.07$}                   & $61.18$ \small{\transparent{0.5} $\pm 0.78$}                   & $14$                                             & $41.49$ \small{\transparent{0.5} $\pm 4.11$}                   & $33.57$ \small{\transparent{0.5} $\pm 1.56$}                  & $40$    & $53.05$ \small{\transparent{0.5} $\pm 2.57$}          & $35.09$ \small{\transparent{0.5} $\pm 0.25$}                   & $42$    \\
            StoCFL                                                                                                                                                                                                                                                                                                                                                                                                                                                                                                          \\
            $\quad \tau=0.05$                            & $59.79$ \small{\transparent{0.5} $\pm 1.34$}                   & $57.35$ \small{\transparent{0.5} $\pm 0.92$}                   & $15$                                             & $29.97$ \small{\transparent{0.5} $\pm 0.47$}                   & $31.40$ \small{\transparent{0.5} $\pm 2.16$}                  & $4$     & $31.85$ \small{\transparent{0.5} $\pm 0.08$}          & $31.39$ \small{\transparent{0.5} $\pm 0.87$}                   & $1$     \\
            $\quad \tau=0.10$                            & $70.84$ \small{\transparent{0.5} $\pm 1.58$}                   & $51.72$ \small{\transparent{0.5} $\pm 0.07$}                   & $54$                                             & $\mathbf{69.76}$ \small{\transparent{0.5} $\pm 2.57$}          & $\textcolor{blue}{9.42}$ \small{\transparent{0.5} $\pm 0.07$} & $89$    & $\mathbf{67.48}$ \small{\transparent{0.5} $\pm 1.53$} & $\textcolor{blue}{13.03}$ \small{\transparent{0.5} $\pm 0.67$} & $91$    \\
            \midrule
            \algfednew (FeSEM)                                                                                                                                                                                                                                                                                                                                                                                                                                                                                              \\
            \quad $\rho=0.05$                            & $\mathbf{87.77}$ \small{\transparent{0.5} $\pm 1.11$}          & $41.85$ \small{\transparent{0.5} $\pm 4.11$}                   & $58$                                             & $\mathbf{69.25}$ \small{\transparent{0.5} $\pm 0.69$}          & $14.24$ \small{\transparent{0.5} $\pm 1.93$}                  & $67$    & $60.44$ \small{\transparent{0.5} $\pm 0.86$}          & $23.14$ \small{\transparent{0.5} $\pm 1.46$}                   & $32$    \\
            \quad $\rho=0.1$                             & $85.08$ \small{\transparent{0.5} $\pm 0.11$}                   & $43.34$ \small{\transparent{0.5} $\pm 0.94$}                   & $44$                                             & $62.32$ \small{\transparent{0.5} $\pm 0.23$}                   & $16.67$ \small{\transparent{0.5} $\pm 2.97$}                  & $38$    & $52.18$ \small{\transparent{0.5} $\pm 2.90$}          & $32.97$ \small{\transparent{0.5} $\pm 1.27$}                   & $14$    \\
            \quad $\rho=0.3$                             & $79.31$ \small{\transparent{0.5} $\pm 3.95$}                   & $47.62$ \small{\transparent{0.5} $\pm 2.90$}                   & $17$                                             & $44.49$ \small{\transparent{0.5} $\pm 1.57$}                   & $28.03$ \small{\transparent{0.5} $\pm 0.85$}                  & $8$     & $45.76$ \small{\transparent{0.5} $\pm 0.09$}          & $36.08$ \small{\transparent{0.5} $\pm 1.25$}                   & $4$     \\
            \algfednew (FedEM)                                                                                                                                                                                                                                                                                                                                                                                                                                                                                              \\
            \quad $\rho=0.05$                            & $82.45$ \small{\transparent{0.5} $\pm 0.13$}                   & $57.73$ \small{\transparent{0.5} $\pm 1.70$}                   & $22$                                             & $60.36$ \small{\transparent{0.5} $\pm 1.47$}                   & $22.95$ \small{\transparent{0.5} $\pm 1.44$}                  & $40$    & $63.41$ \small{\transparent{0.5} $\pm 0.05$}          & $34.24$ \small{\transparent{0.5} $\pm 0.33$}                   & $33$    \\
            \quad $\rho=0.1$                             & $84.64$ \small{\transparent{0.5} $\pm 1.47$}                   & $60.90$ \small{\transparent{0.5} $\pm 0.61$}                   & $16$                                             & $62.98$ \small{\transparent{0.5} $\pm 0.42$}                   & $26.17$ \small{\transparent{0.5} $\pm 1.22$}                  & $34$    & $59.88$ \small{\transparent{0.5} $\pm 0.11$}          & $37.17$ \small{\transparent{0.5} $\pm 0.37$}                   & $20$    \\
            \quad $\rho=0.3$                             & $83.67$ \small{\transparent{0.5} $\pm 0.72$}                   & $62.43$ \small{\transparent{0.5} $\pm 0.71$}                   & $10$                                             & $50.72$ \small{\transparent{0.5} $\pm 2.97$}                   & $32.13$ \small{\transparent{0.5} $\pm 0.18$}                  & $9$     & $45.53$ \small{\transparent{0.5} $\pm 0.53$}          & $38.64$ \small{\transparent{0.5} $\pm 0.23$}                   & $3$     \\
            \algfednew (FedRC)                                                                                                                                                                                                                                                                                                                                                                                                                                                                                              \\
            \quad $\rho=0.05$                            & $69.16$ \small{\transparent{0.5} $\pm 0.65$}                   & $67.37$ \small{\transparent{0.5} $\pm 0.42$}                   & $8$                                              & $39.20$ \small{\transparent{0.5} $\pm 0.31$}                   & $34.38$ \small{\transparent{0.5} $\pm 0.64$}                  & $11$    & $43.78$ \small{\transparent{0.5} $\pm 0.31$}          & $38.75$ \small{\transparent{0.5} $\pm 0.54$}                   & $10$    \\
            \quad $\rho=0.1$                             & $71.67$ \small{\transparent{0.5} $\pm 0.83$}                   & $68.64$ \small{\transparent{0.5} $\pm 0.76$}                   & $8$                                              & $39.56$ \small{\transparent{0.5} $\pm 0.14$}                   & $34.62$ \small{\transparent{0.5} $\pm 0.78$}                  & $8$     & $44.26$ \small{\transparent{0.5} $\pm 0.10$}          & $38.82$ \small{\transparent{0.5} $\pm 0.77$}                   & $6$     \\
            \quad $\rho=0.3$                             & $69.33$ \small{\transparent{0.5} $\pm 0.24$}                   & $\mathbf{69.67}$ \small{\transparent{0.5} $\pm 1.27$}          & $3$                                              & $39.97$ \small{\transparent{0.5} $\pm 0.21$}                   & $\mathbf{36.50}$ \small{\transparent{0.5} $\pm 0.28$}         & $4$     & $42.60$ \small{\transparent{0.5} $\pm 0.21$}          & $\mathbf{40.65}$ \small{\transparent{0.5} $\pm 0.36$}          & $3$     \\
            \bottomrule
        \end{tabular}
    }
    \label{tab:performance-beta-02}
    \vspace{-1.5em}

\end{table*}

\subsection{Datasets and Experiment settings}
\textbf{Diverse distribution shifts scenarios.}
We establish clients with three types of distribution shifts.
For label distribution shifts, we employ LDA with $\alpha = 1.0$, as introduced by~\citep{yoshida2019hybrid,hsu2019measuring,reddi2021adaptive}.
For feature distribution shifts, we adopt the methodology from CIFAR10-C and CIFAR100-C creation~\citep{hendrycks2019benchmarking}.
Regarding concept shift, we draw inspiration from~\citep{guo2023fedconceptem,jothimurugesan2022federated}, and selectively swap labels based on the parameter $\beta$.
For example, with $\beta = 0.1$ for CIFAR10, two labels per concept are swapped, while the remaining eight labels remain unchanged. By default, we create three concepts in the experiments. More details about the construction of scenarios are included in Appendix~\ref{sec:datasets-models}.
\looseness=-1


\textbf{Baselines.}
We use FedAvg~\citep{mcmahan2016communication} as a single-model FL example.
We consider the most recently published clustered FL methods as our baselines.
For clustered FL with fixed cluster number, we select IFCA~\citep{ghosh2020efficient}, FedEM~\citep{marfoq2021federated}, FeSEM~\citep{long2023multi}, and FedRC~\citep{guo2023fedconceptem}.
For the adaptive clustering FL methods, we choose CFL~\citep{sattler2020byzantine}, ICFL~\citep{yan2023clustered}, and StoCFL~\citep{zeng2023stochastic}.
For the \algfednew variants, we enhance backbones that necessitate a predetermined number of clusters (FeSEM, FedEM, and FedRC) by integrating them with the adaptive soft clustering method and the feature extractor-classifier split mechanism of \algfednew. The \algfednew (FedSoft) is devised by substituting the local training component of \algfednew (FedEM) with proximal regularized local updates as outlined in FedSoft~\citep{ruan2022fedsoft}.
Further elaboration can be found in Appendix~\ref{sec: experiment-settings}.
\looseness=-1


\textbf{Experiment settings.}
Unless specifically mentioned, we divide the datasets into 100 clients and execute all algorithms for 200 communication rounds.
Additional settings are provided in Appendix~\ref{sec: experiment-settings}.
We conducted all experiments using MobileNet-V2~\citep{sandler2018mobilenetv2} and results on ResNet18 defer to Table~\ref{tab:resnet} of Appendix~\ref{sec:Additional Experiment Results}.

\textbf{Evaluation metrics.}
We present the following metrics to evaluate the personalization and generalization abilities of the algorithms:
(1) Validation Accuracy for evaluating personalization: The average accuracy on local validation datasets that match the distribution of local training sets.
(2) Test Accuracy evaluating generalization: The average accuracy on global shared test datasets.


\subsection{Results on diverse distribution shifts scenarios}
In this section, we compare the performance of \algfednew with other clustered FL methods. We also perform ablation studies to confirm the effectiveness of \algfednew's proposed components.

\textbf{\algfednew achieves comparable performance and better personalization-generalization trade-offs.}
We highlight some key observations in Table~\ref{tab:performance-beta-02}.
\textbf{A.} \algfednew consistently achieves superior test accuracy, with validation accuracy surpassing that of baseline methods with a similar number of clusters. This demonstrates improved efficiency and a better balance between personalization and generalization.
\textbf{B.} Soft clustering methods like \algfednew (FedEM) and \algfednew (FedRC) outperform hard clustering methods in test accuracy, showcasing their superior generalization capabilities.
\textbf{C.} While baseline methods may achieve higher validation accuracy by separating every client into different clusters (namely when the value of $K^{T}$ is close to 100), these trained clusters tend to overfit local distributions, resulting in significantly lower test accuracy.  As detailed in Table~\ref{tab:detail-icfl}, ICFL exhibits lower test accuracy when achieving the best personalization, while its personalization cannot outperform
\algfednew when achieving its best test performance.
\textbf{D.} The extended algorithms, namely \algfednew (FeSEM), \algfednew (FedEM), and \algfednew (FedRC), outperform the original methods that rely on fixed cluster numbers significantly.
Additionally, these extended algorithms can automatically adjust the number of clusters, making the algorithms more practical, as we illustrated in Figure~\ref{fig:cluster-number}. \textbf{E.} As shown in Table~\ref{tab:small-rho}, it is possible to enhance the personalized performance of \algfednew by employing a smaller $\rho$. However, this comes at the expense of reduced generalization performance.
\looseness=-1


\textbf{Ablation studies on Sec~\ref{sec:inconsistency-aware-objective}.}
We perform ablation studies on $\tilde{\mu}$, which control the distance between sample-wise weights $\omega_{i,j;k}$ and client-wise weights $\tilde{\omega}_{i;k}$ in Figures~\ref{fig:fedem-test-mu} and~\ref{fig:fedrc-test-mu}.
A larger $\tilde{\mu}$ signifies a greater difference between $\omega_{i,j;k}$ and $\tilde{\omega}_{i;k}$.
Our results show that \algfednew(FedEM) prefers smaller distance between $\omega_{i,j;k}$ and $\tilde{\omega}_{i;k}$.
However, \algfednew(FedRC) prefers larger $\tilde{\mu}$ values, highlighting the necessity of different clustering weights among samples within the same clients.


\begin{table*}[!t]
    \centering
    \caption{\textbf{Ablation studies on Sec~\ref{sec:interpretable-distance-metric}.} We conducted experiments on the CIFAR10 and CIFAR100 datasets, showcasing the highest test accuracies, the maximum number of clusters during training ($\max_{t} K^{t}$), and the final number of clusters ($K^{T}$) or each algorithm while maintaining a fixed value of $\rho = 0.3$. We used FedRC as the backbone, with 3 clusters identified as ideal by \cpB. \ding{192} Measuring client distance by gradient similarity; \ding{193} Excluding $\E_{D_i}[\tilde{\cL}_k(\zz,y;\mtheta_k)]$ and $\E_{D_i}[\tilde{\cL}_k(\zz,y;\mtheta_k)]$ from Eq~\eqref{equ:theory-distance}; \ding{194} Averaging $d_c$ over label $y$ instead of maximizing it over $y$.
        \looseness=-1}
    \resizebox{1.\textwidth}{!}{
        \begin{tabular}{l c c c c c c c c c c c c c}
            \toprule
            \multirow{2}{*}{Algorithm} & \multicolumn{3}{c}{CIFAR10, $\beta = 0.2$}            & \multicolumn{3}{c}{CIFAR10, $\beta = 0.4$} & \multicolumn{3}{c}{CIFAR100, $\beta = 0.2$} & \multicolumn{3}{c}{CIFAR100, $\beta = 0.4$}                                                                                                                                                                                                                                 \\
            \cmidrule(lr){2-4} \cmidrule(lr){5-7} \cmidrule(lr){8-10} \cmidrule(lr){11-13}
                                       & Test Acc                                              & $\max_{t} K^{t}$                           & $K^{T}$
                                       & Test Acc                                              & $\max_{t} K^{t}$                           & $K^{T}$
                                       & Test Acc                                              & $\max_{t} K^{t}$                           & $K^{T}$
                                       & Test Acc                                              & $\max_{t} K^{t}$                           & $K^{T}$                                                                                                                                                                                                                                                                                                                   \\
            \midrule
            \algfednew                 & $\mathbf{69.67}$ \small{\transparent{0.5} $\pm 1.27$} & $\mathbf{4.5}$                             & $\mathbf{3.0}$                              & $\mathbf{70.13}$ \small{\transparent{0.5} $\pm 0.42$} & $\mathbf{7.0}$ & $\mathbf{6.0}$ & $\mathbf{36.50}$ \small{\transparent{0.5} $\pm 0.28$} & $\mathbf{3.5}$ & $\mathbf{3.5}$ & $\mathbf{32.22}$ \small{\transparent{0.5} $\pm 0.20$} & $5.0$          & $\mathbf{4.0}$ \\
            $\;$ + \ding{192}          & $67.83$ \small{\transparent{0.5} $\pm 1.70$}          & $9.5$                                      & $7.0$                                       & $64.53$ \small{\transparent{0.5} $\pm 0.23$}          & $10.5$         & $10.0$         & $\mathbf{36.77}$ \small{\transparent{0.5} $\pm 0.67$} & $9.5$          & $8.5$          & $31.33$ \small{\transparent{0.5} $\pm 2.12$}          & $11.0$         & $7.5$          \\
            $\;\;\;\;$ + \ding{193}    & $56.14$ \small{\transparent{0.5} $\pm 8.11$}          & $10.5$                                     & $5.5$                                       & $50.87$ \small{\transparent{0.5} $\pm 2.26$}          & $12.5$         & $8.5$          & $34.11$ \small{\transparent{0.5} $\pm 1.58$}          & $10.0$         & $6.0$          & $\mathbf{32.75}$ \small{\transparent{0.5} $\pm 0.67$} & $7.5$          & $6.5$          \\
            $\;$ + \ding{193}          & $68.52$ \small{\transparent{0.5} $\pm 0.64$}          & $8.0$                                      & $6.0$                                       & $69.47$ \small{\transparent{0.5} $\pm 0.15$}          & $11.0$         & $8.5$          & $34.65$ \small{\transparent{0.5} $\pm 1.16$}          & $8.5$          & $5.5$          & $31.61$ \small{\transparent{0.5} $\pm 0.54$}          & $11.0$         & $7.5$          \\
            $\;$ + \ding{194}          & $68.82$ \small{\transparent{0.5} $\pm 0.59$}          & $5.5$                                      & $3.5$                                       & $65.74$ \small{\transparent{0.5} $\pm 0.09$}          & $\mathbf{7.0}$ & $7.0$          & $35.97$ \small{\transparent{0.5} $\pm 0.80$}          & $4.0$          & $\mathbf{3.5}$ & $31.72$ \small{\transparent{0.5} $\pm 0.59$}          & $\mathbf{4.5}$ & $\mathbf{4.0}$

                                       &                                                                                                                                                                                                                                                                                                                                                                                                                                \\
            \bottomrule
        \end{tabular}
    }
    \vspace{-1em}
    \label{tab:ablation-adaptive}
\end{table*}

\textbf{Ablation studies on Sec~\ref{sec:cover-soft-cluster}.}
In Figures~\ref{fig:val-scwu} and~\ref{fig:test-scwu}, we perform ablation studies on the soft clustering weight updating mechanism (w/ SCWU) introduced in Section~\ref{sec:cover-soft-cluster}.
The term w/o SCWU refers to using the traditional clustering weight updating mechanism as described in~\citep{sattler2020byzantine,zeng2023stochastic}.
The results demonstrate that our proposed SCWU consistently achieves better performance in terms of both validation and test accuracies.


\begin{figure}
    \centering
    \begin{subfigure}[b]{.23\textwidth}
        \includegraphics[width=1.\textwidth]{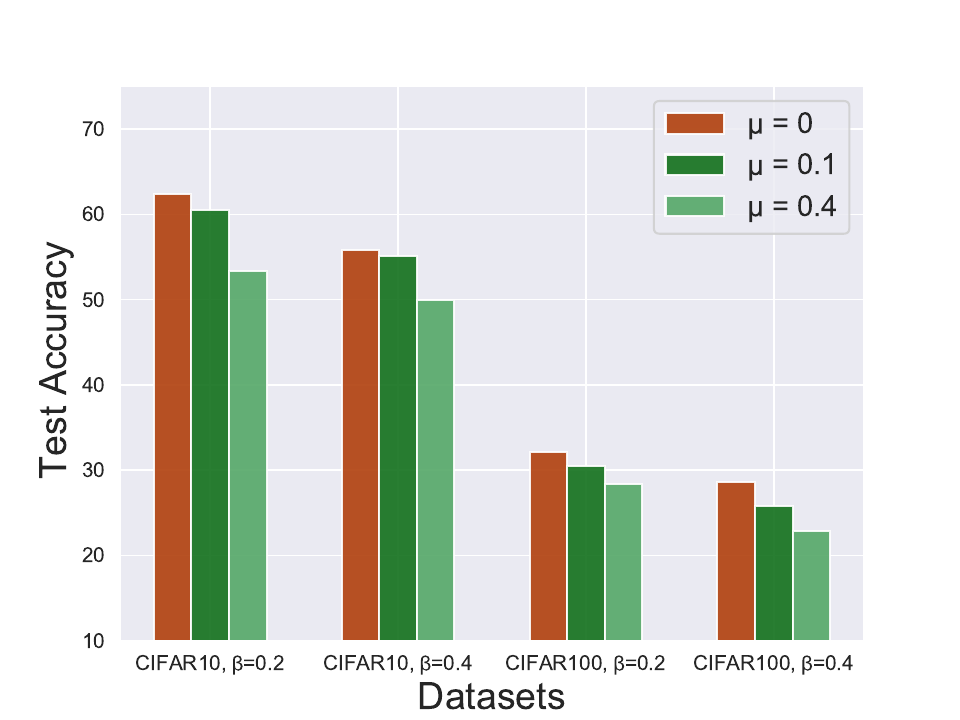}
        \caption{\tiny{\algfednew(FedEM)}}
        \label{fig:fedem-test-mu}
    \end{subfigure}
    \begin{subfigure}[b]{.23\textwidth}
        \includegraphics[width=1.\textwidth]{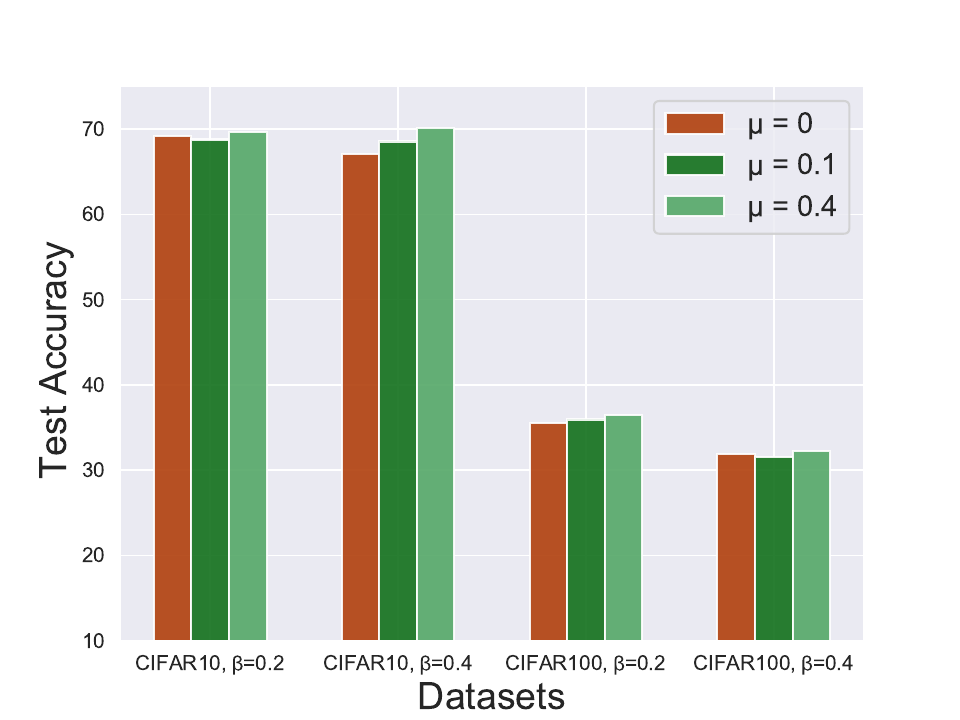}
        \caption{\tiny{\algfednew(FedRC)}}
        \label{fig:fedrc-test-mu}
    \end{subfigure}
    \begin{subfigure}[b]{.23\textwidth}
        \includegraphics[width=1.\textwidth]{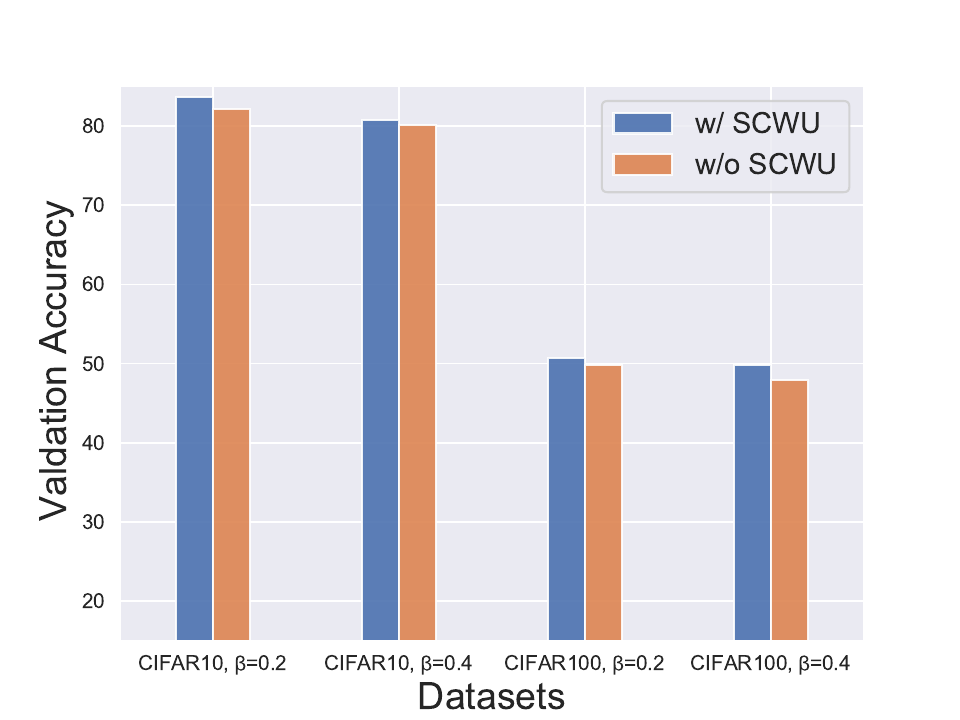}
        \caption{\tiny{Validation Acc}}
        \label{fig:val-scwu}
    \end{subfigure}
    \begin{subfigure}[b]{.23\textwidth}
        \includegraphics[width=1.\textwidth]{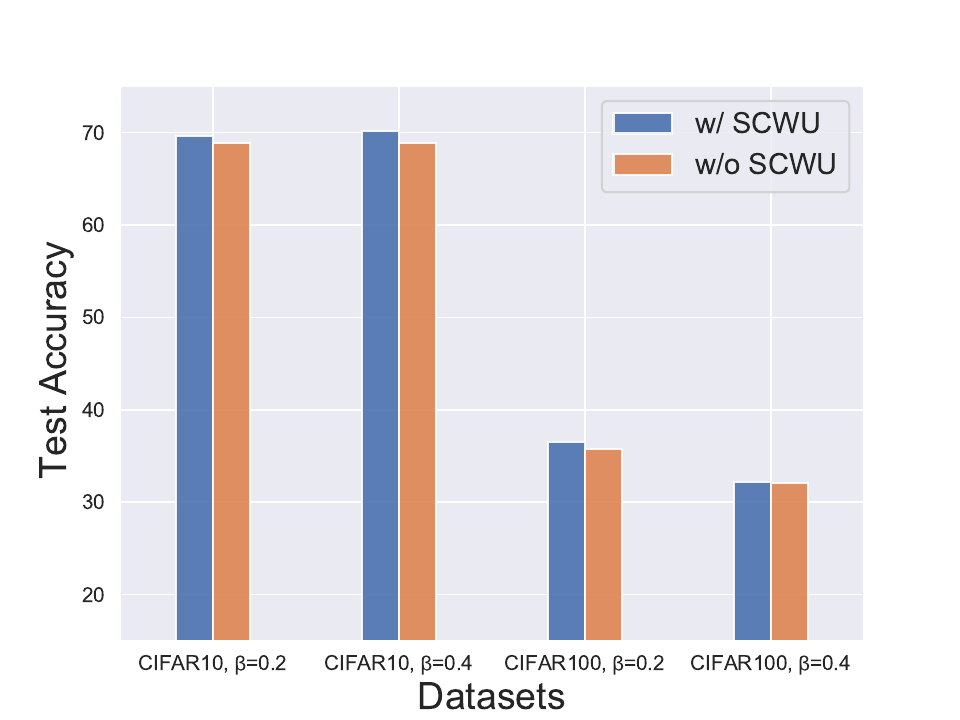}
        \caption{\tiny{Test Acc}}
        \label{fig:test-scwu}
    \end{subfigure}
    \caption{\small \textbf{Ablation studies on Sections~\ref{sec:inconsistency-aware-objective} and~\ref{sec:cover-soft-cluster}.} For Sec~\ref{sec:inconsistency-aware-objective}, we evaluated test accuracies of \algfednew using different backbones (FedEM and FedRC) and varying values of $\tilde{\mu}$, as shown in Figures~\ref{fig:fedem-test-mu} and~\ref{fig:fedrc-test-mu}. For Sec~\ref{sec:cover-soft-cluster}, we present the best Val and Test accuracy achieved by \algfednew with either FedEM or FedRC as backbones. ``w/ SCWU'' indicates the use of soft clustering weight updating mechanisms introduced in Section~\ref{sec:cover-soft-cluster}. More detailed results can be found in Tables~\ref{tab:ablation-studies-on-tier-1} and~\ref{tab:ablation-studies-on-tier-2} in Appendix~\ref{sec:Additional Experiment Results}.}
    \vspace{-1em}
    \label{fig:enter-label}
\end{figure}

\textbf{Ablation studies on techniques in Sec~\ref{sec:interpretable-distance-metric}.}
We perform ablation studies to demonstrate the effectiveness of the designed distance metrics in Sec~\ref{sec:interpretable-distance-metric}.
The ablation studies include:
\ding{192} Using gradient similarity, as in previous works~\citep{sattler2020byzantine,yan2023clustered}, instead of distance on $\cP(\zz|x;\mphi)$ and $\cP(\zz|x, y;\mphi)$, as we proposed in Equation~\ref{equ:theory-distance};
\ding{193} Remove $\E_{D_i}[\tilde{\cL}_k(\zz,y;\mtheta_k)]$ and $\E_{D_i}[\tilde{\cL}_k(\zz,y;\mtheta_k) ]$ in~\eqref{equ:theory-distance};
\ding{194} Using mean distances instead of maximum distances in \eqref{equ:theory-distance}.
The results show that \algfednew consistently achieves the highest test accuracy and produces a number of clusters closer to the ideal number than other ablation studies.

\begin{table}[!]
    \centering
    \caption{\textbf{Performance of \algfednew with small $\rho$.} We set the $\rho = 0.01$ for \algfednew, and use FedEM as the backbone algorithm.}
    \vspace{-.5em}\resizebox{.8\textwidth}{!}{
        \begin{tabular}{c c c c c}
            \toprule
            Algorithms                                          & CIFAR100 Val & CIFAR100 Test & Tiny-ImageNet Val & Tiny-ImageNet Test \\
            \midrule
            Baselines (best)                                    & 69.76        & 37.62         & 67.48             & 39.73              \\
            HCFL+ ($\rho = 0.01$)                               & 78.98        & 14.33         & 77.56             & 21.53              \\
            HCFL+ (best in Table~\ref{tab:performance-beta-02}) & 69.25        & 36.50         & 63.41             & 40.65              \\
            \bottomrule
        \end{tabular}
    }
    \label{tab:small-rho}
    \vspace{-1.5em}
\end{table}

\section{Conclusion and Limitations}
\label{sec:conclusion}
In this paper, we introduce \algframework, a comprehensive clustered FL framework that unifies existing methods while enabling the integration of diverse algorithms to gather the advantages of various clustered FL approaches.
Additionally, we identify persistent challenges unaddressed by current algorithms and propose \algfednew as a solution.
The \algframework is flexible and can generate numerous clustered FL methods by altering techniques in each tier.
Though we have chosen some typical components and demonstrated their effectiveness, conducting further performance verification with more choices in each tier would be beneficial.

\section*{Acknowledgments}
This work is supported in part by the funding from Shenzhen Institute of Artificial Intelligence and Robotics for Society, in part by the Shenzhen Key Lab of Crowd Intelligence Empowered Low-Carbon Energy Network (Grant No. ZDSYS20220606100601002), in part by Shenzhen Stability Science Program 2023, and in part by the Guangdong Provincial Key Laboratory of Future Networks of Intelligence (Grant No. 2022B1212010001).
This work is also supported in part by the Research Center for Industries of the Future (RCIF) at Westlake University, and Westlake Education Foundation.

%% file: appendix.tex
\onecolumn
{
    \hypersetup{linkcolor=black}
    \parskip=0em
    \renewcommand{\contentsname}{Contents of Appendix}
    \tableofcontents
    \addtocontents{toc}{\protect\setcounter{tocdepth}{3}}
}

Below are the key takeaways of the appendix:
\begin{itemize}[leftmargin=12pt]
    \item \textbf{Appendix~\ref{sec:Proof of EM steps}:} We provide a proof showing how ~\eqref{equ:uptate-gamma}--\eqref{equ:update-mphi} are derived, demonstrating how they solve the EM-like objective function ~\eqref{equ:fedias-objective}--~\eqref{equ:obtain-cluster-weights} in practice.
    \item \textbf{Appendix~\ref{sec:Theoretical Study on Linear Representation Case}:} We analyze the HCFL framework in the context of linear representation learning, without assuming an equal number of samples across clusters. Our findings indicate that an imbalance in sample sizes and a higher level of drift can slow down convergence.
    \item \textbf{Appendix~\ref{sec:related-works-appendix}:} A more detailed discussion of related works.
    \item \textbf{Appendix~\ref{sec:algorithms}:} We explain the rationale behind the design of the distance metric and its practical implementation. Specifically, for each pair of clients, we compute the distance between their local class prototypes for each class, as well as the distance between their feature means. The maximum of these two distances is then selected as the final distance metric for the client pair.
    \item \textbf{Appendix~\ref{sec:algorithms}:} A detailed version of the \algfednew algorithm is provided.
    \item \textbf{Appendix~\ref{sec:Additional Experiment Results}:} We report the experimental settings used in our study in detail.
    \item \item \textbf{Appendix~\ref{sec:Additional Experiment Results}:} Ablation studies on the impact of hyperparameters and algorithmic components show that \algfednew is robust to variations in hyperparameters, and that all proposed components provide individual performance gains.
\end{itemize}

\section{Proof of Optimization Steps}

\label{sec:Proof of EM steps}

\begin{theorem}

    Given objective function $\cL(\mphi, \mTheta, \mOmega, \tilde{\mOmega})$

    \begin{align}
        \cL(\mphi, \mTheta, \mOmega, \tilde{\mOmega}) & = \frac{1}{N} \sum_{i=1}^{M} \sum_{j=1}^{N_i} \log \left( \sum_{k=1}^{K} \omega_{i,j;k} \cL_k(\xx_{i,j}, y_{ij}; \mphi, \mtheta_k) \right) \nonumber \\
                                                      & + \sum_{i=1}^{M} \sum_{j=1}^{N_i} \lambda_{i,j} \left( \sum_{k=1}^{K} \omega_{i,j;k} - 1 \right) \nonumber                                           \\
                                                      & - \mu \sum_{i=1}^{M} \sum_{j=1}^{N_i} \left( \sum_{k=1}^{K} \tilde{\omega}_{i;k} \log \frac{\tilde{\omega}_{i;k}}{\omega_{i,j;k}} \right)  \, ,
    \end{align}
    and we define $\tilde{\mOmega} = \{ \tilde{\omega}_{i;k} | \forall i, k \}$, then $\tilde{\mOmega}$ is obtained by
    \begin{align}
        \tilde{\mOmega} & = \argmin_{\tilde{\mOmega}} \left| \max_{\mOmega} \cL(\mphi, \mTheta, \mOmega, \mOmega) -  \cL(\mphi, \mTheta, \tilde{\mOmega}, \tilde{\mOmega}) \right| \, .
    \end{align}

    Then E-M steps are obtained by maximizing $\cL(\mphi, \mTheta, \mOmega, \tilde{\mOmega})$.


    \begin{align}
        \gamma_{i,j;k}^{t+1}         & = \frac{ \omega_{i,j;k}^{t} \cL_k(\xx_{i,j}, y_{ij}; \mtheta_k^{t})}{\sum_{n=1}^{K} \omega_{i,j;n}^{t} \cL_k(\xx_{i,j}, y_{ij}; \mtheta_n^{t})} \, ,                                                                 \\
        \tilde{\gamma}_{i,j;k}^{t+1} & = \frac{ \tilde{\omega}_{i;k}^{t} \cL_k(\xx_{i,j}, y_{ij}; \mtheta_k^{t})}{\sum_{n=1}^{K} \omega_{i;n}^{t} \cL_k(\xx_{i,j}, y_{ij}; \mtheta_n^{t})} \, ,                                                             \\
        \tilde{\omega}_{i;k}^{t+1}   & = \frac{1}{N_i} \sum_{j=1}^{N_i} \tilde{\gamma}_{i,j;k}^{t+1} \, ,                                                                                                                                                   \\
        \omega_{i,j;k}^{t+1}         & = \frac{\gamma_{i,j;k}^{t+1}}{1 + \mu N} + \frac{\mu N}{1 + \mu N} \tilde{\omega}_{i;k}^{t+1} \, ,                                                                                                                   \\
        \mtheta_k^{t+1}              & = \mtheta_k^{t} - \eta \sum_{i=1}^{M} \sum_{j=1}^{N_i} \frac{\gamma_{i,j;k}^{t+1}}{\cL_k(\xx_{ij}, y_{ij}, \mphi^{t}, \mtheta_k^{t})} \nabla_{\mtheta_k} \cL_k (\xx_{ij}, y_{ij}, \mphi^{t}, \mtheta_k^{t}) \, ,     \\
        \mphi^{t+1}                  & = \mphi^{t} - \eta \sum_{i=1}^{M} \sum_{j=1}^{N_i} \sum_{k=1}^{K} \frac{\gamma_{i,j;k}^{t+1}}{\cL_k(\xx_{ij}, y_{ij}, \mphi^{t}, \mtheta_k^{t})} \nabla_{\mphi} \cL_k (\xx_{ij}, y_{ij}, \mphi^{t}, \mtheta_k^{t+1})
    \end{align}


    \label{the: EM steps}
\end{theorem}

\begin{proof}
    Let's begin by assuming $\tilde{\omega}_{i;k}^{t+1}$ is given for each round $t$, then we will discuss how to compute $\omega_{i,j;k}^{t+1}$. Consider the objective function $\cL(\mphi, \mTheta, \mOmega, \tilde{\mOmega})$, we have
    \begin{align}
        \derive{\cL(\mphi, \mTheta, \mOmega, \tilde{\mOmega})}{\omega_{i,j;k}}
        = \frac{1}{N} \frac{\cL_k(\xx_{i,j}, y_{ij}; \mtheta_n)}{\sum_{n=1}^{K} \omega_{i,j;n} \cL_n(\xx_{i,j}, y_{ij}; \mtheta_n)} + \lambda_{i,j} + \mu \frac{\tilde{\omega}_{i;k}}{\omega_{i,j;k}} \, .
    \end{align}
    Then define
    \begin{align}
        \gamma_{i,j;k} = \frac{\omega_{i,j;k} \cL_k(\xx_{i,j}, y_{ij}; \mtheta_n)}{\sum_{n=1}^{K} \omega_{i,j;n} \cL_n(\xx_{i,j}, y_{ij}; \mtheta_n)} \, ,
    \end{align}
    and take $\derive{\cL(\mTheta, \mOmega)}{\omega_{i,j;k}} = 0$ we have
    \begin{align}
        \frac{\gamma_{i,j;k}}{N} + \mu \tilde{\omega}_{i;k} = - \lambda_{i,j} \omega_{i,j;k} \, .
    \end{align}
    Then we have
    \begin{align}
        \omega_{i,j;k} = - \frac{1}{\lambda_{i,j}} \left( \frac{\gamma_{i,j;k}}{N} + \mu \tilde{\omega}_{i;k} \right) \, .
    \end{align}
    Because we have $\sum_{k=1}^{K} \omega_{i,j;k} = 1$, we have
    \begin{align}
        1             & = - \frac{1}{\lambda_{i,j}} \left( \frac{1}{N} + \mu \right) \\
        \lambda_{i,j} & = - \frac{1 + \mu N}{N} \, .
    \end{align}
    Then we have
    \begin{align}
        \omega_{i,j;k} = \frac{\gamma_{i,j;k}}{1 + \mu N} + \frac{\mu N}{1 + \mu N} \tilde{\omega}_{i;k} \, .
    \end{align}
    Then consider to optimize $\mtheta_k$, we have,
    \begin{align}
         & \derive{\cL (\mphi, \mathbf{\Theta}, \mOmega, \tilde{\mOmega}) }{\mtheta_k}  \nonumber                                                                                                                                   \\
         & = \frac{1}{N} \sum_{i=1}^{M} \sum_{j=1}^{N_i} \frac{\omega_{i,j;k} }{\sum_{n=1}^{K} \omega_{i,j;n} \cL_n (\xx_{ij}, y_{ij}; \mphi, \mtheta_n)} \cdot \derive{\cL_k (\xx_{ij}, y_{ij}; \mphi, \mtheta_k)}{\mtheta_k} \, , \\
         & = - \frac{1}{N} \sum_{i=1}^{M} \sum_{j=1}^{N_i} \frac{\gamma_{i,j;k}^{t+1}}{\cL_k(\xx_{ij}, y_{ij}; \mphi, \mtheta_k)} \nabla_{\mtheta_k} \cL_k (\xx_{ij}, y_{ij}; \mphi, \mtheta_k) \, .
    \end{align}
    Finally, consider to optimize $\mphi$, we have
    \begin{align}
         & \derive{\cL (\mphi, \mathbf{\Theta}, \mOmega, \tilde{\mOmega}) }{\mphi} \nonumber                                                                                                                                              \\
         & = \frac{1}{N} \sum_{i=1}^{M} \sum_{j=1}^{N_i} \sum_{k=1}^{K} \frac{\omega_{i,j;k} }{\sum_{n=1}^{K} \omega_{i,j;n} \cL_n (\xx_{ij}, y_{ij}; \mphi, \mtheta_n)} \cdot \derive{\cL_k (\xx_{ij}, y_{ij}; \mphi, \mtheta_k)}{\mphi}
        \, ,                                                                                                                                                                                                                              \\
         & = - \frac{1}{N} \sum_{i=1}^{M} \sum_{j=1}^{N_i} \sum_{k=1}^{K} \frac{\gamma_{i,j;k}^{t+1}}{\cL_k(\xx_{ij}, y_{ij}; \mphi, \mtheta_k)} \nabla_{\mphi} \cL_k (\xx_{ij}, y_{ij}; \mphi, \mtheta_k)
        \, .
    \end{align}
    Because if hard to find a close-form solution to $\derive{\cL (\mphi, \mathbf{\Theta}, \mOmega, \tilde{\mOmega}) }{\mtheta_k} = 0$ when $\mtheta_k$ is the parameter of deep neural networks, we use gradient ascent to optimize $\mtheta_k$. The same method is used for feature extractors $\mphi$.

    Then the remaining thing is how to decide $\tilde{\omega}_{i;k}$. From the formulation of objective function, the $\tilde{\omega}_{i;k}$ is decided by
    \begin{align}
        \tilde{\mOmega} & = \argmin_{\tilde{\mOmega}} \left| \max_{\mOmega} \cL(\mphi, \mTheta, \mOmega, \mOmega) -  \cL(\mphi, \mTheta, \tilde{\mOmega}, \tilde{\mOmega}) \right| \, .
    \end{align}
    To solve this problem, firstly, we can transform the definition of $\cL(\mphi, \mTheta, \tilde{\mOmega}, \tilde{\mOmega})$ to
    \begin{align}
        \cL(\mphi, \mTheta, \tilde{\mOmega}, \tilde{\mOmega}) & = \frac{1}{N} \sum_{i=1}^{M} \sum_{j=1}^{N_i} \log \left( \sum_{k=1}^{K} \omega_{i,j;k} \cL_k(\xx_{i,j}, y_{ij}; \mphi, \mtheta_k) \right) \nonumber \\
                                                              & + \sum_{i=1}^{M} \sum_{j=1}^{N_i} \lambda_{i,j} \left( \sum_{k=1}^{K} \omega_{i,j;k} - 1 \right) \nonumber                                           \\
                                                              & - \mu \sum_{i=1}^{M} \sum_{j=1}^{N_i} \left( \sum_{k=1}^{K} \tilde{\omega}_{i;k} \log \frac{\tilde{\omega}_{i;k}}{\omega_{i,j;k}} \right)  \, ,      \\
        s.t.\;                                                & \; \omega_{i,j;k} = \tilde{\omega}_{i;k}, \; \forall i,j,k \, .
    \end{align}
    Then by removing the constrains, we can always find
    \begin{align}
        \max_{\mOmega} \cL(\mphi, \mTheta, \mOmega, \tilde{\mOmega}) \ge \tilde{\cL}(\mphi, \mTheta, \tilde{\mOmega}, \tilde{\mOmega}) \, .
    \end{align}
    Then we have
    \begin{align}
        \tilde{\mOmega} & = \argmin_{\tilde{\mOmega}} \left| \max_{\mOmega} \cL(\mphi, \mTheta, \mOmega, \tilde{\mOmega}) -  \cL(\mphi, \mTheta, \tilde{\mOmega}, \tilde{\mOmega}) \right| \\
                        & = \argmin_{\tilde{\mOmega}} \left( \max_{\mOmega} \cL(\mphi, \mTheta, \mOmega, \tilde{\mOmega}) -  \cL(\mphi, \mTheta, \tilde{\mOmega}, \tilde{\mOmega}) \right) \\
                        & = \argmax_{\tilde{\mOmega}} \cL(\mphi, \mTheta, \tilde{\mOmega}, \tilde{\mOmega})
    \end{align}
    Then we can obtain the results by directly use the proof in \citep{guo2023fedconceptem} and \citep{marfoq2021federated}.
\end{proof}

\input{appendix-proof}

\section{Related Works}
\label{sec:related-works-appendix}
\paragraph{Federated Learning.}
As the de-facto algorithm in FL, FedAvg employs local SGD~\citep{mcmahan2016communication,lin2020dont} to reduce communication costs and protect client privacy.
However, distribution shifts among clients pose a significant challenge in FL and hinder the performance of FL algorithms~\citep{li2018federated,wang2020federated,karimireddy2020scaffold,jiang2023test,guo2021towards}.
Traditional FL methods primarily aim to improve the convergence speed of global models and incorporate bias reduction techniques~\citep{tang2022virtual,guo2022fedaug,li2021model,li2018federated}.
At the same time, some studies investigate feature distribution shifts using domain generalization techniques~\citep{peng2019federated,wang2022framework,shen2021fedmm,sun2022multi,gan2021fruda}. 
However, single-model approaches are inadequate for handling heterogeneous data distributions, especially when dealing with concept shifts~\citep{ke2022quantifying,guo2023fedconceptem,jothimurugesan2022federated}.
To tackle these challenges, clustered FL algorithms are introduced to enhance FL algorithm performance.\looseness=-1

\paragraph{Clustered FL with fixed cluster numbers.} Clustered FL groups clients based on their local data distribution, tackling the distribution shift problem. Most methods employ hard clustering with a fixed number of clusters, grouping clients by various similarity metrics, such as local loss values~\citep{ghosh2020efficient}, local model parameter differences~\citep{long2023multi}, communication time/local calculation time~\citep{wang2022accelerating}, and fuzzy $c$-Means~\citep{stallmann2022towards}. However, hard clustering may not capture complex relationships between local distributions adequately, and soft clustering paradigms have been proposed to address this issue. 
For instance, FedEM~\citep{marfoq2021federated} employs Expectation-Maximization techniques to maximize likelihood functions. FedGMMcitep{wu2023personalized} suggests using joint distributions instead of conditional distributions. FedRC\citep{guo2023fedconceptem} introduces Robust Clustering, assigning clients with concept shifts to different clusters to enhance model generalization. FedSoft~\citep{ruan2022fedsoft} calculates weights based on the distances between clients' local model parameters and cluster model parameters, with smaller distances indicating larger weights for that cluster.
In this paper, we propose a generalized formulation for clustered FL that encompasses the current methods and improves them by addressing issues related to intra-client inconsistency and efficiency.

\paragraph{Clustered FL with adaptive clustering numbers.} Another line of research focuses on automatically determining the number of clusters.
Current methods utilize hierarchical clustering, which measures client dissimilarity using model parameters or local gradient distances. Most current methods modify cluster numbers by splitting them when client distances within clusters are large~\citep{sattler2020byzantine,sattler2020clustered,zhao2020cluster,briggs2020federated,duan2021fedgroup,duan2021flexible}. Recently, StoCFL~\citep{zeng2023stochastic} suggests initially setting cluster numbers equal to the client count and merging clusters with small distances.
In addition to model parameter distances, some papers employ alternative distance metrics for improved performance.
For instance, \citep{yan2023clustered} employ principal eigenvectors of model parameters. \citep{vahidian2023efficient} use truncated singular value decomposition (SVD) to obtain a reduced set of principal vectors for distance measurement. Meanwhile, \citep{wei2023edge} focus on the distance of normalized local features.
FEDCOLLAB~\citep{bao2023optimizing} focuses on cross-silo scenarios with a limited number of clients
and quantifies client similarity by training client discriminators. However, the need for discriminators between every pair of clients in FEDCOLLAB makes it challenging to expand to cross-device scenarios with numerous clients.
In this paper, we concentrate on cross-device settings, introducing a holistic adaptive clustering framework enabling cluster splitting and merging. We also present enhanced weight updating for soft clustering and finer distance metrics for various clustering principles.\looseness=-1

\section{Algorithms}
\label{sec:algorithms}

\paragraph{Details of the \algfednew.} In Algorithm~\ref{alg:algorithm-framework}, we present a concise summary of the comprehensive algorithm that integrates all the enhanced components of \algfednew, as introduced in Section~\ref{sec:method}. Specifically, during each communication round, the algorithm performs the following steps:
(1) Randomly selects a subset of clients.
(2) Calculates prototypes using Equations~\eqref{equ:get-the-prototypes} and~\eqref{equ:get-the-mean-prototypes}.
(3) Performs local updates using Algorithm~\ref{alg:local-updates}.
(4) The server aggregates local updates, updates cluster model parameters, and computes client distance metrics using Equation~\eqref{equ:theory-distance} for each cluster $k$.
(5) Identifies $k_{max}$ as the cluster with the highest average distance.
(6) Checks if the maximum distance within $k_{max}$ significantly exceeds the average distance in this cluster.
(7) If the following condition is met, splits the clusters using Algorithm~\ref{alg:add-cluster}.
\begin{align}
    \max (D_{k_{max}}^{t}) - \operatorname{mean} (D_{k_{max}}^{t}) \ge \rho \, .
\end{align}
(8) Mark and remove the empty clusters no clients will assign large clustering weights to using Algorithm~\ref{alg:remove-cluster}.

\paragraph{Intuitions on the distance metrics design.} 
From the objective function (Eq.~\eqref{equ:fedias-objective}), we should assign higher clustering weights $\omega_{i,j;k}$ to clusters with greater $\cL_k$ to maximize the objective function.
Because the ultimate goals of the clustering algorithms are solving the objective functions, we analyse the $\cL_k$ to check the key factors influencing the value of $\cL_k$, and the relationships between these factors and the clustering principles.

We use the following algorithms as examples. For FedEM~\citep{marfoq2021federated} and IFCA~\citep{ghosh2020efficient},  $\cL_k(\xx, y, \mphi, \mtheta_k) = \cP_{\mphi, \mtheta_k}(y | \xx)$;
For FedRC~\citep{guo2023fedconceptem}, $\cL_k(\xx, y, \mphi, \mtheta_k) = \frac{\cP_{\mphi, \mtheta_k}(\xx, y)}{\cP_{\mphi, \mtheta_k}(\xx)\cP_{\mphi, \mtheta_k}(y)}$.
Defining $\zz = g(\xx;\mphi)$ as the local features extracted by $\mphi$, assuming a $\xx \to \zz \to y$ Probabilistic Graphical Model (with $\xx$ and $y$ being independent given $\zz$), we obtain:\looseness=-1
\begin{small}
    \begin{align*}
        \textstyle
        \cL_k(\xx, y, \mphi, \mtheta_k)
        =
        \begin{split}
            & \left \{
            \begin{array}{ll}
                \cP(y|\xx; \mphi, \mtheta_k) =  \frac{\cP(y | \zz; \mtheta_k) \cP(\zz |\xx; \mphi)}{ \cP(\zz | \xx, y; \mphi) }
                \,                                                                                                                                                                                                     & \text{\tiny (FedEM, IFCA)} \\
                \frac{\cP(\xx, y; \mphi, \mtheta_k)}{\cP(y; \mphi, \mtheta_k) \cP(\xx; \mphi, \mtheta_k)} = \frac{\cP(y | \zz; \mtheta_k) \cP(\zz |\xx; \mphi)}{\cP(y; \mphi, \mtheta_k) \cP(\zz | \xx, y; \mphi) } \, & \text{\tiny (FedRC)}
            \end{array}
            \right \} \nonumber \\
           & = \frac{\tilde{\cL}_k(\zz, y;\mtheta_k) \cP(\zz | \xx; \mphi)}{ \cP(\zz | \xx, y; \mphi) } \,.
        \end{split}
    \end{align*}
\end{small}%
Then we aim to give the following explanations of the three terms $\cP(\zz | \xx; \mphi)$, $\cP(\zz | \xx; \mphi)$, and $\tilde{\cL}_k(\zz, y;\mtheta_k)$, which align with the terms considered in Sec~\ref{sec:interpretable-distance-metric}.
\begin{itemize}[leftmargin=12pt,nosep]
    \item \textbf{$\cP(\zz | \xx; \mphi)$ for feature and label shifts}. Feature shifts introduce significant distances in $\xx$. Additionally, $\xx$ with different $y$ values generally exhibit substantial distances in the feature space. Without this, classifiers cannot distinguish samples with different labels. Hence, we employ $\cP(\zz | \xx; \mphi)$ to assess both feature and label shifts.
    \item \textbf{$\cP(\zz | \xx, y; \mphi)$ for concept shifts.} Concept shifts signify alerted $\xx-y$ correlations. Hence, samples with concept shifts but have the same $y$ should exhibit a significant difference in $\cP(\zz | \xx, y; \mphi)$.
    \item \textbf{$\tilde{\cL}_k(\zz, y;\mtheta_k)$ for the quality of clustering.} The $\tilde{\cL}_k(\zz, y;\mtheta_k)$ is defined using features $\zz = g(\xx;\mphi)$ instead of data $\xx$ in $\cL_k(\xx, y, \mphi, \mtheta_k)$. This term evaluates if features $\zz$ can be correctly assigned to clusters given the current $\mTheta$; otherwise, the objectives in~\eqref{equ:fedias-objective} cannot be achieved.
\end{itemize}

Finally, we propose the following distance metric:
\begin{small}
    \begin{align}
        \textstyle
        \mD_{i,j}^{k}
        \begin{split} =
            \left \{
            \begin{array}{ll}
                \max \left \{ d_c, d_{lf} \right \}
                \E_{D_i} \left[\tilde{\cL}_k(\zz, y;\mtheta_k) \right]
                \E_{D_j} \left[\tilde{\cL}_k(\zz, y;\mtheta_k) \right] \, , & \text{\cpA} \\
                d_c
                \E_{D_i} \left[\tilde{\cL}_k(\zz, y;\mtheta_k) \right]
                \E_{D_j} \left[\tilde{\cL}_k(\zz, y;\mtheta_k) \right] \, ,                             & \text{\cpB}
            \end{array}
            \right.
        \end{split}
        \label{equ:theory-distance-appendix}
    \end{align}
\end{small}%
where $\text{dist}$ is the cos-similarity in this paper, and
\begin{align}
    d_c \!=\! \max_{y} \left\{ \text{dist} \left( \E_{D_i} \left[\cP(\zz | \xx, y;\mphi) \right], \E_{D_j} \left[\cP(\zz| \xx, y; \mphi) \right] \right) \right\} \, , \\
    d_{lf} \!=\! \text{dist} \left( \E_{D_i} \left[\cP(\zz | \xx;\mphi) \right], \E_{D_j} \left[\cP(\zz| \xx; \mphi) \right] \right)
\end{align}
The distances above become large only when the following conditions occur together:
(1) Large values of $d_c$ indicate concept shifts between clients $i$ and $j$;
(2) Large $d_{lf}$ indicate significant feature and label distribution differences.
(2) Large values of $\E_{D_i} \left[\tilde{\cL}_k(\zz, y;\mtheta_k) \right] \E_{D_j} \left[\tilde{\cL}_k(\zz, y;\mtheta_k) \right]$ indicate incorrect clustering weights with high confidence.

\paragraph{Approximation of the distance metrics in practice.}
When calculating the distance metrics (Equation~\eqref{equ:theory-distance}) in practice, to avoid training extra generative networks and transmitting more data between servers and clients, we substitute $\tilde{\omega}_{i;k}$ for $\tilde{\cL}_k(\zz, y;\mtheta_k)$ since $\tilde{\omega}_{i;k}$ is positively correlated with $\tilde{\cL}_k(\zz, y;\mtheta_k)$~\citep{marfoq2021federated,guo2023fedconceptem}. Additionally, we approximate $\E_{D_i} \left[\cP(\zz | \xx, y;\mphi) \right]$ and $\E_{D_i} \left[\cP(\zz | \xx, y;\mphi) \right]$ using feature prototypes.
The prototypes are defined by the following equation:
\begin{small}
    \begin{align}
        \tilde{d}_c = Dist(\mP_{c, i}, \mP_{c, j}) \, ,
        \tilde{d}_{lf} = Dist(\mP_{lf, i}, \mP_{lf, j}) \, ,
    \end{align}
\end{small}
where
\begin{small}
    \begin{align}
        \mP_{c, i} \in \R^{d \times C}
                               & = [\frac{1}{N_{i, 1}} \sum_{j=1}^{N_i} \mathbf{1}_{y_{i,j}=1} g(\xx_{i,j}, \mphi), \cdots, \frac{1}{N_{i, C}} \sum_{j=1}^{N_i} \mathbf{1}_{y_{i,j}=C} g(\xx_{i,j}, \mphi)] \, ,
        \label{equ:get-the-prototypes-appendix}                                                                                                                                                                           \\
        \mP_{lf, i} \in \R^{d} & = \frac{1}{N_{i}} \sum_{j=1}^{N_i} g(\xx_{i,j}, \mphi) \, ,
        \label{equ:get-the-mean-prototypes-appendix}
    \end{align}
\end{small}
$N_{i,c} = \sum_{j=1}^{N_i} \mathbf{1}_{y_{i,j}=c}$, $g(\xx_{i,j}, \mphi)$ is the function parameterized by $\mphi$,
$Dist$ is a function to measure the distance between prototypes, which we use the cosine similarity as an example in this paper.

\begin{algorithm}[!t]
    \small
    \begin{algorithmic}[1]
        \small
        \Require{Local datasets $D_1, \dots, D_N$, number of local iterations $\mathcal{T}$, number of communication rounds $T$, number of clients chosen in each round $S$, initial number of clusters $K^{0}$, number of classes $C$, and hyper-parameter $\rho$.}
        \Ensure{Trained global feature extractor $\mphi^{T}$, final number of clusters $K^{T}$, and cluster-specific predictors $\mTheta^{T} = [\mtheta_1^{T}, \cdots, \mtheta_{K^{T}}^{T}]$.}
        \myState{Initialize $\mphi^{0}, \mTheta^{0} = [\mtheta_1^{0}, \cdots, \mtheta_{K^{0}}^{0}]$.}
        \For{$t = 0, \dots, T-1$}
        \myState{Choose a subset of clients $\mathcal{S}^{t}$, where $|\mathcal{S}^{t}| = S$.}
        \For{chosen client $i \in \mathcal{S}^{t}$}
        \myState{Calculate client prototypes $\mP_{i}^{t}$ by Equation~\eqref{equ:get-the-prototypes}-~\eqref{equ:get-the-mean-prototypes}.}
        \myState{$\mathcal{F}_{i}^{t+1}, \tilde{\omega}_{i;k}^{t+1}, \mphi_{i}^{\mathcal{T}}, \mtheta_{k, i}^{\mathcal{T}} \gets$ Local updates by Algorithm~\ref{alg:local-updates}.}
        \myState{Send $\mP_{i}^{t}$, $\mathcal{F}_{i}^{t+1}$, and  $\tilde{\omega}_{i;k}^{t+1}, \mphi_{i}^{\mathcal{T}}, \mtheta_{k, i}^{\mathcal{T}}$, $\forall k \le K^{t}$ to the server.}
        \EndFor
        \myState{$\mathcal{F}_{i}^{latest} \gets \mathcal{F}_{i}^{t+1}$.}
        \myState{$\mphi^{t+1} = \frac{1}{\sum_{i \in \cS^{t}}} \sum_{i \in \cS^{t}} N_i \mphi_i^{\cT}$.}
        \myState{$\mtheta_k^{t+1} = \frac{1}{\sum_{i \in \cS^{t}}} \sum_{i \in \cS^{t}} N_i \mtheta_{k,i}^{\cT}$, $\forall k \le K^{t}$.}
        \myState{$\cF_{g}^{t+1} \gets [ \cF_{g, 1}^{t+1}, \cdots, \cF_{g, K^{t}}^{t+1} ]$, where $\cF_{g, k}^{t+1} \gets [\sum_{i} \cF_{i, k, 1}^{latest}, \cdots, \sum_{i} \cF_{i, k, C}^{latest}]$.}
        \myState{Initialize $\cC_{k}^{t} = \emptyset$, $\forall k \le K^t$.}
        \For{all client $i$}
        \myState{$c_i \gets \argmax_{k} \tilde{\omega}_{i;k}$.}
        \myState{$\cC_{c_{i}}^{t} \gets \cC_{c_i}^{t} \cup i$.}
        \EndFor
        \myState{$\cR^{t} \gets \emptyset$.}
        \For{$k \le K^{t}$}
        \If{$\cC_{k}^{t}$ is empty}
        \myState{$\cR^{t} \gets \cR^{t} \cup k$.}
        \EndIf
        \myState{Get the cluster-specific distance matrix $\mD_k \in \R^{|\tilde{\cS}_{k}^{t}| \times |\tilde{\cS}_{k}^{t}|}$, $\forall k \le K^{t}$ by Equation~\eqref{equ:theory-distance}.}
        \EndFor
        \myState{$k_{min} \gets \argmax_{k} \max(\mD_{k}^{t})$.}
        \If{$\max(\mD_{k_{min}}^{t}) - mean(\mD_{k_{min}}^{t}) \ge \rho$}
        \myState{Split $\cC_{k_{min}}^{t}$ into two clusters $\cC_{k_{min}, 1}^{t}$ and $\cC_{k_{min}, 2}^{t}$.}
        \myState{$\mtheta_{k_{min}}^{t+1} = \frac{1}{\sum_{i \in \cC_{k_{min}, 1}^{t}}} \sum_{i \in \cC_{k_{min}, 1}^{t} } N_i \mtheta_{k_{min}, i}^{\cT} $.}
        \myState{Add new cluster and update $\cF_g$ by server side of Algorithm~\ref{alg:add-cluster}.}
        \myState{$K^{t+1} \gets K^{t} + 1$.}
        \Else
        \myState{$K^{t+1} \gets K^{t}$.}
        \EndIf
        \For{cluster $k_r \in \cR^{t}$}
        \myState{Remove cluster $k_r$ and update $\cF_g$ by server side of Algorithm~\ref{alg:remove-cluster}.}
        \myState{$K^{t+1} \gets K^{t+1} - 1$.}
        \EndFor
        \myState{Send $\mphi^{t+1}$, $\mTheta^{t+1} = [\mtheta_1^{t+1}, \cdots, \mtheta_{K^{t+1}}^{t+1}]$, and information about add/remove cluster to clients.}
        \EndFor
    \end{algorithmic}
    \mycaptionof{algorithm}{\small Algorithm Framework of \algfednew}
    \label{alg:algorithm-framework}
\end{algorithm}

\begin{algorithm}[!t]
    \small
    \begin{algorithmic}[1]
        \small
        \Require{Number of local iterations $\mathcal{T}$, current number of clusters $K^{t}$, number of classes $C$, local dataset $D_i$, global feature extractor $\mphi^{t}$, cluster-specific predictors $\mTheta^{t} = [\mtheta_1^{t}, \cdots, \mtheta_{K^{t}}^{t}]$.}
        \Ensure{Trained feature extractor $\mphi_{i}^{\mathcal{T}}$, predictors $\mTheta_{i}^{\mathcal{T}} = [\mtheta_{i,1}^{\mathcal{T}}, \cdots, \mtheta_{i,K^{t}}^{\mathcal{T}}]$, $\tilde{\mOmega}_{i}^{t+1} = [\tilde{\omega}_{i;k}, \cdots, \tilde{\omega}_{i;K^{t}}]$, and $\mathcal{F}_{i}^{t+1} = [\mathcal{F}_{i, 1}^{t+1}, \cdots, \mathcal{F}_{i, K^{t}}^{t+1}]$, where $\mathcal{F}_{i, k}^{t+1} = [\mathcal{F}_{i, k, 1}^{t+1}, \cdots, \mathcal{F}_{i, k, C}^{t+1}]$.}
        \myState{Update $\gamma_{i,j;k}^{t+1}, \tilde{\gamma}_{i,j;k}^{t+1}, \omega_{i,j;k}^{t+1}, \tilde{\omega}_{i;k}^{t+1}$ by Equations~\eqref{equ:uptate-gamma}-\eqref{equ:update-omega}, $\forall j \le N_i, k \le K^{t}$.\Comment{Tier 2}}
        \For{$\tau = 1, \dots, \mathcal{T}$}\Comment{Tier 1}
        \myState{Update $\mtheta_{k, i}^{\tau}$ by Equation~\eqref{equ:update-theta}, $\forall k \le K^{t}$.}
        \myState{Update $\mphi_{i}^{\tau}$ by Equation~\eqref{equ:update-mphi}.}
        \EndFor
        \myState{$\mathcal{F}_{i, k, c}^{t+1} \gets \sum_{j=1}^{N_i} \mathbf{1}_{y_{i,j}=c} \gamma_{i,j;k}^{t+1}$.}

    \end{algorithmic}
    \mycaptionof{algorithm}{\small Local Updates of \algfednew}
    \label{alg:local-updates}
\end{algorithm}

\begin{algorithm}[!t]
    \small
    \begin{algorithmic}[1]
        \small
        \Require{$k_{min}$, set of clients $\cC_{k_{min}, 2}^{t}$, the corresponding $\mtheta_{k_{min}, i}^{\cT}$ for each client $i \in \cC_{k_{min}, 2}^{t}$, and $\cF_g$.}
        \Ensure{New $\cF_g^{t+1}$, and predictor of the new cluster $\mtheta_{K^{t} + 1}^{t+1}$.}
        \myState{\textbf{Server Side:}}
        \myState{$\mtheta_{K^{t} + 1}^{t+1} = \frac{1}{\sum_{i \in \cC_{k_{min}, 2}^{t}}} \sum_{i \in \cC_{k_{min}, 2}^{t} } N_i \mtheta_{k_{min}, i}^{\cT} $.}
        \myState{Add $\cF_{g, K^{t} + 1} \gets \cF_{g, k_{min}}$ to $\cF_g^{t+1}$.}
        \myState{Add $\cF_{i, K^{t} + 1}^{latest} \gets \cF_{i, k_{min}}^{latest}$ to $\cF_{i}^{latest}$, $\forall i$. }
        \myState{\textbf{Client Side:}}
        \myState{$\omega_{i,j;K^{t} + 1} \gets \omega_{i,j;k_{min}} / 2$, $\forall j \le N_i$.}
        \myState{$\omega_{i,j;k_{min}} \gets \omega_{i,j;k_{min}} / 2$, $\forall j \le N_i$.}
        \myState{$\tilde{\omega}_{i;K^{t} + 1} \gets \tilde{\omega}_{i;k_min} / 2$.}
        \myState{$\tilde{\omega}_{i;k_{min}} \gets \tilde{\omega}_{i;k_min} / 2$.}

    \end{algorithmic}
    \mycaptionof{algorithm}{\small Cluster Adding of \algfednew}
    \label{alg:add-cluster}
\end{algorithm}

\begin{algorithm}[!t]
    \small
    \begin{algorithmic}[1]
        \small
        \Require{The cluster needs to be removed $k_r$, and $\cF_g^{t+1}$.}
        \myState{\textbf{Server Side:}}
        \myState{Remove $\cF_{g, k_r}^{t+1}$ from $\cF_g^{t+1}$.}
        \myState{Remove $\cF_{i, k_r}^{latest}$ from $\cF_{i}^{latest}$, $\forall i$.}
        \myState{\textbf{Client Side:}}
        \myState{$\omega_{i,j;k} \gets \frac{\omega_{i,j;k}}{\sum_{n \not = k_r} \omega_{i,j;n}}$, $\forall j \le N_i, k \not = k_r$.}
        \myState{$\tilde{\omega}_{i;k} \gets \frac{\omega_{i;k}}{\sum_{n \not = k_r} \omega_{i;n}}$, $\forall k \not = k_r$.}
        \myState{Remove $\gamma_{i,j;k_r}, \tilde{\gamma}_{i,j;k_r}, \omega_{i,j;k_r}, \tilde{\omega}_{i,j;k_r}$, $\forall j \le N_i$.}
    \end{algorithmic}
    \mycaptionof{algorithm}{\small Cluster Removing of \algfednew}
    \label{alg:remove-cluster}
\end{algorithm}



\clearpage
\section{Experiment Results}
\label{sec:Additional Experiment Results}

\subsection{Datasets and Models}
\label{sec:datasets-models}

\paragraph{Diverse distribution shifts scenarios.} Similar to previous work~\citep{guo2023fedconceptem}, the diverse distribution shift scenario construct clients with three types of distribution shifts with each other:
\begin{itemize}[leftmargin=12pt,nosep]
    \item \textbf{Label Distribution Shifts:} We use the idea introduced~\citep{yoshida2019hybrid,hsu2019measuring,reddi2021adaptive}, where we leverage the Latent Dirichlet Allocation (LDA) with $\alpha = 1.0$. We split datasets to 100 clients by default.
    \item \textbf{Feature Distribution Shifts:} We utilize the idea of constructing CIFAR10-C, CIFAR100-C, and ImageNet-C~\citep{hendrycks2019benchmarking}. In detail, we apply random augmentations to client samples, selecting from 20 types, including 'Original', 'Gaussian Noise', 'Shot Noise', 'Impulse Noise', 'Defocus Blur', 'Glass Blur', 'Motion Blur', 'Zoom Blur', 'Snow', 'Frost', 'Fog', 'Brightness', 'Contrast', 'Elastic', 'Pixelate', 'JPEG', 'Speckle Noise', 'Gaussian Blur', 'Spatter', and 'Saturate'. Augmentation types remain consistent within each client.
    \item \textbf{Concept Shifts:} For label $y \le C_{\beta}$, it becomes $y$, $(1 + y) \% C_{\beta}$, and $(2 + y) \% C_{\beta}$ across concepts, where $C_{\beta} = \lfloor C * \beta \rfloor$, and $C$ is the number of classes.
\end{itemize}

\paragraph{Noisy label scenarios.} We follow the methodology of previous works~\citep{fang2022robust,ke2022quantifying} to construct noisy label scenarios. Our approach involves two types of noisy labels: symmetric flip and pair flip. Symmetric flip entails randomly flipping the original class label to any wrong class label with equal probability. Pair flip involves flipping the original class label only to a very similar wrong category. We use the parameter $\chi$ to control the noisy rate, where $\chi = 0.1$ indicates that $10\%$ of the data have wrong labels.

\subsection{Baselines and Hyper-Parameter Settings}
\label{sec: experiment-settings}

\paragraph{Detailed implementations and hyper-parameter settings for all the algorithms} Unless special mentioned, we split each dataset to 100 clients with 3 concepts. The learning rates are chosen in $\{0.03, 0.06, 1.0\}$, and we report the best results for each algorithm. We run the algorithms for 200 communication rounds and set the number of local epochs to 1. The experiments are conducted on single NVIDIA RTX 3090 GPUs.

\paragraph{Detailed implementations and hyper-parameter settings of baseline algorithms.} The details of the settings and hyper-parameters we used for the baseline methods a summarized below. We exclude the algorithms that do not require additional hyper-parameters here.
\begin{itemize}[leftmargin=12pt]
    \item \textbf{CFL}~\citep{sattler2020byzantine}. We use the public code provided by \citep{marfoq2021federated} for the CFL algorithm. The hyper-parameters $\text{tol}_1$ and $\text{tol}_2$ are tuned, and we report how the hyper-parameters affect the results of the algorithm in Table~\ref{tab:performance-beta-02}.
    \item \textbf{ICFL}~\citep{yan2023clustered}. Follow the same setting as the original paper, we set the hyper-parameter $\alpha*(0)$ to $\{0.85, 0.98\}$, and $\epsilon_1 = 4.0$.
    \item \textbf{stoCFL}~\citep{zeng2023stochastic}. We choose $\tau = \{0, 0.05, 0.1, 0.15\}$ to control the trade-off between personalization and generalization as suggested by the original paper. In addition, we choose $\lambda = 0.5$, which always achieve the best performance as reported in the original paper.
    \item \textbf{\algfednew (FedRC)}~\citep{guo2023fedconceptem}. We set $\tilde{\mu} = 0.4$, and choose $\rho = \{0.05, 0.1, 0.3\}$. The distance between clients are calculated by Equation~\eqref{equ:theory-distance}.
    \item \textbf{\algfednew (FedEM)}~\citep{marfoq2021federated}. We set $\tilde{\mu} = 0.4$, and choose $\rho = \{0.05, 0.1, 0.3\}$. The distance between clients are calculated by Equation~\eqref{equ:theory-distance}.
    \item \textbf{\algfednew (FeSEM)}~\citep{long2023multi}. We choose $\rho = \{0.05, 0.1, 0.3\}$. Follow the original paper, we use hard clustering paradigms that does not require the hyper-parameter $\tilde{\mu}$. The model splitting process is the same as \citep{sattler2020byzantine} that designed for hard clustering paradigms.
          The distance between clients are calculated by Equation~\eqref{equ:theory-distance}.
\end{itemize}

\subsection{Additional Experiment Results}

\paragraph{Results on noisy data scenarios}
In Table~\ref{tab:results on noisy data scenario}, we show the performance of clustered FL emthods on noisy data scenarios. Results show that \algfednew consistently outperform other methods by a large margin.

\begin{table}[!t]
    \centering
     \caption{\textbf{Performance of algorithms on noisy data scenarios.} We evaluated the performance of algorithms using the CIFAR10 dataset split into 100 clients. For each algorithm, we report the best test accuracy for all 200 communication rounds.}
    \resizebox{1.0\textwidth}{!}{
        \begin{tabular}{c c c c c c c}
            \toprule
            \multirow{2}{*}{Algorithm} & \multicolumn{4}{c}{CIFAR10 (MobileNetV2)}                                                                                                                                                                                     \\
            \cmidrule(lr){2-5} \cmidrule(lr){6-7}
                                       & \textit{Pairflip}, $\chi=0.1$                         & \textit{Pairflip},$\chi=0.2$                          & \textit{Symflip}, $\chi=0.2$                          & \textit{Symflip}, $\chi=0.4$                          \\
            \midrule
            FedAvg                     & $54.75$ \small{\transparent{0.5} $\pm 1.45$}          & $52.35$ \small{\transparent{0.5} $\pm 1.65$}          & $52.60$ \small{\transparent{0.5} $\pm 0.50$}          & $41.80$ \small{\transparent{0.5} $\pm 0.50$}          \\
            FeSEM                      & $32.60$ \small{\transparent{0.5} $\pm 1.30$}          & $35.25$ \small{\transparent{0.5} $\pm 2.95$}          & $32.40$ \small{\transparent{0.5} $\pm 2.80$}          & $29.70$ \small{\transparent{0.5} $\pm 0.01$}          \\
            IFCA                       & $24.95$ \small{\transparent{0.5} $\pm 7.05$}          & $20.55$ \small{\transparent{0.5} $\pm 4.65$}          & $30.35$ \small{\transparent{0.5} $\pm 2.05$}          & $36.05$ \small{\transparent{0.5} $\pm 4.45$}          \\
            FedEM                      & $64.40$ \small{\transparent{0.5} $\pm 1.10$}          & $57.55$ \small{\transparent{0.5} $\pm 2.95$}          & $53.00$ \small{\transparent{0.5} $\pm 1.90$}          & $43.10$ \small{\transparent{0.5} $\pm 0.20$}          \\
            FedRC                      & $67.90$ \small{\transparent{0.5} $\pm 1.00$}          & $59.95$ \small{\transparent{0.5} $\pm 1.05$}          & $55.25$ \small{\transparent{0.5} $\pm 2.05$}          & $42.00$ \small{\transparent{0.5} $\pm 0.40$}          \\
            \algfednew                 & $\mathbf{66.70}$ \small{\transparent{0.5} $\pm 0.40$} & $\mathbf{62.70}$ \small{\transparent{0.5} $\pm 0.30$} & $\mathbf{59.95}$ \small{\transparent{0.5} $\pm 1.15$} & $\mathbf{47.20}$ \small{\transparent{0.5} $\pm 0.20$} \\
            \bottomrule
        \end{tabular}
    }
    \label{tab:results on noisy data scenario}
\end{table}

\paragraph{Additional results on diverse distribution shift scenarios.} In Table~\ref{tab:performance-beta-04}, we show the performance of algorithms with $\beta = 0.4$. Results show \algfednew always achieve the best test accuracy, and achieve a good local-global balance.

\begin{table}[!t]
    \centering
    \caption{\textbf{Performance of the adaptive clustering methods.} We evaluated algorithm performance on CIFAR10 and CIFAR100 datasets, employing 100 clients. For each algorithm, we present the highest validation and test accuracies across 200 communication rounds, and the final number of clusters during training denoted as $K^{T}$. All experiments utilized MobileNet-V2~\citep{sandler2018mobilenetv2}.}
    \resizebox{1.0\textwidth}{!}{

        \begin{tabular}{l c c c c c c c c c c c c}
            \toprule
            \multirow{2}{*}{Algorithm}                   & \multicolumn{3}{c}{CIFAR10, $\beta = 0.4$}   & \multicolumn{3}{c}{CIFAR100, $\beta = 0.4$}                                                                                                                    \\
            \cmidrule(lr){2-4} \cmidrule(lr){5-7}
                                                         & Val                                          & Test                                         & $K^{T}$ & Val                                          & Test                                         & $K^{T}$ \\
            \midrule
            FedAvg                                       & $48.16$ \small{\transparent{0.5} $\pm 1.64$} & $49.93$ \small{\transparent{0.5} $\pm 0.80$} & 3.0     & $22.77$ \small{\transparent{0.5} $\pm 0.01$} & $24.62$ \small{\transparent{0.5} $\pm 0.55$} & 3.0     \\
            FeSEM                                        & $46.08$ \small{\transparent{0.5} $\pm 4.54$} & $35.99$ \small{\transparent{0.5} $\pm 4.59$} & 3.0     & $23.56$ \small{\transparent{0.5} $\pm 1.52$} & $22.31$ \small{\transparent{0.5} $\pm 1.08$} & 3.0     \\
            IFCA                                         & $36.15$ \small{\transparent{0.5} $\pm 3.45$} & $24.79$ \small{\transparent{0.5} $\pm 1.18$} & 3.0     & $27.72$ \small{\transparent{0.5} $\pm 0.82$} & $21.37$ \small{\transparent{0.5} $\pm 1.33$} & 3.0     \\
            FedEM                                        & $60.26$ \small{\transparent{0.5} $\pm 1.10$} & $54.44$ \small{\transparent{0.5} $\pm 0.04$} & 3.0     & $25.80$ \small{\transparent{0.5} $\pm 0.20$} & $22.88$ \small{\transparent{0.5} $\pm 0.19$} & 3.0     \\
            FedRC                                        & $57.99$ \small{\transparent{0.5} $\pm 0.29$} & $56.75$ \small{\transparent{0.5} $\pm 0.38$} & 3.0     & $30.94$ \small{\transparent{0.5} $\pm 0.88$} & $31.63$ \small{\transparent{0.5} $\pm 0.20$} & 3.0     \\
            \midrule
            CFL                                                                                                                                                                                                                                                          \\
            $\quad \text{tol}_1 =0.4, \text{tol}_2 =1.6$ & $61.86$ \small{\transparent{0.5} $\pm 5.29$} & $51.15$ \small{\transparent{0.5} $\pm 0.82$} & 6.0     & $34.11$ \small{\transparent{0.5} $\pm 6.35$} & $21.04$ \small{\transparent{0.5} $\pm 2.21$} & 5.0
            \\
            $\quad \text{tol}_1 =0.4, \text{tol}_2 =0.8$ & $60.42$ \small{\transparent{0.5} $\pm 0.31$} & $41.59$ \small{\transparent{0.5} $\pm 2.14$} & 8.0     & $36.23$ \small{\transparent{0.5} $\pm 3.58$} & $16.03$ \small{\transparent{0.5} $\pm 2.69$} & 6.0
            \\
            $\quad \text{tol}_1 =0.2, \text{tol}_2 =0.8$ & $49.14$ \small{\transparent{0.5} $\pm 6.11$} & $49.88$ \small{\transparent{0.5} $\pm 4.21$} & 3.0     & $34.20$ \small{\transparent{0.5} $\pm 7.13$} & $26.42$ \small{\transparent{0.5} $\pm 0.73$} & 2.5
            \\
            ICFL                                                                                                                                                                                                                                                         \\
            $\quad \alpha^{*} (0) =0.85$                 & $77.73$ \small{\transparent{0.5} $\pm 0.47$} & $52.03$ \small{\transparent{0.5} $\pm 0.10$} & 100.0   & $49.71$ \small{\transparent{0.5} $\pm 0.55$} & $28.55$ \small{\transparent{0.5} $\pm 0.03$} & 100.0   \\
            $\quad \alpha^{*} (0) =0.98$                 & $63.69$ \small{\transparent{0.5} $\pm 3.58$} & $54.02$ \small{\transparent{0.5} $\pm 1.11$} & 81.5    & $45.72$ \small{\transparent{0.5} $\pm 1.10$} & $28.45$ \small{\transparent{0.5} $\pm 0.82$} & 70.0
            \\
            StoCFL                                                                                                                                                                                                                                                       \\
            $\quad \tau=0.00$                            & $48.55$ \small{\transparent{0.5} $\pm 0.95$} & $51.25$ \small{\transparent{0.5} $\pm 1.16$} & 1.5     & $24.50$ \small{\transparent{0.5} $\pm 0.03$} & $25.70$ \small{\transparent{0.5} $\pm 1.51$} & 1.0     \\
            $\quad \tau=0.05$                            & $57.84$ \small{\transparent{0.5} $\pm 2.26$} & $50.42$ \small{\transparent{0.5} $\pm 0.97$} & 20.5    & $26.24$ \small{\transparent{0.5} $\pm 1.46$} & $26.60$ \small{\transparent{0.5} $\pm 1.17$} & 4.0     \\
            $\quad \tau=0.10$                            & $72.91$ \small{\transparent{0.5} $\pm 2.25$} & $47.84$ \small{\transparent{0.5} $\pm 2.60$} & 59.0    & $67.67$ \small{\transparent{0.5} $\pm 1.68$} & $9.89$ \small{\transparent{0.5} $\pm 0.45$}  & 86.0    \\
            $\quad \tau=0.15$                            & $77.19$ \small{\transparent{0.5} $\pm 2.31$} & $41.49$ \small{\transparent{0.5} $\pm 0.97$} & 92.0    & $70.13$ \small{\transparent{0.5} $\pm 0.27$} & $7.77$ \small{\transparent{0.5} $\pm 0.23$}  & 94.0    \\
            \midrule
            \algfednew (FeSEM)                                                                                                                                                                                                                                           \\
            \quad $\rho=0.1$                             & $85.30$ \small{\transparent{0.5} $\pm 1.05$} & $45.20$ \small{\transparent{0.5} $\pm 0.28$} & $47.0$  & $58.61$ \small{\transparent{0.5} $\pm 4.14$} & $18.29$ \small{\transparent{0.5} $\pm 2.38$} & $35.5$  \\
            \quad $\rho=0.3$                             & $80.34$ \small{\transparent{0.5} $\pm 1.33$} & $48.25$ \small{\transparent{0.5} $\pm 2.72$} & 20.5    & $44.65$ \small{\transparent{0.5} $\pm 0.35$} & $21.73$ \small{\transparent{0.5} $\pm 1.27$} & 12.0    \\
            \algfednew (FedEM)                                                                                                                                                                                                                                           \\
            \quad $\rho=0.05$                            & $80.31$ \small{\transparent{0.5} $\pm 1.60$} & $53.62$ \small{\transparent{0.5} $\pm 4.36$} & $18.5$  & $62.19$ \small{\transparent{0.5} $\pm 1.54$} & $21.15$ \small{\transparent{0.5} $\pm 0.88$} & $44.5$  \\
            \quad $\rho=0.1$                             & $82.89$ \small{\transparent{0.5} $\pm 0.92$} & $56.27$ \small{\transparent{0.5} $\pm 1.08$} & 26.5    & $59.08$ \small{\transparent{0.5} $\pm 0.06$} & $21.29$ \small{\transparent{0.5} $\pm 0.87$} & 31.5    \\
            \quad $\rho=0.3$                             & $80.72$ \small{\transparent{0.5} $\pm 1.90$} & $55.77$ \small{\transparent{0.5} $\pm 1.93$} & 10.0    & $49.84$ \small{\transparent{0.5} $\pm 6.85$} & $28.62$ \small{\transparent{0.5} $\pm 0.78$} & 11.0    \\
            \algfednew (FedRC)                                                                                                                                                                                                                                           \\
            \quad $\rho=0.05$                            & $68.48$ \small{\transparent{0.5} $\pm 0.25$} & $66.77$ \small{\transparent{0.5} $\pm 0.28$} & 9.5     & $38.75$ \small{\transparent{0.5} $\pm 0.98$} & $30.45$ \small{\transparent{0.5} $\pm 0.07$} & 10.0    \\
            \quad $\rho=0.1$                             & $68.56$ \small{\transparent{0.5} $\pm 3.56$} & $65.75$ \small{\transparent{0.5} $\pm 5.40$} & 6.0     & $40.30$ \small{\transparent{0.5} $\pm 1.19$} & $30.23$ \small{\transparent{0.5} $\pm 0.85$} & 11.0    \\
            \quad $\rho=0.3$                             & $70.86$ \small{\transparent{0.5} $\pm 0.31$} & $70.13$ \small{\transparent{0.5} $\pm 0.42$} & 5.5     & $39.62$ \small{\transparent{0.5} $\pm 0.34$} & $32.22$ \small{\transparent{0.5} $\pm 0.20$} & 5.0     \\
            \bottomrule
        \end{tabular}
    }
    \label{tab:performance-beta-04}
\end{table}

\begin{table*}[!t]
    \centering
    \caption{\textbf{Ablation studies on techniques in Sec~\ref{sec:inconsistency-aware-objective}.} We evaluated algorithm performance on CIFAR10 and CIFAR100 datasets, showcasing the top Validation and Test accuracies for each. We kept $\rho = 0.3$ consistent across all algorithms and varied $\tilde{\mu}$ to adjust the penalty term's strength in the objective function. The best results in each block are highlighted.}
    \resizebox{1.0\textwidth}{!}{
        \begin{tabular}{l c c c c c c c c}
            \toprule
            \multirow{2}{*}{Algorithm} & \multicolumn{2}{c}{CIFAR10, $\beta = 0.2$}            & \multicolumn{2}{c}{CIFAR10, $\beta = 0.4$}            & \multicolumn{2}{c}{CIFAR100, $\beta = 0.2$}           & \multicolumn{2}{c}{CIFAR100, $\beta = 0.4$}                                                                                                                                                                                                                                           \\
            \cmidrule(lr){2-3} \cmidrule(lr){4-5} \cmidrule(lr){6-7} \cmidrule(lr){8-9}
                                       & Val                                                   & Test                                                  & Val                                                   & Test                                                  & Val                                                   & Test                                                  & Val                                                   & Test                                                  \\
            \midrule
            \algfednew (FedEM)                                                                                                                                                                                                                                                                                                                                                                                                                                                                         \\
            \quad $\tilde{\mu} = 0.0$  & $\mathbf{83.67}$ \small{\transparent{0.5} $\pm 0.72$} & $\mathbf{62.43}$ \small{\transparent{0.5} $\pm 0.71$} & $\mathbf{80.72}$ \small{\transparent{0.5} $\pm 1.90$} & $\mathbf{55.77}$ \small{\transparent{0.5} $\pm 1.93$} & $\mathbf{50.72}$ \small{\transparent{0.5} $\pm 2.97$} & $\mathbf{32.13}$ \small{\transparent{0.5} $\pm 0.18$} & $\mathbf{49.84}$ \small{\transparent{0.5} $\pm 6.85$} & $\mathbf{28.62}$ \small{\transparent{0.5} $\pm 0.78$} \\
            \quad $\tilde{\mu} = 0.1$  & $81.60$ \small{\transparent{0.5} $\pm 0.59$}          & $60.48$ \small{\transparent{0.5} $\pm 0.50$}          & $80.36$ \small{\transparent{0.5} $\pm 2.40$}          & $55.10$ \small{\transparent{0.5} $\pm 1.75$}          & $48.78$ \small{\transparent{0.5} $\pm 0.62$}          & $30.50$ \small{\transparent{0.5} $\pm 0.33$}          & $48.56$ \small{\transparent{0.5} $\pm 1.10$}          & $25.80$ \small{\transparent{0.5} $\pm 1.17$}          \\
            \quad $\tilde{\mu} = 0.4$  & $79.52$ \small{\transparent{0.5} $\pm 0.11$}          & $53.33$ \small{\transparent{0.5} $\pm 2.97$}          & $76.50$ \small{\transparent{0.5} $\pm 0.34$}          & $49.97$ \small{\transparent{0.5} $\pm 2.26$}          & $44.85$ \small{\transparent{0.5} $\pm 0.48$}          & $28.39$ \small{\transparent{0.5} $\pm 0.12$}          & $41.52$ \small{\transparent{0.5} $\pm 0.08$}          & $22.83$ \small{\transparent{0.5} $\pm 0.42$}          \\
            \midrule
            \algfednew (FedRC)                                                                                                                                                                                                                                                                                                                                                                                                                                                                         \\
            \quad $\tilde{\mu} = 0.0$  & $\mathbf{70.82}$ \small{\transparent{0.5} $\pm 0.25$} & $69.15$ \small{\transparent{0.5} $\pm 0.35$}          & $69.95$ \small{\transparent{0.5} $\pm 1.99$}          & $67.09$ \small{\transparent{0.5} $\pm 1.01$}          & $39.55$ \small{\transparent{0.5} $\pm 1.29$}          & $35.49$ \small{\transparent{0.5} $\pm 0.16$}          & $38.77$ \small{\transparent{0.5} $\pm 1.20$}          & $31.87$ \small{\transparent{0.5} $\pm 1.13$}          \\
            \quad $\tilde{\mu} = 0.1$  & $69.91$ \small{\transparent{0.5} $\pm 0.16$}          & $68.77$ \small{\transparent{0.5} $\pm 1.56$}          & $69.53$ \small{\transparent{0.5} $\pm 0.21$}          & $68.54$ \small{\transparent{0.5} $\pm 1.08$}          & $39.38$ \small{\transparent{0.5} $\pm 0.40$}          & $35.95$ \small{\transparent{0.5} $\pm 0.59$}          & $\mathbf{39.77}$ \small{\transparent{0.5} $\pm 2.33$} & $31.52$ \small{\transparent{0.5} $\pm 0.45$}          \\
            \quad $\tilde{\mu} = 0.4$  & $69.33$ \small{\transparent{0.5} $\pm 0.24$}          & $\mathbf{69.67}$ \small{\transparent{0.5} $\pm 1.27$} & $\mathbf{70.86}$ \small{\transparent{0.5} $\pm 0.31$} & $\mathbf{70.13}$ \small{\transparent{0.5} $\pm 0.42$} & $\mathbf{39.97}$ \small{\transparent{0.5} $\pm 0.21$} & $\mathbf{36.50}$ \small{\transparent{0.5} $\pm 0.28$} & $39.62$ \small{\transparent{0.5} $\pm 0.34$}          & $\mathbf{32.22}$ \small{\transparent{0.5} $\pm 0.20$} \\
            \bottomrule
        \end{tabular}
    }
    \label{tab:ablation-studies-on-tier-1}
\end{table*}

\begin{table*}[!t]
    \centering
    \caption{\textbf{Ablation studies on techniques in Sec~\ref{sec:cover-soft-cluster}.} We evaluated algorithm performance on CIFAR10 and CIFAR100 datasets, displaying their highest Validation and Test accuracies. We kept $\rho$ consistent at 0.3 for all algorithms. "w/ SCWU" denotes the use of soft clustering weight updating mechanisms designed in Section~\ref{sec:cover-soft-cluster}.}
    \resizebox{1.0\textwidth}{!}{
        \begin{tabular}{l c c c c c c c c}
            \toprule
            \multirow{2}{*}{Algorithm} & \multicolumn{2}{c}{CIFAR10, $\beta = 0.2$}            & \multicolumn{2}{c}{CIFAR10, $\beta = 0.4$}            & \multicolumn{2}{c}{CIFAR100, $\beta = 0.2$}           & \multicolumn{2}{c}{CIFAR100, $\beta = 0.4$}                                                                                                                                                                                                                                           \\
            \cmidrule(lr){2-3} \cmidrule(lr){4-5} \cmidrule(lr){6-7} \cmidrule(lr){8-9}
                                       & Val                                                   & Test                                                  & Val                                                   & Test                                                  & Val                                                   & Test                                                  & Val                                                   & Test                                                  \\
            \midrule
            \algfednew (FedEM)                                                                                                                                                                                                                                                                                                                                                                                                                                                                         \\
            \quad w/ SCWU              & $\mathbf{83.67}$ \small{\transparent{0.5} $\pm 0.72$} & $62.43$ \small{\transparent{0.5} $\pm 0.71$}          & $\mathbf{80.72}$ \small{\transparent{0.5} $\pm 1.90$} & $55.77$ \small{\transparent{0.5} $\pm 1.93$}          & $\mathbf{50.72}$ \small{\transparent{0.5} $\pm 2.97$} & $32.13$ \small{\transparent{0.5} $\pm 0.18$}          & $\mathbf{49.84}$ \small{\transparent{0.5} $\pm 6.85$} & $28.62$ \small{\transparent{0.5} $\pm 0.78$}          \\
            \quad w/o SCWU             & $82.11$ \small{\transparent{0.5} $\pm 2.39$}          & $63.84$ \small{\transparent{0.5} $\pm 0.19$}          & $80.08$ \small{\transparent{0.5} $\pm 0.99$}          & $58.83$ \small{\transparent{0.5} $\pm 2.12$}          & $49.77$ \small{\transparent{0.5} $\pm 1.93$}          & $32.90$ \small{\transparent{0.5} $\pm 1.11$}          & $47.91$ \small{\transparent{0.5} $\pm 2.67$}          & $27.40$ \small{\transparent{0.5} $\pm 1.17$}          \\
            \midrule
            \algfednew (FedRC)                                                                                                                                                                                                                                                                                                                                                                                                                                                                         \\
            \quad w/ SCWU              & $69.33$ \small{\transparent{0.5} $\pm 0.24$}          & $\mathbf{69.67}$ \small{\transparent{0.5} $\pm 1.27$} & $70.86$ \small{\transparent{0.5} $\pm 0.31$}          & $\mathbf{70.13}$ \small{\transparent{0.5} $\pm 0.42$} & $39.97$ \small{\transparent{0.5} $\pm 0.21$}          & $\mathbf{36.50}$ \small{\transparent{0.5} $\pm 0.28$} & $39.62$ \small{\transparent{0.5} $\pm 0.34$}          & $\mathbf{32.22}$ \small{\transparent{0.5} $\pm 0.20$} \\
            \quad w/o SCWU             & $69.88$ \small{\transparent{0.5} $\pm 0.30$}          & $68.83$ \small{\transparent{0.5} $\pm 0.71$}          & $70.77$ \small{\transparent{0.5} $\pm 0.47$}          & $68.87$ \small{\transparent{0.5} $\pm 0.23$}          & $40.96$ \small{\transparent{0.5} $\pm 1.24$}          & $35.72$ \small{\transparent{0.5} $\pm 1.01$}          & $39.18$ \small{\transparent{0.5} $\pm 0.13$}          & $32.08$ \small{\transparent{0.5} $\pm 0.78$}          \\
            \bottomrule
        \end{tabular}
    }
    \label{tab:ablation-studies-on-tier-2}
\end{table*}

\begin{table}[!t]
    \centering
    \caption{\textbf{Performance of algorithms with Resnet18.} We evaluated algorithm performance on CIFAR10 datasets with $\beta = 0.2$, displaying their highest Validation and Test accuracies. All algorithms utilize ResNet18 and run for 200 communication rounds.}
    \begin{tabular}{l c c c}
        \toprule
        Algorithm                                      & Val                                          & Test                                         \\
        \midrule
        CFL                                                                                                                                          \\
        \quad $\text{tol}_1 = 0.4, \text{tol}_2 = 0.6$ & $63.07$ \small{\transparent{0.5} $\pm 7.42$} & $53.65$ \small{\transparent{0.5} $\pm 2.33$} \\
        \quad $\text{tol}_1 = 0.4, \text{tol}_2 = 0.8$ & $61.14$ \small{\transparent{0.5} $\pm 1.87$} & $54.87$ \small{\transparent{0.5} $\pm 1.32$} \\
        ICFL                                                                                                                                         \\
        \quad $\alpha^{*}(0) = 0.85$                   & $80.46$ \small{\transparent{0.5} $\pm 0.99$} & $45.28$ \small{\transparent{0.5} $\pm 6.56$} \\
        \quad $\alpha^{*}(0) = 0.98$                   & $82.34$ \small{\transparent{0.5} $\pm 0.28$} & $44.08$ \small{\transparent{0.5} $\pm 0.40$} \\
        StoCFL                                                                                                                                       \\
        \quad $\tau = 0.1$                             & $57.41$ \small{\transparent{0.5} $\pm 6.69$} & $48.95$ \small{\transparent{0.5} $\pm 1.95$} \\
        \quad $\tau = 0.15$                            & $66.54$ \small{\transparent{0.5} $\pm 1.05$} & $47.77$ \small{\transparent{0.5} $\pm 0.14$} \\
        \midrule
        \algfednew (FeSEM)                                                                                                                           \\
        \quad $\rho = 0.05$                            & $86.90$ \small{\transparent{0.5} $\pm 0.20$} & $50.34$ \small{\transparent{0.5} $\pm 5.99$} \\
        \quad $\rho = 0.1$                             & $85.55$ \small{\transparent{0.5} $\pm 0.24$} & $49.38$ \small{\transparent{0.5} $\pm 6.15$} \\
        \algfednew (FedEM)                                                                                                                           \\
        \quad $\rho = 0.05$                            & $83.88$ \small{\transparent{0.5} $\pm 0.25$} & $58.92$ \small{\transparent{0.5} $\pm 1.11$} \\
        \quad $\rho = 0.1$                             & $83.83$ \small{\transparent{0.5} $\pm 0.01$} & $60.27$ \small{\transparent{0.5} $\pm 3.11$} \\
        \algfednew (FedRC)                                                                                                                           \\
        \quad $\rho = 0.05$                            & $67.72$ \small{\transparent{0.5} $\pm 1.30$} & $64.13$ \small{\transparent{0.5} $\pm 0.37$} \\
        \quad $\rho = 0.1$                             & $67.51$ \small{\transparent{0.5} $\pm 0.24$} & $63.15$ \small{\transparent{0.5} $\pm 0.78$} \\
        \bottomrule
    \end{tabular}
    \label{tab:resnet}
\end{table}

\begin{table}[!t]
    \centering
    \caption{\small \textbf{More details} about the ICFL results.}
        \begin{tabular}{l c c c c c}
            \toprule
            ICFL               & CIFAR100 (Val) & CIFAR100 (Test) \\
            \midrule
            Round on Best Val  & 52.73          & 6.40            \\
            Round on Best Test & 29.22          & 32.77           \\
            \bottomrule
        \end{tabular}
    \label{tab:detail-icfl}
\end{table}

\begin{table}[!t]
    \centering
    \vspace{-1em}
    \caption{\small \textbf{Ablation studied on only using Feature extractor-classifier split mechnism.} HCFL$'$ uses the feature-extractor classifier split mechanism and the optimization steps outlined in Eq~\eqref{equ:uptate-gamma}-~\eqref{equ:update-mphi}.}
    \resizebox{.8\textwidth}{!}{
        \begin{tabular}{l c c c c c}
            \toprule
            Algorithms           & CIFAR10 (Val) & CIFAR10 (Test) & CIFAR100 (Val) & CIFAR100 (Test) \\
            \midrule
            FedEM                & 66.49         & 53.64          & 29.75          & 24.18           \\
            FedRC                & 63.65         & 59.41          & 34.56          & 37.62           \\
            FedSoft              & 83.62         & 34.70          & 73.58          & 4.60            \\
            HCFL${'}$ (FedEM)    & 61.24         & 59.83          & 42.44          & 28.13           \\
            HCFL${'}$ (FedRC)    & 63.58         & 63.57          & 37.46          & 37.27           \\
            HCFL${'}$ (FedSoft)  & 85.76         & 52.20          & 76.70          & 13.47           \\
            HCFL$^{+}$ (FedEM)   & 83.67         & 62.43          & 50.72          & 32.13           \\
            HCFL$^{+}$ (FedRC)   & 69.33         & 69.67          & 39.97          & 36.50           \\
            HCFL$^{+}$ (FedSoft) & 85.92         & 51.23          & 76.28          & 12.50           \\
            \bottomrule
        \end{tabular}
    }
    \label{tab:ablation-feature-extractor-split-only}
\end{table}

\textbf{Ablation studies on only using feature extractor-classifier split mechanism.} In Table~\ref{tab:ablation-feature-extractor-split-only}, we present a new ablation study HCFL${'}$, only using the feature-extractor classifier split method on top of the naive HCFL framework (Using backbone algorithms for Tiers 1 \& 2, and CFL for Tiers 3 \& 4). The results indicate HCFL${'}$ outperforming the original algorithms, notably on CIFAR100~\footnote{Evaluating the HCFL without the split is challenging due to the 200+ million trainable parameters, but this can be reduced to 10 million with the split.}.

\textbf{Number of clusters in \algfednew over communication rounds.} Figure~\ref{fig:cluster-number} shows how the number of clusters evolves across communication rounds for different $\rho$ values using the CIFAR-10 dataset in our experiments. It is evident that \algfednew allows the algorithm to automatically determine the optimal number of clusters during training.

\begin{figure}
    \centering
    \begin{subfigure}[b]{.45\textwidth}
        \includegraphics[width=1.\textwidth]{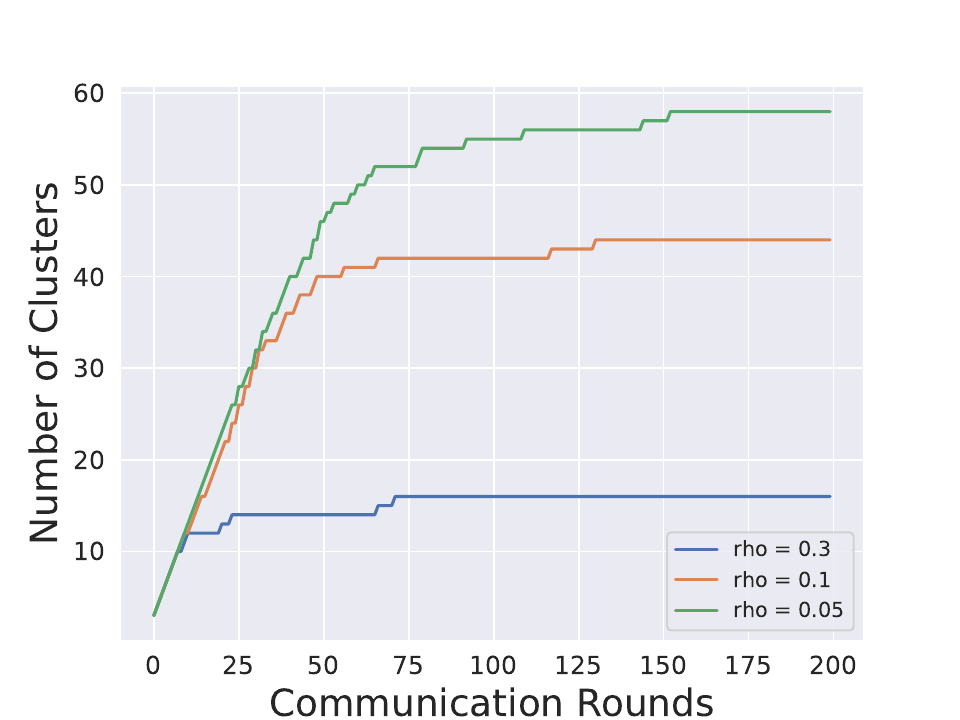}
        \caption{\algfednew (FeSEM)}
    \end{subfigure}
    \begin{subfigure}[b]{.45\textwidth}
        \includegraphics[width=1.\textwidth]{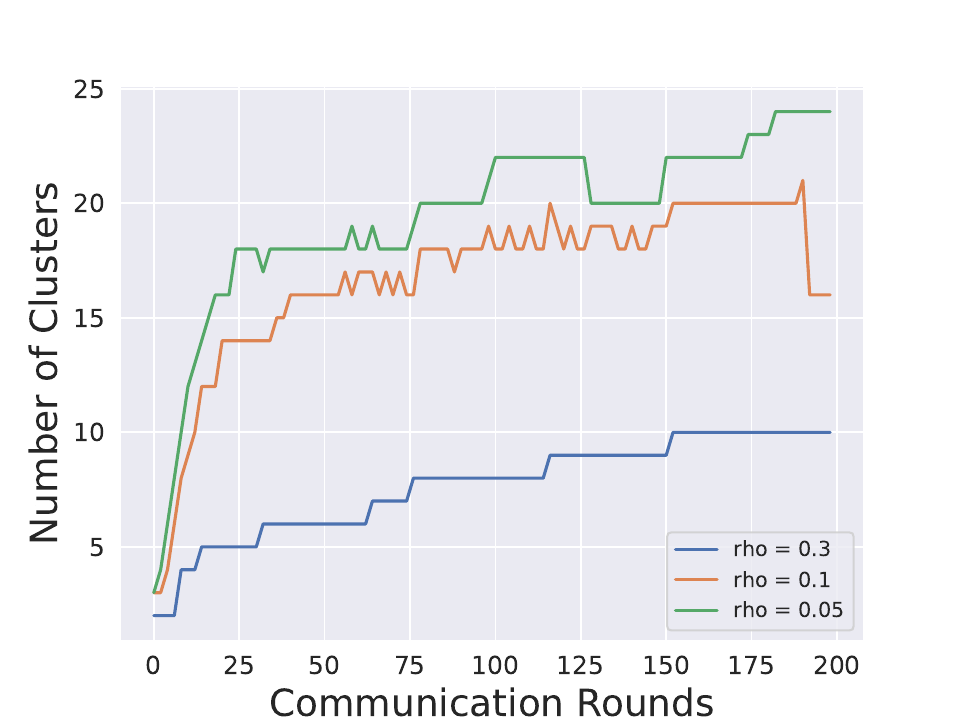}
        \caption{\algfednew (FedEM)}
    \end{subfigure}
    \caption{\small
        \textbf{Number of clusters in \algfednew over communication rounds.} We illustrate changes in cluster numbers across communication rounds for various $\rho$ values using the CIFAR-10 dataset in our experiments.\looseness=-1
    }
    \vspace{-1.5em}
    \label{fig:cluster-number}
    \end{figure}

\begin{table}
    \centering
    \caption{\small \textbf{Efficiency comparison.} The CIFAR-10 dataset is divided among 100 clients, with each client holding data from 2 classes. The hyperparameters are chosen based on the best-performing configuration from Table~\ref{tab:performance-beta-02}.}
    \label{tab:efficiency-comparision}
    \resizebox{.8\textwidth}{!}{
        \begin{tabular}{l c c c c c}
            \toprule
            Algorithms & Memory (M) & Simulation Time (s/it) & Final Cluster Number & Val & Test \\
            \midrule
            ICFL ($\alpha(0) = 0.85$) &	4190 & 124.71 & 100 & 83.96 & 46.00 \\
            StoCFL ($\tau = 0.15$) & 2834 & 178.53 & 41 & 86.92 & 20.00 \\
            \algfednew (FedEM, $\rho = 0.1$) & 2564 & 176.48 & 5 & 95.14 & 49.00 \\
            \bottomrule
        \end{tabular}
    }
    
\end{table}

\textbf{Efficiency comparison.} In Table~\ref{tab:efficiency-comparision}, we compare the efficiency of \algfednew with baseline methods. Notably, \algfednew (FedEM) demonstrates superior performance while using a significantly reduced number of clusters, resulting in lower simulation time and memory usage.

\begin{table*}[!t]
    \centering
    \caption{\textbf{Performance Under Various Heterogeneous Settings.} We evaluated the algorithm’s performance on the CIFAR-10 dataset, presenting the top Validation and Test accuracies for each setting. Results are reported with $\beta = 0.8$. Additionally, $C = 2$ and $C = 4$ indicate that each client contains data from 2 and 4 classes, respectively.}
        \begin{tabular}{l c c c c c c c c}
            \toprule
            \multirow{2}{*}{Algorithm} & \multicolumn{2}{c}{$\beta = 0.9$}            & \multicolumn{2}{c}{$C = 2$}            & \multicolumn{2}{c}{$C = 4$}                                                                                                                                                                                                                                \\
            \cmidrule(lr){2-3} \cmidrule(lr){4-5} \cmidrule(lr){6-7}
                                       & Val                                                   & Test                                                  & Val                                                   & Test                                                  & Val                                                   & Test                                                  \\
            \midrule
            ICFL &	72.94 &	24.98 &	83.96 &	46.00 &	72.34 &	72.30 \\
            stoCFL & 60.0 & 11.97 & 86.92 & 20.00 & 70.38 & 32.60 \\
            \algfednew(FedEM) & 73.68 & 28.63 & 95.14 & 49.00 & 89.80 & 62.30 \\
            \bottomrule
        \end{tabular}
    \label{tab:performance-various-heterogeneity}
\end{table*}

\paragraph{Performance of algorithms under various heterogeneous settings.} In Table~\ref{tab:performance-various-heterogeneity}, we show that \algfednew consistently outperforms the baseline algorithms across all settings, demonstrating its reliable performance improvement.

\begin{table}
    \centering
    \caption{\small \textbf{Clustering quality illustration.} We present the clustering results of \algfednew in the $C=2$ setting, where each client is assigned data from two classes. Specifically, \algfednew generates 5 clusters in this setting, and we report the number of samples from each class assigned to clusters 1 through 5.}
    \label{tab:cluster-quality}
        \begin{tabular}{l c c c c c}
            \toprule
            HCFL+(FedEM, $\rho=0.1$) & Cluster 1 & Cluster 2 & Cluster 3 & Cluster 4 & Cluster 5 \\
            \midrule
            Class 1 &	4483 &	0 &	0 &	0 &	0 \\
            Class 2 &	0 &	4502 &	0 &	0 &	0 \\
            Class 3 &	0 &	0 &	4464 &	0 &	0 \\
            Class 4 &	0 &	0 &	4536 &	0 &	0 \\
            Class 5 &	0 &	4498 &	0 &	0 &	0 \\
            Class 6 &	0 &	0 &	0 &	0 &	4483 \\
            Class 7 &	0 &	0 &	0 &	4491 &	0 \\
            Class 8 &	4517 &	0 &	0 &	0 &	0 \\
            Class 9 &	0 &	0 &	0 &	4509 &	0 \\
            Class 10 &	0 &	0 &	0 &	0 &	4517 \\
            \bottomrule
        \end{tabular}
    
\end{table}

\paragraph{Clustering quality illustration.} In Table~\ref{tab:cluster-quality}, we present the clustering results to evaluate the clustering quality. It is evident that \algfednew successfully identifies a relatively optimal clustering by: (1) assigning all samples with the same label to the same cluster; (2) avoiding the creation of an excessive number of clusters, as observed in stoCFL and ICFL; and (3) achieving higher validation and test accuracy compared to the baseline methods.

%% file: appendix-proof.tex
\section{Theoretical Study on Linear Representation Case}
\label{sec:Theoretical Study on Linear Representation Case}

\subsection{Definition and Assumptions}

In this section, we target on the effectiveness of split-feature-classifier method. Therefore, we focus on a case study that clients are solving a linear representation learning problem, similar to the analysis in \citep{collins2021exploiting,tziotis2022straggler}, and assume the optimal $\omega_{i,j;k}$ is already given.
In detail, we assume local data $\xx_{i,j} \in \R^{d}$, and $\mphi$ is a global shared projection onto a $c$-dimensional subspace $\R^{d}$, which is parameterized by matrix $\mB \in \R^{d \times c} $. Besides, for each underlying distribution, we have $\mtheta_k \in \R^{c}$, and the labels of each sample $\xx_{i,j}$ belongs to the distribution $k$ is given by $y_{i,j} = {\mtheta^{*}_{k}}^{T} {\mB^{*}}^{T} x_{i,j} + z_{k}$, where $\mtheta^{*}_{k}$ and $\mB^{*}$ are ground truth parameters, and $z_{k} \sim \cN (0, \sigma^{2})$ is to capture the heterogeneous between $K$ underlying distributions. Under these assumptions, the global empirical risk is
\begin{align}
    \min_{\mB, \mTheta} \frac{1}{2N} \sum_{i=1}^{M} \sum_{j=1}^{N_i} \left( y_{i,j} - \sum_{k=1}^{K} \omega_{i,j;k} \mtheta_{k}^{T} \mB^{T} \xx_{i,j} \right)^{2}
    \label{equ: linear objective function appendix}
\end{align}
The distance that measuring the distance between sub-spaces is defined as follows~\citep{jain2013low,collins2021exploiting,tziotis2022straggler},
\begin{definition} The principal angle distance between the column spaces of $\mB_1, \mB_2 \in \R^{d \times c}$ is given by,
\begin{align}
    \textstyle{dist} (\mB_1, \mB_2) = \norm{\hat{\mB}_{1, \perp}^{T} \hat{\mB}_2}_2 \, ,
\end{align}
where $\hat{\mB}_{1, \perp}$ and $\hat{\mB}_2$ are orthogonal matrices satisfying $span(\hat{\mB}_{1, \perp}) = span(\mB_1)^{\perp}$, and $span(\hat{\mB}_2) = span(\mB_2)$.
\end{definition}

\begin{definition}[$\norm{\mA}_2$-sub-Gaussian] For a random vector $\xx \in \R^{d}$ and a fixed matrix $\mA \in \R^{d \times c}$, the vector $\mA^{T} \xx$ is called $\norm{\mA}_2$-sub-Gaussian if $\yy^{T} \mA^{T} \xx$ is sub-Gaussian with sub-Gaussian norm $\norm{\mA}_2 \norm{\yy}_2$ for all $\yy \in \R^{c}$, i.e., $\Eb{\exp (\yy^{T} \mA^{T} \xx)} \le \exp(\norm{\mA}_2^{2} \norm{\yy}_2^{2} / 2)$.
\label{def: A sub-gaussian}
\end{definition}

\begin{assumption}[Sub-Gaussian design] The samples $\xx_{i,j} \in \R^{d}$ are i.i.d. with mean $\mathbf{0}$, covariance $\mI_d$, and are $\mI_d$-sub-Gaussian, i.e., $\Eb{e^{\vv^{T} \xx_{i,j}}} \le e^{\norm{\vv}_2^{2} / 2}$ for all $\vv \in \R^{d}$.
\label{ass: Sub-Gaussian design}
\end{assumption}

\begin{assumption}[Underlying distribution diversity] Let $\bar{\sigma}_{min, *}$ be the minimum singular value of any matrix $\bar{W} \in \R^{K \times c}$ with rows being a $K$-sized subset of ground-truth distribution-specific parameters $\{\mtheta_1, \cdots, \mtheta_K\}$. Then $\bar{\sigma}_{min, *} > 0$.
\label{ass: Underlying distribution diversity}
\end{assumption}

\begin{assumption}[Client normalization] The ground-truth distribution-specific parameters satisfy $\norm{\mtheta_{k}^{*}}_2 = \sqrt{c}$ for all $k \in [K]$, and $\mB^{*}$ has orthogonal columns.
\label{ass: Client normalization}
\end{assumption}

\begin{assumption}[Binary weights] We assume the value of $\omega_{i,j;k}$ is given in $\{0, 1\}$.
\label{ass: Binary weights}
\end{assumption}

Assumption~\ref{ass: Sub-Gaussian design}, ~\ref{ass: Underlying distribution diversity}, ~\ref{ass: Client normalization} are widely used assumptions in the analysis of linear representation learning problem
\citep{collins2021exploiting,tziotis2022straggler}. We introduce Assumption~\ref{ass: Binary weights} to simplify the analysis, and this assumption match the observation in our numerical experiments that the max weights $\max_{k} [\omega_{i,j;k}]$ usually close to $1$ when converge.

\subsection{Preliminary and Lemmas}

We first introduce some important lemmas that we would like to use in the following parts of proof. 

\begin{lemma}
Given orthogonal matrix $\hat{\mB} \in \R^{d \times c}$, matrix $\mX \in \R^{N \times d}$ that each row of $\mX$ is sub-Gaussian (Assumption~\ref{ass: Sub-Gaussian design}), and matrix $\bar{\mOmega}_k \in \R^{N \times d}$ valued by $\{ 0, 1 \}$. Then let $\delta_k = \frac{12 \cC_{1;k} c^{3/2} \sqrt{\log (M)}}{\sqrt{\hat{N}_k}}$, we have
\begin{align}
     \sigma_{\min} \left( \frac{1}{N}  \hat{\mB}^{T} \left( \bar{\mOmega}_k^{T} \odot \mX^{T} \right) \left( \mX \odot \bar{\mOmega}_k \right) \hat{\mB} \right)  \ge \frac{\hat{N}_k}{N} \left(1 - \delta_k \right) \, ,
\end{align}
with probability at least $1 - \exp (121 c^{3} \log (M))$, and $\hat{N}_k = \sum_{i=1}^{M} \sum_{j=1}^{N_i} \1_{\omega_{i,j;k} = 1}$.
\label{lem: lemma 1}
\end{lemma}

\begin{proof}
    Firstly, we can rewrite 
    \begin{align}
        \frac{1}{N}  \hat{\mB}^{T} \left( \bar{\mOmega}_k^{T} \odot \mX^{T} \right) \left( \mX \odot \bar{\mOmega}_k \right) \hat{\mB} 
        & = \frac{\hat{N}_k}{N} \sum_{i=1}^{M} \sum_{j=1}^{N_i} \left( \frac{\omega_{i,j;k}}{\sqrt{\hat{N}_k}} \hat{\mB}^{T} \xx_{i,j} \right) \left( \frac{\omega_{i,j;k}}{\sqrt{\hat{N}_k}} \hat{\mB}^{T} \xx_{i,j} \right)^{T} \, , \\
        & = \frac{\hat{N}_k}{N} \sum_{\omega_{i,j;k} = 1}^{N} \left( \frac{1}{\sqrt{\hat{N}_k}} \hat{\mB}^{T} \xx_{i,j} \right) \left( \frac{1}{\sqrt{\hat{N}_k}} \hat{\mB}^{T} \xx_{i,j} \right)^{T} \, ,
    \end{align}
where $\hat{N}_k = \sum_{i=1}^{M} \sum_{j=1}^{N_i} \1_{\omega_{i,j;k} = 1}$. Define $\vv_{i,j;k} = \frac{\omega_{i,j;k}}{\sqrt{\hat{N}_k}} \hat{\mB}^{T} \xx_{i,j}$, by Definition~\ref{def: A sub-gaussian} and Assumption~\ref{ass: Sub-Gaussian design}, we can observe that each  $\vv_{i,j;k}$ is either $\norm{\frac{1}{\sqrt{\hat{N}_k}} \hat{\mB}^{T}}_2$-sub-Gaussian or $\vv_{i,j;k} = \boldsymbol{0}$. Then we can define $\mV_k \in \R^{\hat{N}_k \times c}$, and each row of $\mV_k$ is $\vv_{i,j;k}$ that $\vv_{i,j;k} \not = 0$. Then we have
\begin{align}
    \frac{1}{N}  \hat{\mB}^{T} \left( \bar{\mOmega}_k^{T} \odot \mX^{T} \right) \left( \mX \odot \bar{\mOmega}_k \right) \hat{\mB} 
        & = \frac{\hat{N}_k}{N} \mV_k^{T} \mV_k \, .
\end{align}
Then based on Theorem 4.6.1, Equation (4.22) in \citep{vershynin2018high}, we have
\begin{align}
    \sigma_{min} \left( \frac{1}{N}  \hat{\mB}^{T} \left( \bar{\mOmega}_k^{T} \odot \mX^{T} \right) \left( \mX \odot \bar{\mOmega}_k \right) \hat{\mB}  \right) 
    = \frac{\hat{N}_k}{N} \sigma_{min} \left( \mV_k^{T} \mV_k \right) 
    \ge \frac{\hat{N}_k}{N} \left(1 - \delta_k \right)
\end{align} \, , 
with probability at least $1 - \exp(-\left( \delta_k \sqrt{N_k} / \cC_{1;k} - \sqrt{C} \right)^2)$ for constant $\cC_{1;k}$. We then set $\delta_k = \frac{12 \cC_{1;k} c^{3/2} \sqrt{\log (M)}}{\sqrt{\hat{N}_k}}$, and we have
\begin{align}
    1 - \exp(-\left( \delta_k \sqrt{N_k} / \cC_{1;k} - \sqrt{C} \right)^2)
    \ge
    1 - \exp (121 c^{3} \log (M)) \, ,
\end{align}
which finish the proof.
\end{proof}

\begin{lemma}
Given the objective function defined in Equation~\eqref{equ: linear objective function appendix}
\begin{align}
    \min_{\mB, \mTheta} \frac{1}{2N} \sum_{i=1}^{M} \sum_{j=1}^{N_i} \left( y_{i,j} - \sum_{k=1}^{K} \omega_{i,j;k} \mtheta_{k}^{T} \mB^{T} \xx_{i,j} \right)^{2} \, ,
\end{align}
define the matrix form of $\xx_{i,j}$, $y_{i,j}$, and $\omega_{i,j;k}$ by $\mX \in \R^{N \times d}$, $\mY \in \R^{N}$, $\mOmega_k \in \R^{N}$, and define $\bar{\mOmega}_k \in \R^{N \times d}$ by repeat $\mOmega_k$ for $d$ times, we can have the following optimization steps to solve the above objective function
\begin{align}
    \mtheta_k^{t+1} & = \left( \frac{1}{N} [\hat{\mB}^{t}]^{T} \left( \bar{\mOmega}_k^{T} \odot \mX^{T} \right) \left( \mX \odot \bar{\mOmega}_k \right) {\hat{\mB}^{t}} \right)^{-1} \left( \frac{1}{N} [\hat{\mB}^{t}]^{T} \left( \bar{\mOmega}_k^{T} \odot \mX^{T} \right) \left(  \mY \odot \mOmega_k \right) \right) \, , \\
    \mB^{t+1} &= \hat{\mB}^{t} - \sum_{k=1}^{K} \frac{\eta}{N} \left( \left( \bar{\mOmega}_k^{T} \odot \mX^{T} \right) \left( \mX \odot \bar{\mOmega}_k \right) \hat{\mB} \mtheta_k - \left( \bar{\mOmega}_k^{T} \odot \mX^{T} \right) \left(  \mY \odot \mOmega_k \right) \right) \mtheta_k^{T} \, , \\
    \hat{\mB}^{t+1} & = \mB^{t+1} \left( \mR^{t+1} \right)^{-1} \, ,
\end{align}
where $\hat{\mB}^{t+1} \mR^{t+1}$ is the QR  factorization of $\mB^{t+1}$.
\label{lem: lemma 2}
\end{lemma}

\begin{proof}
    We first extend the empirical function via matrices $\mX_i \in \R^{N_i \times d}$, $\mY_i \in \R^{N_i}$, and $\mOmega_{i;k} \in \R^{N_i}$, $\bar{\mOmega}_{i;k} \in \R^{N_i \times d}$ is repeat $\mOmega_{i;k} \in \R^{N_i}$ for $d$ times.
\begin{align}
    \cL(\mB, \mTheta) = \frac{1}{2N} \sum_{i=1}^{M} \norm{ \sum_{k=1}^{K} \mOmega_{i;k} \odot (\mY_i - \mX_i \hat{\mB} \mtheta_k) }_{2}^2 \, .
\end{align}
Define $\mX \in \R^{N \times d} = [\mX_1^{T} , \cdots, \mX_M^{T}]^{T}$, $\bar{\mOmega}_{k} \in \R^{N \times d} = [\bar{\mOmega}_{1,k}^{T}, \cdots, \bar{\mOmega}_{M,k}^{T}]^{T}$, and $\mOmega_{k} \in \R^{N} = [\mOmega_{1,k}^{T}, \cdots, \mOmega_{M,k}^{T}]^{T}$, then compute the gradients of $\hat{B}$ and $\mtheta_k$,
\begin{align}
    \derive{\cL(\mB, \mTheta)}{\hat{B}} & = - \frac{1}{N} \sum_{i=1}^{M} \sum_{k=1}^{K} \left( \bar{\mOmega}_{i;k} \odot \mX_i \right)^{T} \left( \mOmega_{i;k} \odot (\mY_i - \mX_i \hat{\mB} \mtheta_k) \right) \mtheta_k^{T} \, , \\
    & = \frac{1}{N} \sum_{k=1}^{K} \left( \left( \bar{\mOmega}_k^{T} \odot \mX^{T} \right) \left( \mX \odot \bar{\mOmega}_k \right) \hat{\mB} \mtheta_k - \left( \bar{\mOmega}_k^{T} \odot \mX^{T} \right) \left(  \mY \odot \mOmega_k \right) \right) \mtheta_k^{T} \, ,
    \\
    \derive{\cL(\mB, \mTheta)}{\mtheta_k} & = - \frac{1}{N} \sum_{i=1}^{M} \left( \bar{\mOmega}_{i;k} \odot \mX_i \hat{\mB} \right)^{T} \left( \mOmega_{i;k} \odot (\mY_i - \mX_i \hat{\mB} \mtheta_k) \right) \, , \\
    & = \frac{1}{N} \left(  \hat{\mB}^{T} \left( \bar{\mOmega}_k^{T} \odot \mX^{T} \right) \left( \mX \odot \bar{\mOmega}_k \right) \hat{\mB} \mtheta_k - \hat{\mB}^{T} \left( \bar{\mOmega}_k^{T} \odot \mX^{T} \right) \left(  \mY \odot \mOmega_k \right) \right)  \, .
\end{align}

Then we can define the following optimization steps based on the above gradients, similar to the analysis in \citep{collins2021exploiting}.
\begin{align}
    \mtheta_k^{t+1} & = \left( \frac{1}{N} [\hat{\mB}^{t}]^{T} \left( \bar{\mOmega}_k^{T} \odot \mX^{T} \right) \left( \mX \odot \bar{\mOmega}_k \right) {\hat{\mB}^{t}} \right)^{-1} \left( \frac{1}{N} [\hat{\mB}^{t}]^{T} \left( \bar{\mOmega}_k^{T} \odot \mX^{T} \right) \left(  \mY \odot \mOmega_k \right) \right) \, , \\
    \mB^{t+1} &= \hat{\mB}^{t} - \sum_{k=1}^{K} \frac{\eta}{N} \left( \left( \bar{\mOmega}_k^{T} \odot \mX^{T} \right) \left( \mX \odot \bar{\mOmega}_k \right) \hat{\mB} \mtheta_k - \left( \bar{\mOmega}_k^{T} \odot \mX^{T} \right) \left(  \mY \odot \mOmega_k \right) \right) \mtheta_k^{T} \, , \\
    \hat{\mB}^{t+1} & = \mB^{t+1} \left( \mR^{t+1} \right)^{-1} \, ,
\end{align}
where $\hat{\mB}^{t+1} \mR^{t+1}$ is the QR  factorization of $\mB^{t+1}$. Then the $\hat{\mB}$ will be orthogonal matrix at the end of each optimization step. From Lemma~\ref{lem: lemma 1} we know that $\left( \frac{1}{N} [\hat{\mB}^{t}]^{T} \left( \bar{\mOmega}_k^{T} \odot \mX^{T} \right) \left( \mX \odot \bar{\mOmega}_k \right) {\hat{\mB}^{t}} \right)$ is invertible with high probability, which indicates the feasibility of the above optimization steps.
\end{proof}

\begin{lemma}
    The optimization step of $\mtheta_k$ can be expressed by the following equation
    \begin{align}
        \mtheta_{k} = \hat{\mB}^{T} \hat{\mB}^{*} \mtheta_k^{*} + \mF_k + \mG_k \, ,
    \end{align}
where $\mF_k, \mG_k \in \R^{c}$ are given in Equation~\eqref{equ: Fk} and~\eqref{equ: Gk}.
\label{lem: lemma 3}
\end{lemma}

\begin{proof}
    From Lemma~\ref{lem: lemma 2} we have
    \begin{align}
        \mtheta_k & = \left( \frac{1}{N} \hat{\mB}^{T} \left( \bar{\mOmega}_k^{T} \odot \mX^{T} \right) \left( \mX \odot \bar{\mOmega}_k \right) {\hat{\mB}} \right)^{-1} \left( \frac{1}{N} \hat{\mB}^{T} \left( \bar{\mOmega}_k^{T} \odot \mX^{T} \right) \left(  \mY \odot \mOmega_k \right) \right) \, .
    \end{align}
Based on the fact that $\mY = \sum_{k=1}^{K} \mOmega_{k} \odot \left( \mX \hat{B}^{*} \mtheta_k^{*} + \mZ \right)$, $\mOmega_k \odot \mOmega_k = \mOmega_k$, $\mOmega_k \odot \mOmega_{k^{'}} = \0$ we have
\begin{small}
\begin{align}
    \mtheta_k & = \left( \frac{1}{N} \hat{\mB}^{T} \left( \bar{\mOmega}_k^{T} \odot \mX^{T} \right) \left( \mX \odot \bar{\mOmega}_k \right) {\hat{\mB}} \right)^{-1} 
    \left( \frac{1}{N} \hat{\mB}^{T} \left( \bar{\mOmega}_k^{T} \odot \mX^{T} \right) \left(  \mY \odot \mOmega_k \right) \right) \, , \\
    & = \left( \frac{1}{N} \hat{\mB}^{T} \left( \bar{\mOmega}_k^{T} \odot \mX^{T} \right) \left( \mX \odot \bar{\mOmega}_k \right) {\hat{\mB}} \right)^{-1}
    \nonumber \\
    & \left( \frac{1}{N} \hat{\mB}^{T} \left( \bar{\mOmega}_k^{T} \odot \mX^{T} \right) \left(  \sum_{n=1}^{K} \mOmega_{k} \odot \mOmega_{n} \odot \left( \mX \hat{\mB}^{*} \mtheta_k^{*} + \mZ \right) \right) \right) \, , \\
    & = \left( \frac{1}{N} \hat{\mB}^{T} \left( \bar{\mOmega}_k^{T} \odot \mX^{T} \right) \left( \mX \odot \bar{\mOmega}_k \right) {\hat{\mB}} \right)^{-1}
    \left( \frac{1}{N}  \left( \hat{\mB}^{T} \left( \bar{\mOmega}_k^{T} \odot \mX^{T} \right)  \left( \mX \odot \bar{\mOmega}_k \right) \hat{\mB}^{*} \mtheta_k^{*} \right) \right) \nonumber \\
    & + \left( \frac{1}{N} \hat{\mB}^{T} \left( \bar{\mOmega}_k^{T} \odot \mX^{T} \right) \left( \mX \odot \bar{\mOmega}_k \right) {\hat{\mB}} \right)^{-1}
    \left( \frac{1}{N} \left( \hat{\mB}^{T} \left( \bar{\mOmega}_k^{T} \odot \mX^{T} \right)  \left( \mZ \odot \mOmega_k \right) \right) \right) \, , \\
    & = \hat{\mB}^{T} \hat{\mB}^{*} \mtheta_k^{*} 
    +  \left( \frac{1}{N} \hat{\mB}^{T} \left( \bar{\mOmega}_k^{T} \odot \mX^{T} \right) \left( \mX \odot \bar{\mOmega}_k \right) {\hat{\mB}} \right)^{-1}
     \nonumber \\
     & \left( \frac{1}{N}  \left( \hat{\mB}^{T} \left( \bar{\mOmega}_k^{T} \odot \mX^{T} \right)  \left( \mX \odot \bar{\mOmega}_k \right) \hat{\mB}^{*} \mtheta_k^{*} \right) \right) \nonumber \\
     & -  \left( \frac{1}{N} \hat{\mB}^{T} \left( \bar{\mOmega}_k^{T} \odot \mX^{T} \right) \left( \mX \odot \bar{\mOmega}_k \right) {\hat{\mB}} \right)^{-1} 
      \nonumber \\
      & \left( \frac{1}{N}  \left( \hat{\mB}^{T} \left( \bar{\mOmega}_k^{T} \odot \mX^{T} \right)  \left( \mX \odot \bar{\mOmega}_k \right)   \hat{\mB} \hat{\mB}^{T} \hat{\mB}^{*} \mtheta_k^{*} \right) \right)
      \nonumber \\
      & +  \left( \frac{1}{N} \hat{\mB}^{T} \left( \bar{\mOmega}_k^{T} \odot \mX^{T} \right) \left( \mX \odot \bar{\mOmega}_k \right) {\hat{\mB}} \right)^{-1}
    \left( \frac{1}{N} \left( \hat{\mB}^{T} \left( \bar{\mOmega}_k^{T} \odot \mX^{T} \right)  \left( \mZ \odot \mOmega_k \right) \right) \right) \, .
\end{align}
\end{small}
Then define 
\begin{align}
    \mF_k & = \left( \frac{1}{N} \hat{\mB}^{T} \left( \bar{\mOmega}_k^{T} \odot \mX^{T} \right) \left( \mX \odot \bar{\mOmega}_k \right) {\hat{\mB}} \right)^{-1} 
     \nonumber \\
     & \left( \frac{1}{N}  \left( \hat{\mB}^{T} \left( \bar{\mOmega}_k^{T} \odot \mX^{T} \right)  \left( \mX \odot \bar{\mOmega}_k \right) \hat{\mB}^{*} \mtheta_k^{*} \right) \right) \nonumber \\
     & - \left( \frac{1}{N} \hat{\mB}^{T} \left( \bar{\mOmega}_k^{T} \odot \mX^{T} \right) \left( \mX \odot \bar{\mOmega}_k \right) {\hat{\mB}} \right)^{-1} 
      \nonumber \\
      & \left( \frac{1}{N}  \left( \hat{\mB}^{T} \left( \bar{\mOmega}_k^{T} \odot \mX^{T} \right)  \left( \mX \odot \bar{\mOmega}_k \right)   \hat{\mB} \hat{\mB}^{T} \hat{\mB}^{*} \mtheta_k^{*} \right) \right) \, , \label{equ: Fk}\\
      \mG_k & = \left( \frac{1}{N} \hat{\mB}^{T} \left( \bar{\mOmega}_k^{T} \odot \mX^{T} \right) \left( \mX \odot \bar{\mOmega}_k \right) {\hat{\mB}} \right)^{-1}
    \left( \frac{1}{N} \left( \hat{\mB}^{T} \left( \bar{\mOmega}_k^{T} \odot \mX^{T} \right)  \left( \mZ \odot \mOmega_k \right) \right) \right) \, . \label{equ: Gk}
\end{align}
\end{proof}

\begin{corollary} Define $\mTheta, \mF, \mG$ by the matrices that each row is $\mtheta_k, \mF_k, \mG_k$, respectively, we have
\begin{align}
    \mTheta^{t+1} = \mTheta^{*} [\hat{\mB}^{*}]^{T} \hat{\mB}^{t} + \mF^{t} + \mG^{t} \, .
\end{align}
\label{cor: corollary 1}
\end{corollary}

\begin{proof}
    We can easily obtain this result by Lemma~\ref{lem: lemma 3}.
\end{proof}

\begin{lemma}
    Define
    \begin{align}
        \mH_k = \left( \frac{1}{\sqrt{N}}  \hat{\mB}^{T} \left( \bar{\mOmega}_k^{T} \odot \mX^{T} \right)\right)
        \frac{1}{\sqrt{N}} \left( \mX \odot \bar{\mOmega}_k \right)
        \left( \mI_{d} - \hat{\mB} \hat{\mB}^{T} \right) \hat{\mB}^{*}
    \end{align} \, ,
$\hat{N}_k = \mOmega_k^{T} \mOmega_k$, and $\delta = \cC \frac{c^{3/2} \sqrt{\log(M)}}{\sqrt{\min_k{\hat{N}_k}}}$ for constant $\cC$,  we have
\begin{align}
    \norm{\mH_k}_2 & \le \frac{\delta \hat{N}_k}{\sqrt{c} N} dist \left( \hat{\mB}, \hat{\mB}^{*} \right) \, , \\
    \sum_{k=1}^{K} \norm{\mH_k \mtheta_k^{*}}_2^{2} & \le \left( \frac{\sum_{k=1}^{K} \hat{N}_k^{2}}{K N^{2}} \right) \delta^{2} \norm{\mTheta^{*}}_2^{2} dist^{2} \left( \hat{\mB}, \hat{\mB}^{*} \right) \, ,
\end{align}
with probability at least $1 - \exp(-111 c^{2} \log (M))$.
\label{lem: lemma 4}
\end{lemma}

\begin{proof}
Because we have
\begin{align}
    \mH_{k} 
    & = \sum_{i=1}^{M} \sum_{j=1}^{N_i}
    \left( \frac{ \omega_{i,j;k}}{\sqrt{N}} \hat{\mB}^{T}  \xx_{i,j} \right)
    \left( \frac{ \omega_{i,j;k}}{\sqrt{N}} [\hat{\mB}^{*}]^{T}
    \left(
    \mI_d - \hat{\mB} \hat{\mB}^{T}
    \right) \xx_{i,j}
    \right)^{T} \, , \\
    & = \sum_{i=1}^{M} \sum_{\omega_{i,j;k} = 1}^{N_i} 
     \left( \frac{1}{\sqrt{N}} \hat{\mB}^{T}  \xx_{i,j} \right)
    \left( \frac{1}{\sqrt{N}} [\hat{\mB}^{*}]^{T}
    \left(
    \mI_d - \hat{\mB} \hat{\mB}^{T}
    \right) \xx_{i,j}
    \right)^{T} \, , \\
    & = \left( \frac{1}{\sqrt{N}}  \hat{\mB}^{T}  \hat{\mX}_k^{T} )\right)
        \frac{1}{\sqrt{N}} \hat{\mX}_{k} 
        \left( \mI_{d} - \hat{\mB} \hat{\mB}^{T} \right) \hat{\mB}^{*} \, , \\
    & = \frac{\hat{N}_k}{N} \left( \frac{1}{\sqrt{\hat{N}_k}}  \hat{\mB}^{T}  \hat{\mX}_k^{T} )\right)
        \frac{1}{\sqrt{\hat{N}_k}} \hat{\mX}_{k} 
        \left( \mI_{d} - \hat{\mB} \hat{\mB}^{T} \right) \hat{\mB}^{*}
\end{align}
where $\hat{\mX}_k \in \R^{\hat{N}_k \times d}$ with rows the concatenation of $\xx_{i,j}$ that $\omega_{i,j;k} = 1$. Here we define $\hat{N}_k = \sum_{i=1}^{M} \sum_{j=1}^{N_i} \1_{\omega_{i,j;k} = 1}$. Then directly use Lemma 4 of \citep{collins2021exploiting}, and define $\delta = \cC \frac{c^{3/2} \sqrt{\log(M)}}{\sqrt{\min_k{\hat{N}_k}}}$ for constant $\cC$, we have
\begin{align}
    \norm{\mH_k}_2 & \le \frac{\delta \hat{N}_k}{\sqrt{c} N} dist \left( \hat{\mB}, \hat{\mB}^{*} \right) \, .
\end{align}
with probability at least $1 - \exp(-111 c^{2} \log (M))$. Then we have
\begin{align}
     \sum_{k=1}^{K} \norm{\mH_k \mtheta_k^{*}}_2^{2} 
     & \le \sum_{k=1}^{K} \norm{\mH_k}_{2}^{2} \norm{\mtheta_k^{*}}_{2}^{2} \, , \\
     & \le \frac{c}{K} \norm{\mTheta^{*}}_{2}^{2} \sum_{k=1}^{K} \norm{\mH_k}_{2}^{2} \, , \\
     & \le \left( \frac{\sum_{k=1}^{K} \hat{N}_k^{2}}{K N^{2}} \right) \delta^{2} \norm{\mTheta^{*}}_2^{2} dist^{2} \left( \hat{\mB}, \hat{\mB}^{*} \right) \, .
\end{align}
Then the proof finished.
\end{proof}

\begin{lemma}
    Given $\mF$ defined in Corollary~\ref{cor: corollary 1}, define $\delta = \cC \frac{c^{3/2} \sqrt{\log(M)}}{\sqrt{\min_k{\hat{N}_k}}}$, we have
    \begin{align}
        \norm{\mF_k}_2 & \le \frac{\delta}{(1 - \delta) \sqrt{c}} \norm{\mtheta_k^{*}}_{2} dist \left( \hat{\mB}, \hat{\mB}^{*} \right) \, , \\
        \norm{\mF}_F & \le \frac{\delta}{(1 - \delta) \sqrt{K}} \norm{\mTheta^{*}}_2  dist \left( \hat{\mB}, \hat{\mB}^{*} \right) \, ,
    \end{align}
    with probability at least $1 - \exp(-111 c^{2} \log (M))$.
    \label{lem: lemma 5}
\end{lemma}

\begin{proof}
From Lemma~\ref{lem: lemma 1}, we have
\begin{align}
    \norm{\left( \frac{1}{N} \hat{\mB}^{T} \left( \bar{\mOmega}_k^{T} \odot \mX^{T} \right) \left( \mX \odot \bar{\mOmega}_k \right) {\hat{\mB}} \right)^{-1}}_2^{2} \le \frac{N^2}{\hat{N}_k^{2} (1 - \delta)^{2}} \, .
\end{align}
Then we consider to bound $\mF_k$ first, by Lemma~\ref{lem: lemma 4} we have
\begin{align}
    \norm{\mF_k}_2^{2} 
    & \le \norm{\left( \frac{1}{N} \hat{\mB}^{T} \left( \bar{\mOmega}_k^{T} \odot \mX^{T} \right) \left( \mX \odot \bar{\mOmega}_k \right) {\hat{\mB}} \right)^{-1}}_2^{2} \norm{\mH_k}_2^{2} \norm{\mtheta_k^{*}}_{2}^{2} \, , \\
    & \le \frac{1}{(1 - \delta)^{2}} \frac{\delta^{2}}{c} dist^2 \left( \hat{\mB}, \hat{\mB}^{*} \right) \norm{\mtheta_k^{*}}_{2}^{2} \, ,
\end{align}
for $\delta = \cC \frac{c^{3/2} \sqrt{\log(M)}}{\sqrt{\min_k{\hat{N}_k}}}$  with probability at least $1 - \exp(-111 c^{2} \log (M))$. Then we consider to bound $\mF$, and we have
\begin{align}
    \norm{\mF}_F^{2} & = \sum_{k=1}^{K} \norm{\mF_i}_2^{2} \, , \\
    & \le \sum_{k=1}^{K} \norm{\left( \frac{1}{N} \hat{\mB}^{T} \left( \bar{\mOmega}_k^{T} \odot \mX^{T} \right) \left( \mX \odot \bar{\mOmega}_k \right) {\hat{\mB}} \right)^{-1}}_2^{2} \norm{\mH_k \mtheta_k}_{2}^{2} \, , \\
    & \le \frac{1}{(1 - \delta)^{2}} \sum_{k=1}^{K} \frac{N^2}{\hat{N}_k^{2}} \norm{\mH_k \mtheta_k}_{2}^{2} \, , \\
    & \le \frac{c}{(1 - \delta)^{2} K} \norm{\mTheta^{*}}_2^{2} \sum_{k=1}^{K} \frac{N^2}{\hat{N}_k^{2}} \norm{\mH_k}_2^{2} \, , \\
    & \le \frac{\delta^2}{(1 - \delta)^{2} K} \norm{\mTheta^{*}}_2^{2}  dist^2 \left( \hat{\mB}, \hat{\mB}^{*} \right)
\end{align}
 with probability at least $1 - \exp(-111 c^{2} \log (M))$. The last equation comes from Lemma~\ref{lem: lemma 4}.
\end{proof}

\begin{lemma}
     Given $\mG_k$ defined by 
     \begin{align}
    \mG_k  = \left( \frac{1}{N} \hat{\mB}^{T} \left( \bar{\mOmega}_k^{T} \odot \mX^{T} \right) \left( \mX \odot \bar{\mOmega}_k \right) {\hat{\mB}} \right)^{-1}
    \left( \frac{1}{N} \left( \hat{\mB}^{T} \left( \bar{\mOmega}_k^{T} \odot \mX^{T} \right)  \left( \mZ \odot \mOmega_k \right) \right) \right)
     \end{align}
     and $\mG$ defined in Corollary~\ref{cor: corollary 1}. Define $\delta =  \cC \frac{c^{3/2} \sqrt{\log (M)}}{\sqrt{\min_{k}{\hat{N}_k}}}$, we have
     \begin{align}
        \norm{\mG_k}_2 & \le \frac{\delta}{1 - \delta} \sigma^2 \, , \\
         \norm{\mG}_F & \le \sqrt{K} \frac{\delta}{(1 - \delta)} \sigma^{2} \, ,
     \end{align}
      with probability at least $1 - \exp(-110c^{2} \log(M))$.
      \label{lem: lemma 6}
\end{lemma}

\begin{proof}
    We can rewrite $\mG_k$ by
    \begin{align}
        \mG_k  & = \left( \frac{1}{N} \hat{\mB}^{T} \left( \bar{\mOmega}_k^{T} \odot \mX^{T} \right) \left( \mX \odot \bar{\mOmega}_k \right) {\hat{\mB}} \right)^{-1}
        \frac{1}{N} \sum_{i=1}^{M} \sum_{j=1}^{N_i} \omega_{i,j;k} z_{k} \hat{\mB}^{T} \xx_{i,j} \, , \\
        & = \left( \frac{1}{N} \hat{\mB}^{T} \left( \bar{\mOmega}_k^{T} \odot \mX^{T} \right) \left( \mX \odot \bar{\mOmega}_k \right) {\hat{\mB}} \right)^{-1}
        \frac{1}{N} \sum_{\omega_{i,j;k} = 1}^{N}  z_{k} \hat{\mB}^{T} \xx_{i,j} \, , \\
        & = \left( \frac{1}{N} \hat{\mB}^{T} \left( \bar{\mOmega}_k^{T} \odot \mX^{T} \right) \left( \mX \odot \bar{\mOmega}_k \right) {\hat{\mB}} \right)^{-1}
         \frac{1}{N} \left( \hat{\mB}^{T}  \hat{\mX}_{k}^{T}  \hat{\mZ}_k \right) \, .
    \end{align}
    where $\hat{\mX}_k \in \R^{\hat{N}_k \times d}$ with rows the concatenation of $\xx_{i,j}$ that $\omega_{i,j;k} = 1$. Here we define $\hat{N}_k = \sum_{i=1}^{M} \sum_{j=1}^{N_i} \1_{\omega_{i,j;k} = 1}$. Then directly use Lemma A.7 of \citep{tziotis2022straggler}, and define $\delta = \cC \frac{c^{3/2} \sqrt{\log (M)}}{\sqrt{\min_k \hat{N}_k}}$, we have
    \begin{align}
        \norm{\frac{1}{\hat{N}_k} \hat{\mB}^{T}  \hat{\mX}_{k}^{T}  \hat{\mZ}_k }_2 \le \sigma^{2} \delta \, .
    \end{align}
    with probability at least $1 - \exp(-113c^{2} \log(M))$.
    Then consider $\mG_k$, we have
    \begin{align}
        \norm{\mG_k}_2 & \le \frac{\hat{N}_k}{N} \norm{\left( \frac{1}{N} \hat{\mB}^{T} \left( \bar{\mOmega}_k^{T} \odot \mX^{T} \right) \left( \mX \odot \bar{\mOmega}_k \right) {\hat{\mB}} \right)^{-1}} \norm{\frac{1}{\hat{N}_k} \hat{\mB}^{T}  \hat{\mX}_{k}^{T}  \hat{\mZ}_k }_2 \, , \\
        & \le \frac{\delta}{1 - \delta} \sigma^{2} \, .
    \end{align}
    Then consider $\mG$, we have
    \begin{align}
        \norm{\mG}_F^{2} = \sum_{k=1}^{K} \norm{\mG_k}_2^{2} \le  K \left( \frac{\delta}{(1 - \delta)} \right)^{2} \sigma^{4} \, .
    \end{align}
\end{proof}

\begin{lemma}
Define $\delta =  \cC \frac{c^{3/2} \sqrt{\log (M)}}{\sqrt{\min_{k}{\hat{N}_k}}}$, we have
\begin{align}
    & \norm{\mtheta_k}_2 \le \sqrt{c} + \frac{\delta}{1 - \delta} dist \left( \hat{\mB},  \hat{\mB}^{*} \right) + \frac{\delta}{1 - \delta} \sigma^{2} \, , \\ 
    & \norm{\hat{\mB} \mtheta_k - \hat{\mB}^{*} \mtheta_k^{*}}_2 \le \left( \sqrt{c} + \frac{\delta}{1 - \delta} \right) dist \left( \hat{\mB},  \hat{\mB}^{*} \right)
    + \frac{\delta}{1 - \delta} \sigma^2 \, , \\
     & \norm{\frac{1}{N} \sum_{k=1}^{K} \left( \bar{\mOmega}_k^{T} \odot \mX^{T} \right) \left( \mZ \odot \mOmega_k \right) \mtheta_k^{T} }_2 \nonumber \\
     & \le \cC_1 \sigma^2 \frac{\sqrt{d + c}}{\sqrt{N}} \left( \sqrt{c} + \frac{\delta}{1 - \delta} dist \left( \hat{\mB},  \hat{\mB}^{*} \right) + \frac{\delta}{1 - \delta} \sigma^{2} \right) \, , 
\end{align}
with probability at least $1 - \exp(-105(d+c)) - \exp(-105c^2 \log(M))$ for some constant $\cC_1$.
    \label{lem: lemma 7}
\end{lemma}

\begin{proof}
Define 
\begin{align}
    \qq_k = \left( \bar{\mOmega}_k^{T} \odot \mX^{T} \right) \left( \mX \odot \bar{\mOmega}_k \right) \hat{\mB} \mtheta_k - \left( \bar{\mOmega}_k^{T} \odot \mX^{T} \right) \left(  \mY \odot \mOmega_k \right) \, ,
\end{align}
and $\mQ \in \R^{d \times K}$ with rows the concatenation of $\qq_k$.
    With the fact that $\mY = \sum_{k=1}^{K} \mOmega_k \odot \left( \mX \hat{\mB}^{*} \mtheta_k^{*} + \mZ \right)$, we have
    \begin{align}
        \qq_k & = \left( \bar{\mOmega}_k^{T} \odot \mX^{T} \right) \left( \mX \odot \bar{\mOmega}_k \right) \hat{\mB} \mtheta_k - \left( \bar{\mOmega}_k^{T} \odot \mX^{T} \right) \left(  \mY \odot \mOmega_k \right) \, , \\
        & = \left( \bar{\mOmega}_k^{T} \odot \mX^{T} \right) \left( \mX \odot \bar{\mOmega}_k \right) \hat{\mB} \mtheta_k
        -  \left( \bar{\mOmega}_k^{T} \odot \mX^{T} \right) \left( \mOmega_k \odot \sum_{n=1}^{K} \mOmega_n \odot \left( \mX \hat{\mB}^{*} \mtheta_n^{*} + \mZ \right) \right) \, , \\
        & = \left( \bar{\mOmega}_k^{T} \odot \mX^{T} \right) \left( \mX \odot \bar{\mOmega}_k \right) \hat{\mB} \mtheta_k
        -  \left( \bar{\mOmega}_k^{T} \odot \mX^{T} \right) \left( \mOmega_k \odot \left( \mX \hat{\mB}^{*} \mtheta_k^{*} + \mZ \right) \right) \, , \\
        & = \left( \bar{\mOmega}_k^{T} \odot \mX^{T} \right) \left( \mX \odot \bar{\mOmega}_k \right) \hat{\mB} \mtheta_k
        - \left( \bar{\mOmega}_k^{T} \odot \mX^{T} \right) \left( \mX \odot \bar{\mOmega}_k \right) \hat{\mB}^{*} \mtheta_k^{*} \nonumber \\
        & - \left( \bar{\mOmega}_k^{T} \odot \mX^{T} \right) \left( \mZ \odot \mOmega_k \right) \, , \\
        & =  \left( \bar{\mOmega}_k^{T} \odot \mX^{T} \right) \left( \mX \odot \bar{\mOmega}_k \right)
        \left( \hat{\mB} \mtheta_k - \hat{\mB}^{*} \mtheta_k^{*} \right) 
        - \left( \bar{\mOmega}_k^{T} \odot \mX^{T} \right) \left( \mZ \odot \mOmega_k \right) \, .
    \end{align}
Then we would like to consider the $ \left( \hat{\mB} \mtheta_k - \hat{\mB}^{*} \mtheta_k^{*} \right)$ first. From Lemma~\ref{lem: lemma 3}, we have
\begin{align}
     \mtheta_{k} = \hat{\mB}^{T} \hat{\mB}^{*} \mtheta_k^{*} + \mF_k + \mG_k \, .
\end{align}
Therefore we have
\begin{align}
    \norm{\hat{\mB} \mtheta_k - \hat{\mB}^{*} \mtheta_k^{*}}_2 
    & = \norm{\hat{\mB} \hat{\mB}^{T} \hat{\mB}^{*} \mtheta_k^{*} + \hat{\mB} \mF_k + \hat{\mB} \mG_k - \hat{\mB}^{*} \mtheta_k^{*}}_2 \, , \\
    & \le \norm{\left(  \hat{\mB} \hat{\mB}^{T} - \mI_d \right) \hat{\mB}^{*} \mtheta_k^{*} }_2
    +  \norm{\hat{\mB} \mF_k}_2
    +  \norm{\hat{\mB} \mG_k}_2 \, , \\
    & \le  dist \left( \hat{\mB},  \hat{\mB}^{*} \right) \norm{\mtheta_k^{*}}_2
    +  \norm{\mF_k}_2 +  \norm{\mG_k}_2 \, , \\
    & \le \sqrt{c}  dist \left( \hat{\mB},  \hat{\mB}^{*} \right) 
    + \frac{\delta}{(1 - \delta) \sqrt{c}} \norm{\mtheta_k^{*}}_2 dist \left( \hat{\mB},  \hat{\mB}^{*} \right)
    + \frac{\delta}{1 - \delta} \sigma^2 \, , \\
    & = \left( \sqrt{c} + \frac{\delta}{1 - \delta} \right) dist \left( \hat{\mB},  \hat{\mB}^{*} \right)
    + \frac{\delta}{1 - \delta} \sigma^2 \, ,
\end{align}
with probability at least $1 - \exp(-110c^{2} \log(M))$, and $\delta =  \cC \frac{c^{3/2} \sqrt{\log (M)}}{\sqrt{\min_{k}{\hat{N}_k}}}$.

Then we consider to bound $\mtheta_k$, and we have
\begin{align}
    \norm{\mtheta_k}_2 
    & = \norm{\hat{\mB}^{T} \hat{\mB}^{*} \mtheta_k^{*} + \mF_k + \mG_k}_2 \, , \\
    & \le \norm{\mtheta_k^{*}}_2 + \norm{F_k}_2 + \norm{\mG_k}_2 \, , \\
    & \le \sqrt{c} + \frac{\delta}{1 - \delta} dist \left( \hat{\mB},  \hat{\mB}^{*} \right) + \frac{\delta}{1 - \delta} \sigma^{2} \, ,
\end{align}
with probability at least $1 - \exp(-110c^{2} \log(M))$.

Then we consider to bound $\left( \bar{\mOmega}_k^{T} \odot \mX^{T} \right) \left( \mZ \odot \mOmega_k \right)$, and we have
\begin{align}
    \left( \bar{\mOmega}_k^{T} \odot \mX^{T} \right) \left( \mZ \odot \mOmega_k \right) 
    & = \sum_{\omega_{i,j;k}=1}^{N}  z_{k} \xx_{i,j} \, , \\
    & = \hat{\mX}_k^{T} \hat{\mZ}_k \, ,
\end{align}
where rows of $\hat{\mX}_k \in \R^{\hat{N}_k \times d}$ and $\hat{\mZ}_k \in \R^{\hat{N}_k}$ are $\xx_{i,j}$ and $z_k$ subject to $\omega_{i,j;k} = 1$. Then let $\cS^{d-1}, \cS^{c-1}$ denote the unit spheres in $d$ and $c$ dimensions, and $\cN_d, \cN_k$ denote the $\frac{1}{4}$-nets of cardinality $9^{d}$ and $9^{k}$, respectively. Then by Equation 4.13 of \citep{vershynin2018high}, we have
\begin{align}
    \norm{\frac{1}{N} \sum_{k=1}^{K} \hat{\mX}_k^{T} \hat{\mZ}_k \mtheta_k^{T} }_2 
    & \le 2 \max_{\pp \in \cN_d, \yy \in \cN_k} \pp^{T} \left( \frac{1}{N} \sum_{k=1}^{K} \hat{\mX}_k^{T} \hat{\mZ}_k \mtheta_k^{T} \right) \yy \, , \\
    & = 2 \max_{\pp \in \cN_d, \yy \in \cN_k} \sum_{k=1}^{K} \sum_{i,j}^{\hat{N}_k}  \left( \frac{z_k}{N} \langle \xx_{i,j}, \pp \rangle \langle \mtheta_k, \yy \rangle  \right) \, .
\end{align}
Notice that for any fixed $\pp, \yy$, the random variables $\frac{z_k}{N} \langle \xx_{i,j}, \pp \rangle \langle \mtheta_k, \yy \rangle$ are i.i.d. zero-mean sub-exponentials with the norm at most $\cC_1 \frac{\sigma^2 \norm{\mtheta_k}}{N}$ for some constant $\cC$. Then consider the event 
\begin{align}
    \cE = \bigcap_{k=1}^{K} \left\{ \norm{\mtheta_k}_2 \le \sqrt{c} + \frac{\delta}{1 - \delta} dist \left( \hat{\mB},  \hat{\mB}^{*} \right) + \frac{\delta}{1 - \delta} \sigma^{2} \right\} \, ,
\end{align}
which holds with probability at least $1 - \exp(-105 c^2 \log(M))$. Then use the Bernstein's inequality we have
\begin{small}
\begin{align}
\textstyle
    & \Pr \left( \sum_{k=1}^{K} \sum_{i,j}^{\hat{N}_k}  \frac{z_k}{N} \langle \xx_{i,j}, \pp \rangle \langle \mtheta_k, \yy \rangle \ge s  \;\mid\; \cE \right) \nonumber \\
    & \le \exp ( -\cC_1^{'} N \min \{ \frac{s^2}{\sigma^4 \left( \sqrt{c} + \frac{\delta}{1 - \delta} dist \left( \hat{\mB},  \hat{\mB}^{*} \right) + \frac{\delta}{1 - \delta} \sigma^{2} \right)^2}, 
    \nonumber \\
    & \frac{s}{\sigma^2 \left( \sqrt{c} + \frac{\delta}{1 - \delta} dist \left( \hat{\mB},  \hat{\mB}^{*} \right) + \frac{\delta}{1 - \delta} \sigma^{2} \right)} \} ) \, .
\end{align}
\end{small}
Setting
\begin{align}
    s = \cC_2 \frac{\sigma^2 \sqrt{d + c} \left( \sqrt{c} + \frac{\delta}{1 - \delta} dist \left( \hat{\mB},  \hat{\mB}^{*} \right) + \frac{\delta}{1 - \delta} \sigma^{2} \right)}{  \sqrt{N}} \, ,
\end{align}
we have
\begin{align}
    & \Pr ( \sum_{k=1}^{K} \sum_{i,j}^{\hat{N}_k}  \frac{z_k}{N} \langle \xx_{i,j}, \pp \rangle \langle \mtheta_k, \yy \rangle \ge 
    \nonumber \\
    & \cC_2 \sigma^2 \frac{\sqrt{d + c}}{\sqrt{N}} \left( \sqrt{c} + \frac{\delta}{1 - \delta} dist \left( \hat{\mB},  \hat{\mB}^{*} \right) + \frac{\delta}{1 - \delta} \sigma^{2} \right)  \;\mid\; \cE ) \nonumber \\
    & \le \exp(- \cC_1^{'} \cC_2 d) \le exp(-110(d+c)) \, ,
\end{align}
for $\cC_2$ large enough. Taking the union bound over all points $\pp, \yy$ on the $\cN_d, \cN_k$, we have
\begin{align}
    & \Pr ( \norm{\frac{1}{N} \sum_{k=1}^{K} \hat{\mX}_k^{T} \hat{\mZ}_k \mtheta_k^{T} }_2  \ge 
    \nonumber \\
    & 2 \cC_2 \sigma^2 \frac{\sqrt{d + c}}{\sqrt{N}} \left( \sqrt{c} + \frac{\delta}{1 - \delta} dist \left( \hat{\mB},  \hat{\mB}^{*} \right) + \frac{\delta}{1 - \delta} \sigma^{2} \right)  \;\mid\; \cE ) \nonumber \\
    & \le 9^{d+c} \exp(-110(d+c)) \le \exp(-105(d+c)) \, .
\end{align}
Removing the conditional on $\cE$, we have
\begin{align}
    & \Pr \left( \norm{\frac{1}{N} \sum_{k=1}^{K} \hat{\mX}_k^{T} \hat{\mZ}_k \mtheta_k^{T} }_2  \ge 2 \cC_2 \sigma^2 \frac{\sqrt{d + c}}{\sqrt{N}} \left( \sqrt{c} + \frac{\delta}{1 - \delta} dist \left( \hat{\mB},  \hat{\mB}^{*} \right) + \frac{\delta}{1 - \delta} \sigma^{2} \right)  \right) \nonumber \\
    & \le \exp(-105(d+c)) + \Pr (\cE^{C}) \le \exp(-105(d+c)) + \exp(-105c^2 \log(M)).
\end{align}
\end{proof}

\begin{lemma}
    Define $\delta =  \cC \frac{c^{3/2} \sqrt{\log (M)}}{\sqrt{\min_{k}{\hat{N}_k}}}$, we have
    \begin{align}
        & \norm{\sum_{k=1}^{K}  \frac{1}{N} \left( \left( \bar{\mOmega}_k^{T} \odot \mX^{T} \right) \left( \mX \odot \bar{\mOmega}_k \right)
        \left( \hat{\mB} \mtheta_k - \hat{\mB}^{*} \mtheta_k^{*} \right) 
        - \hat{N}_k \left( \hat{\mB} \mtheta_k - \hat{\mB}^{*} \mtheta_k^{*} \right) \right) \mtheta_k^{T}}_2 \nonumber \\
        & \le \cC \frac{\sqrt{d+c} }{\sqrt{N}} \epsilon \, ,
    \end{align}
    where
    \begin{align}
          \epsilon & =   \left(  \sqrt{c} \frac{\delta}{1 - \delta} + \frac{\delta^{2}}{(1 - \delta)^{2}} \right) dist^2 \left( \hat{\mB},  \hat{\mB}^{*} \right) \nonumber \\
        & +   \left( c + (\sigma^2 + 1) \sqrt{c} \frac{\delta}{1 - \delta} + \frac{2 \delta^2}{(1 - \delta)^2} \sigma^2 \right) dist \left( \hat{\mB},  \hat{\mB}^{*} \right)
        + \sqrt{c} \frac{\delta}{1 - \delta} \sigma^2 + \frac{\delta^2}{(1 - \delta)^2} \sigma^4  \, ,
    \end{align}
    with probability at least $1 - \exp(-100(d+c)) - \exp(-105c^2 \log(M))$.
    \label{lem: lemma 8}
\end{lemma}

\begin{proof}
    Define 
    \begin{align}
        \qq_k & = \frac{1}{N} \left( \bar{\mOmega}_k^{T} \odot \mX^{T} \right) \left( \mX \odot \bar{\mOmega}_k \right)
        \left( \hat{\mB} \mtheta_k - \hat{\mB}^{*} \mtheta_k^{*} \right) 
        - \frac{\hat{N}_k}{N} \left( \hat{\mB} \mtheta_k - \hat{\mB}^{*} \mtheta_k^{*} \right) \, , \\
        & = \frac{1}{N} \hat{X}_k^{T} \hat{X}_k \left( \hat{\mB} \mtheta_k - \hat{\mB}^{*} \mtheta_k^{*} \right) 
        - \frac{\hat{N}_k}{N} \left( \hat{\mB} \mtheta_k - \hat{\mB}^{*} \mtheta_k^{*} \right) \, ,
    \end{align}
    where rows of $\hat{\mX}_k \in \R^{\hat{N}_k \times d}$ are $\xx_{i,j}$ subject to $\omega_{i,j;k} = 1$. Then we would like to define the event
    \begin{align}
    \cE & = \bigcap_{k=1}^{K} \left\{ \cA_k \bigcap \cB_k \right\} \, , \\
    \cA_k & = \left\{ \norm{\mtheta_k}_2 \le \sqrt{c} + \frac{\delta}{1 - \delta} dist \left( \hat{\mB},  \hat{\mB}^{*} \right) + \frac{\delta}{1 - \delta} \sigma^{2}   \right\} \, , \\
    \cB_k & = \left\{ \norm{\hat{\mB} \mtheta_k - \hat{\mB}^{*} \mtheta_k^{*}}_2 \le \left( \sqrt{c} + \frac{\delta}{1 - \delta} \right) dist \left( \hat{\mB},  \hat{\mB}^{*} \right) + \frac{\delta}{1 - \delta} \sigma^2 \right\} \, ,
    \end{align}
    happens with the probability at least $1 - \exp(-105(d+c)) - \exp(-105c^{2}\log(M))$ by Lemma~\ref{lem: lemma 7}. Define $\gg_k = \hat{\mB} \mtheta_k - \hat{\mB}^{*} \mtheta_k^{*}$, we have
    \begin{align}
        \sum_{k=1}^{K} \qq_k \mtheta_k^{T} = \frac{1}{N} \left( \sum_{k=1}^{K} \sum_{i,j}^{\hat{N}_k} \left( \langle \xx_{i,j}, \gg_{k} \rangle \xx_{i,j} \mtheta_k^{T}
        -  \gg_k \mtheta_k^{T} \right) \right) \, .
    \end{align}
    Let $\cS^{d-1}, \cS^{c-1}$ denote the unit spheres in $d$ and $c$ dimensions and $\cN_d, \cN_k$ the $\frac{1}{4}$-nets of cardinality $9^{d}$ and $9^{k}$, respectively. By Equation 4.13 in \citep{vershynin2018high}, we have
    \begin{align}
        \norm{\sum_{k=1}^{K} \qq_k \mtheta_k^{T} }_2 & \le \frac{2}{N} \max_{\pp \in \cN_d, \yy \in \cN_k} \pp^{T} \left( \sum_{k=1}^{K} \sum_{i,j}^{\hat{N}_k} \langle \xx_{i,j}, \gg_{k} \rangle \xx_{i,j} \mtheta_k^{T}
        - \sum_{k=1}^{K} \gg_k \mtheta_k^{T} \right) \yy \, , \\
        & = \frac{2}{N} \max_{\pp \in \cN_d, \yy \in \cN_k} \sum_{k=1}^{K} \sum_{i,j}^{\hat{N}_k} \left( \langle \xx_{i,j}, \gg_k \rangle \langle \pp, \xx_{i,j} \rangle \langle \mtheta_k, \yy \rangle - \langle \pp, \gg_k \rangle \langle \mtheta_k, \yy \rangle \right) \, .
    \end{align}
    Then for any $\pp, \yy$, the inner products $\langle \xx_{i,j}, \gg_k \rangle, \langle \pp, \xx_{i,j} \rangle$ are sub-gaussians with norm at most $\cC_1 \norm{\gg_i}_2$ and $\cC_2 \norm{\pp}_2 = \cC_2$, respectively for some constants $\cC_1, \cC_2$. Then under the condition that $\cE$ holds we have $\frac{1}{N} \langle \xx_{i,j}, \gg_k \rangle \langle \pp, \xx_{i,j} \rangle \langle \mtheta_k, \yy \rangle$ is sub-exponential with norm at most 
    \begin{align}
        \frac{\cC_3}{N} \epsilon & = \frac{\cC_3}{N}  \left(  \sqrt{c} \frac{\delta}{1 - \delta} + \frac{\delta^{2}}{(1 - \delta)^{2}} \right) dist^2 \left( \hat{\mB},  \hat{\mB}^{*} \right) \nonumber \\
        & + \frac{\cC_3}{N} ( \left( c + (\sigma^2 + 1) \sqrt{c} \frac{\delta}{1 - \delta} + \frac{2 \delta^2}{(1 - \delta)^2} \sigma^2 \right) dist \left( \hat{\mB},  \hat{\mB}^{*} \right) \nonumber \\
        & + \sqrt{c} \frac{\delta}{1 - \delta} \sigma^2 + \frac{\delta^2}{(1 - \delta)^2} \sigma^4 ) \, .
    \end{align}
    The same thing can be observed for $\langle \pp, \gg_k \rangle \langle \mtheta_k, \yy \rangle$. Besides, we can observe that
    \begin{align}
        \Eb{\langle \xx_{i,j}, \gg_k \rangle \langle \pp, \xx_{i,j} \rangle \langle \mtheta_k, \yy \rangle - \langle \pp, \gg_k \rangle \langle \mtheta_k, \yy \rangle} = 0 \, .
    \end{align}
    Then we are dealing with $N$ zero-mean, sub-exponential random variables. Using Bernstein’s inequality we have
    \begin{align}
        & \Pr \left( \frac{1}{N} \sum_{k=1}^{K} \sum_{i,j}^{\hat{N}_k} \langle \xx_{i,j}, \gg_k \rangle \langle \pp, \xx_{i,j} \rangle \langle \mtheta_k, \yy \rangle - \langle \pp, \gg_k \rangle \langle \mtheta_k, \yy \rangle \ge s \mid \cE \right) \le \nonumber \\
        & \exp \left( - \cC_4 N \min \left\{ \frac{s^2}{\epsilon^2}, \frac{s}{\epsilon} \right\} \right) \, .
    \end{align}
    Setting $s = \frac{\sqrt{\cC_5 (d+c)} \epsilon }{\sqrt{N}}$, for constant $\cC_5$ that satisfy $\cC_5 \le \frac{N}{d+c}$, and taking union bound over all $\pp, \yy$, we have
    \begin{align}
        & \Pr \left( \norm{\sum_{k=1}^{K} \qq_k \mtheta_k^{T} }_2 \ge \frac{2 \sqrt{\cC_5 (d+c)} \epsilon }{\sqrt{N}} \mid \cE \right) \nonumber \\
        & \le 9^{d+c} \exp \left( - \cC_4 \cC_5 (d + c) \right) \le \exp(-105(d+c)) \, .
    \end{align}
    Then by removing the conditional on $\cE$, we have
    \begin{align}
        \norm{\sum_{k=1}^{K} \qq_k \mtheta_k^{T} }_2 \le \cC \frac{\sqrt{d+c} \epsilon}{\sqrt{N}} \, ,
    \end{align}
    with probability at least $1 - \exp(-100(d+c)) - \exp(-105c^{2}\log(M))$.
\end{proof}

\subsection{Main Results}

\begin{theorem}
\label{the-convergence-appendix}
    Under Assumption~\ref{ass: Sub-Gaussian design}-~\ref{ass: Binary weights}, when we have $N \ge \frac{K^2}{d+c}$, and $\min_{k} \hat{N}_k \ge \cC \frac{c^3(1 + \sigma^2)^4 \log(M)}{E_0^2} \min \left \{ \frac{1}{\kappa^2}, \bar{\sigma}_{\min}^2 \right \}$ for some constant $\cC$, we have
    \begin{align}
         dist(\hat{\mB}^{t+1}, \hat{\mB}^{*}) 
         & \le dist(\hat{\mB}^{t}, \hat{\mB}^{*}) \left(1 - c_{min} + \frac{57}{200} c_{max} \right) \left( 1 - \frac{1}{2} c_{max} \right)^{-1/2} \nonumber \\
         & + \left( \frac{7}{100} c_{max} \right) \left( 1 - \frac{1}{2} c_{max} \right)^{-1/2} \, ,
    \end{align}
    with the  probability at least $1 - \exp(-90(d+c)) - \exp(-90c^2\log(M))$.
    Here $\hat{N_k} = \sum_{i=1}^{M} \sum_{j=1}^{N_i} \omega_{i,j;k}$, $E_0 = 1 - dist^2(\hat{\mB}^{0}, \hat{\mB}^{*})$, $c_{min} = \eta K \frac{\min_k \hat{N}_k}{N} \bar{\sigma}_{\min, *}^2 E_0$, and $c_{max} = \eta K \frac{\max_k \hat{N}_k}{N} \bar{\sigma}_{\min, *}^2 E_0$.
\end{theorem}

\begin{proof}
    From the optimization steps we have
    \begin{align}
        \mB^{t+1} &= \hat{\mB}^{t} - \sum_{k=1}^{K} \frac{\eta}{N} \left( \left( \bar{\mOmega}_k^{T} \odot \mX^{T} \right) \left( \mX \odot \bar{\mOmega}_k \right) \hat{\mB}^{t} \mtheta_k^{t} - \left( \bar{\mOmega}_k^{T} \odot \mX^{T} \right) \left(  \mY \odot \mOmega_k \right) \right) (\mtheta_k^{t})^{T} \, , \\
        & = \hat{\mB}^{t}
        - \sum_{k=1}^{K} \frac{\eta}{N} ( \left( \bar{\mOmega}_k^{T} \odot \mX^{T} \right) \left( \mX \odot \bar{\mOmega}_k \right)
        \left( \hat{\mB}^{t} \mtheta_k^{t} - \hat{\mB}^{*} \mtheta_k^{*} \right) \nonumber \\
        & - \left( \bar{\mOmega}_k^{T} \odot \mX^{T} \right) \left( \mZ \odot \mOmega_k \right) ) (\mtheta_k^{t})^{T} \, , \\
        & =  \hat{\mB}^{t}
        - \eta ( \sum_{k=1}^{K} ( \frac{1}{N} \left( \bar{\mOmega}_k^{T} \odot \mX^{T} \right) \left( \mX \odot \bar{\mOmega}_k \right)
        \left( \hat{\mB}^{t} \mtheta_k^{t} - \hat{\mB}^{*} \mtheta_k^{*} \right) \nonumber \\
        & - \frac{\hat{N}_k}{N} \left( \hat{\mB}^{t} \mtheta_k^{t} - \hat{\mB}^{*} \mtheta_k^{*} \right) ) (\mtheta_k^{t})^{T} ) \nonumber \\
        & - \frac{\eta}{N} \sum_{k=1}^{K} \hat{N}_k \left( \hat{\mB}^{t} \mtheta_k^{t} - \hat{\mB}^{*} \mtheta_k^{*} \right) (\mtheta_k^{t})^{T} + \frac{\eta}{N} \sum_{k=1}^{K}  \left( \bar{\mOmega}_k^{T} \odot \mX^{T} \right) \left( \mZ \odot \mOmega_k \right) (\mtheta_k^{t})^{T} \, .
    \end{align}
    Multiplying both sides by $(\hat{\mB}_{\perp}^{*})^{T}$, we have
    \begin{small}
    \begin{align}
    \textstyle
        & (\hat{\mB}_{\perp}^{*})^{T} \mB^{t+1}
        = (\hat{\mB}_{\perp}^{*})^{T} \hat{\mB}^{t} \nonumber \\
        & - \eta (\hat{\mB}_{\perp}^{*})^{T} ( \sum_{k=1}^{K} ( \frac{1}{N} \left( \bar{\mOmega}_k^{T} \odot \mX^{T} \right) \left( \mX \odot \bar{\mOmega}_k \right)
        \left( \hat{\mB}^{t} \mtheta_k^{t} - \hat{\mB}^{*} \mtheta_k^{*} \right) \nonumber \\
            & - \frac{\hat{N}_k}{N} \left( \hat{\mB}^{t} \mtheta_k^{t} - \hat{\mB}^{*} \mtheta_k^{*} \right) ) (\mtheta_k^{t})^{T} ) \nonumber \\
        & - \frac{\eta}{N} \sum_{k=1}^{K} \hat{N}_k \left( (\hat{\mB}_{\perp}^{*})^{T} \hat{\mB}^{t} \mtheta_k^{t} - (\hat{\mB}_{\perp}^{*})^{T} \hat{\mB}^{*} \mtheta_k^{*} \right) (\mtheta_k^{t})^{T} \nonumber \\
        & + \frac{\eta}{N} \sum_{k=1}^{K} (\hat{\mB}_{\perp}^{*})^{T} \left( \bar{\mOmega}_k^{T} \odot \mX^{T} \right) \left( \mZ \odot \mOmega_k \right) (\mtheta_k^{t})^{T} \, , \\
        & =  (\hat{\mB}_{\perp}^{*})^{T} \hat{\mB}^{t} \left( \mI_c - \frac{\eta}{N} \sum_{k=1}^{K} \hat{N}_k \mtheta_k^{t} (\mtheta_k^{t})^{T} \right)
        + \frac{\eta}{N} \sum_{k=1}^{K} (\hat{\mB}_{\perp}^{*})^{T} \left( \bar{\mOmega}_k^{T} \odot \mX^{T} \right) \left( \mZ \odot \mOmega_k \right) (\mtheta_k^{t})^{T} \nonumber \\
        & - \frac{\eta}{N} (\hat{\mB}_{\perp}^{*})^{T} \left( \sum_{k=1}^{K} \left(  \left( \bar{\mOmega}_k^{T} \odot \mX^{T} \right) \left( \mX \odot \bar{\mOmega}_k \right)
        \left( \hat{\mB}^{t} \mtheta_k^{t} - \hat{\mB}^{*} \mtheta_k^{*} \right) 
            -  \hat{N}_k \left( \hat{\mB}^{t} \mtheta_k^{t} - \hat{\mB}^{*} \mtheta_k^{*} \right) \right) (\mtheta_k^{t})^{T} \right) \, .
    \end{align}
    \end{small}
    Because we have $\hat{\mB}^{t+1} = \mB^{t+1} \left( \mR^{t+1} \right)^{-1}$, multiplying both sides by $\left( \mR^{t+1} \right)^{-1}$ we have
    \begin{small}
    \begin{align}
    \textstyle
        & dist \left( \hat{\mB}^{t+1}, \hat{\mB}^{*} \right) \nonumber \\
        & \le dist \left( \hat{\mB}^{t}, \hat{\mB}^{*} \right) \norm{\mI_c - \eta \sum_{k=1}^{K} \mtheta_k^{t} (\mtheta_k^{t})^{T}}_2 \norm{(\mR^{t+1})^{-1}}_2 \nonumber \\
        & + \norm{\frac{\eta}{N} \sum_{k=1}^{K} (\hat{\mB}_{\perp}^{*})^{T} \left( \bar{\mOmega}_k^{T} \odot \mX^{T} \right) \left( \mZ \odot \mOmega_k \right) (\mtheta_k^{t})^{T}}_2 \norm{(\mR^{t+1})^{-1}}_2 \nonumber \\
        & + \norm{\frac{\eta}{N} (\hat{\mB}_{\perp}^{*})^{T} \left( \sum_{k=1}^{K} \left(  \left( \bar{\mOmega}_k^{T} \odot \mX^{T} \right) \left( \mX \odot \bar{\mOmega}_k \right)
        \left( \hat{\mB}^{t} \mtheta_k^{t} - \hat{\mB}^{*} \mtheta_k^{*} \right) 
            -  \hat{N}_k \left( \hat{\mB}^{t} \mtheta_k^{t} - \hat{\mB}^{*} \mtheta_k^{*} \right) \right) (\mtheta_k^{t})^{T} \right)}_2 \nonumber \\
            & \norm{(\mR^{t+1})^{-1}}_2 \, .
    \end{align}
    \end{small}
    Then we can define
    \begin{small}
        \begin{align}
            \textstyle
        A_1 & = dist \left( \hat{\mB}^{t}, \hat{\mB}^{*} \right) \norm{\mI_c - \frac{\eta}{N} \sum_{k=1}^{K} \hat{N}_k \mtheta_k^{t} (\mtheta_k^{t})^{T}}_2 \, , \\
        A_2 & = \norm{\frac{\eta}{N} \sum_{k=1}^{K} (\hat{\mB}_{\perp}^{*})^{T} \left( \bar{\mOmega}_k^{T} \odot \mX^{T} \right) \left( \mZ \odot \mOmega_k \right) (\mtheta_k^{t})^{T}}_2 \, , \\
        A_3 & = \| \frac{\eta}{N} (\hat{\mB}_{\perp}^{*})^{T} ( \sum_{k=1}^{K} (  \left( \bar{\mOmega}_k^{T} \odot \mX^{T} \right) \left( \mX \odot \bar{\mOmega}_k \right)
        \left( \hat{\mB}^{t} \mtheta_k^{t} - \hat{\mB}^{*} \mtheta_k^{*} \right) \nonumber \\
            & -  \hat{N}_k \left( \hat{\mB}^{t} \mtheta_k^{t} - \hat{\mB}^{*} \mtheta_k^{*} \right) ) (\mtheta_k^{t})^{T} ) \|_{2} \, .
    \end{align}
    \end{small}
    Then the inequality become
    \begin{align}
        dist \left( \hat{\mB}^{t+1}, \hat{\mB}^{*} \right) \le \left( A_1 + A_2 + A_3 \right) \norm{(\mR^{t+1})^{-1}}_2 \, .
        \label{equ: final first}
    \end{align}
    For the following parts of the proof, we consider the following events hold
    \begin{small}
        \begin{align}
         \textstyle
        \cE_1 & = \bigcap_{k=1}^{K} \left\{ \cA_k \bigcap \cB_k \right\} \, , \\
        \cE_2 & = \left \{ \norm{\mF^{t}}_F \le \frac{\delta}{(1 - \delta) \sqrt{K}} \norm{\mTheta^{*}}_2 dist \left( \hat{\mB}^{t}, \hat{\mB}^{*} \right) 
        \bigcap
        \norm{\mG^{t}}_F \le \sqrt{K} \frac{\delta}{(1 - \delta)} \sigma^{2} \right\} \, , \\
        \cE_3 & = \{ \norm{\frac{1}{N} \sum_{k=1}^{K} \left( \bar{\mOmega}_k^{T} \odot \mX^{T} \right) \left( \mZ \odot \mOmega_k \right) (\mtheta_k^{t})^{T} }_2  
        \nonumber \\
        & \le \cC_1 \sigma^2 \frac{\sqrt{d + c}}{\sqrt{N}} \left( \sqrt{c} + \frac{\delta}{1 - \delta} dist \left( \hat{\mB},  \hat{\mB}^{*} \right) + \frac{\delta}{1 - \delta} \sigma^{2} \right) \} \, , \\
        \cE_4 & = \{ \norm{\sum_{k=1}^{K}  \frac{1}{N} \left( \left( \bar{\mOmega}_k^{T} \odot \mX^{T} \right) \left( \mX \odot \bar{\mOmega}_k \right)
        \left( \hat{\mB} \mtheta_k^{t} - \hat{\mB}^{*} \mtheta_k^{*} \right) 
        - \hat{N}_k \left( \hat{\mB} \mtheta_k^{t} - \hat{\mB}^{*} \mtheta_k^{*} \right) \right) (\mtheta_k^{t})^{T}}_2 \nonumber \\
        & \le \cC_2 \frac{\sqrt{d+c} }{\sqrt{N}} \epsilon \} \, , \\
    \end{align}
    \end{small}
    where
    \begin{align}
        \cA_k & = \left\{ \norm{\mtheta_k^{t}}_2 \le \sqrt{c} + \frac{\delta}{1 - \delta} dist \left( \hat{\mB},  \hat{\mB}^{*} \right) + \frac{\delta}{1 - \delta} \sigma^{2}   \right\} \, , \\
    \cB_k & = \left\{ \norm{\hat{\mB} \mtheta_k^{t} - \hat{\mB}^{*} \mtheta_k^{*}}_2 \le \left( \sqrt{c} + \frac{\delta}{1 - \delta} \right) dist \left( \hat{\mB}^{t},  \hat{\mB}^{*} \right) + \frac{\delta}{1 - \delta} \sigma^2 \right\} \, , \\
    \epsilon & =   \left(  \sqrt{c} \frac{\delta}{1 - \delta} + \frac{\delta^2}{(1 - \delta)^{2}} \right) dist^2 \left( \hat{\mB},  \hat{\mB}^{*} \right) \nonumber \\
        & +   \left( c + (\sigma^2 + 1) \sqrt{c} \frac{\delta}{1 - \delta} + \frac{2 \delta^2}{(1 - \delta)^2} \sigma^2 \right) dist \left( \hat{\mB},  \hat{\mB}^{*} \right)
        + \sqrt{c} \frac{\delta}{1 - \delta} \sigma^2 + \frac{\delta^2}{(1 - \delta)^2} \sigma^4 \, .
    \end{align}
    which hold with probability at least $1 - \exp(-90(d+c)) - \exp(-90 c^2 \log(M))$ for some constants $\cC_1, \cC_2$ by Lemma~\ref{lem: lemma 5}, ~\ref{lem: lemma 6}, ~\ref{lem: lemma 7}, ~\ref{lem: lemma 8}. Then we consider to bound $A_1, A_2, A_3$, respectively. Then we consider to bound $A_1$ first, and we have
    \begin{align}
        & \lambda_{\max} ((\mTheta^{t})^{T} \mTheta^{t}) = \norm{\mTheta^{t}}_2^{2} = \norm{\mTheta^{*} (\hat{\mB}^{*})^{T} \hat{\mB}^{t} + \mF^{t} + \mG^{t}}_2^{2} \, , \\
        & \le 2 \norm{\mTheta^{*}}_2^{2} + 2 \norm{\mF^{t}}_2^{2} + 2 \norm{\mG^{t}}_2^{2} \, , \\
        & \le 2 \norm{\mTheta^{*}}_2^{2} + \frac{2 \delta^2}{(1 - \delta)^2 K} \norm{\mTheta^{*}}_2^{2} dist^{2} \left( \hat{\mB}^{t},  \hat{\mB}^{*} \right) + 2K \frac{\delta}{(1 - \delta)^2} \sigma^4 \, , \\
        & \le \left( 2 + \frac{2 \delta^2}{(1 - \delta)^2 K} \right) \norm{\mTheta^{*}}_2^{2} + 2K \frac{\delta}{(1 - \delta)^2} \sigma^4 \, , \\
        & \le \left( 2K + \frac{2 \delta^2}{(1 - \delta)^2} \right) \bar{\sigma}_{\max, *}^{2} + 2K \frac{\delta}{(1 - \delta)^2} \sigma^4 \, .
    \end{align}
    Then when $\eta$ is small enough, it's simple to promise that $\mI_c - \frac{\eta}{N} \sum_{k=1}^{K} \mtheta_k \mtheta_k^{T}$ is positive definite. Define diagonal matrix $\mW \in \R^{K \times K}$, and $\mW_{k,k} = \frac{\hat{N}_k}{N}$. Then we have
    \begin{small}
    \begin{align}
        & \norm{\mI_c - \frac{\eta}{N} \sum_{k=1}^{K} \hat{N}_k \mtheta_k^{t} (\mtheta_k^{t})^{T}}_2
        = \norm{\mI_c - \eta (\mW \mTheta^{t})^{T}  \mTheta^{t}}_2 \\
        & \le 1 - \eta \lambda_{\min} \left( (\mW \mTheta^{t})^{T}  \mTheta^{t} \right) \, , \\
        & \le 1 - \eta \left( \sigma_{\min} \left( \mW \right) \sigma_{\min}^{2} \left( \mTheta^{*} (\hat{\mB}^{*})^{T} \hat{\mB}^{t} \right) 
        - \sigma_{\min} ((\mW \mF^{t})^{T} \mF^{t}) 
        -  \sigma_{\min} ((\mW \mG^{t})^{T} \mG^{t}) \right) \nonumber \\
        & + 2 \eta \left( \sigma_{\max} \left( (\mW \mF^{t})^{T}  \mTheta^{*} (\hat{\mB}^{*})^{T} \hat{\mB}^{t} \right) + \sigma_{\max} \left( (\mW \mF^{t})^{T}  \mG \right) + \sigma_{\max} \left( (\mW \mG^{t})^{T}  \mTheta^{*} (\hat{\mB}^{*})^{T} \hat{\mB}^{t} \right) \right) \, , \\
        & \le 1 
        - \eta \left(\frac{\min_k \hat{N}_k}{N} \right) \sigma_{\min}^{2} (\mTheta^{*}) \sigma_{\min}^{2} \left( (\hat{\mB}^{*})^{T} \hat{\mB}^{t} \right)
        + \eta \left(\frac{\max_k \hat{N}_k}{N} \right) \sigma_{\min}^{2} (\mF^{t}) \nonumber \\
        & + \eta \left(\frac{\max_k \hat{N}_k}{N} \right) \sigma_{\min}^{2} (\mG^{t}) \nonumber \\
        & + \frac{2 \eta}{N}  \left(\frac{\max_k \hat{N}_k}{N} \right) \left( \sigma_{\max}  \left( (\mF^{t})^{T} \mTheta^{*} \right) + \sigma_{\max} \left( (\mG^{t})^{T} \mTheta^{*} \right) \right) \nonumber \\
        & + 2 \eta  \left(\frac{\max_k \hat{N}_k}{N} \right) \sigma_{\max} (\mF^{t}) \sigma_{\max} (\mG^{t}) \, , \\
        & \le 1 - \eta K  \left(\frac{\min_k \hat{N}_k}{N} \right) \bar{\sigma}_{\min, *}^{2} \sigma_{\min}^{2} \left( (\hat{\mB}^{*})^{T} \hat{\mB}^{t} \right) 
        + 2 \eta  \left(\frac{\max_k \hat{N}_k}{N} \right) \left( \norm{\mF^{t}}_2 + \norm{\mG^{t}}_2 \right) \norm{\mTheta^{*}}_2 \nonumber \\
        & + 2 \eta  \left(\frac{\max_k \hat{N}_k}{N} \right) \norm{\mF^{t}}_2 \norm{\mG^{t}}_2
        + \eta  \left(\frac{\max_k \hat{N}_k}{N} \right) \norm{\mF^{t}}_2^{2}
        + \eta  \left(\frac{\max_k \hat{N}_k}{N} \right) \norm{\mG^{t}}_2^{2} \, , \\
        & = 1 - \eta K  \left(\frac{\min_k \hat{N}_k}{N} \right) \bar{\sigma}_{\min, *}^{2} \sigma_{\min}^{2} \left( (\hat{\mB}^{*})^{T} \hat{\mB}^{t} \right) 
        + 2 \eta  \left(\frac{\max_k \hat{N}_k}{N} \right) \left( \norm{\mF^{t}}_2 + \norm{\mG^{t}}_2 \right) \norm{\mTheta^{*}}_2 \nonumber \\
        & + \eta  \left(\frac{\max_k \hat{N}_k}{N} \right) \left( \norm{\mF^{t}}_2 + \norm{\mG^{t}}_2 \right)^{2}
        \, .
    \end{align}
    \end{small}
    Then under condition $\cE_1, \cE_2, \cE_3, \cE_4$, we have
    \begin{align}
         & \norm{\mI_c - \frac{\eta}{N} \sum_{k=1}^{K} \hat{N}_k \mtheta_k^{t} (\mtheta_k^{t})^{T}}_2 \nonumber \\
         & \le 1 - \eta K  \left(\frac{\min_k \hat{N}_k}{N} \right) \bar{\sigma}_{\min, *}^{2} \sigma_{\min}^{2} \left( (\hat{\mB}^{*})^{T} \hat{\mB}^{t} \right) \nonumber \\
         & + 2 \eta  \left(\frac{\max_k \hat{N}_k}{N} \right) \left( \frac{\delta}{(1 - \delta) \sqrt{K}} \norm{\mTheta^{*}}_2 dist \left( \hat{\mB}^{t}, \hat{\mB}^{*} \right) + \sqrt{K} \frac{\delta}{1 - \delta} \sigma^2 \right) \norm{\mTheta^{*}}_2 \nonumber \\
         & + \eta  \left(\frac{\max_k \hat{N}_k}{N} \right) \left( \frac{\delta}{(1 - \delta) \sqrt{K}} \norm{\mTheta^{*}}_2 dist \left( \hat{\mB}^{t}, \hat{\mB}^{*} \right) + \sqrt{K} \frac{\delta}{1 - \delta} \sigma^2 \right)^{2} \, , \\
         & \le  1 - \eta K  \left(\frac{\min_k \hat{N}_k}{N} \right) \bar{\sigma}_{\min, *}^{2} \sigma_{\min}^{2} \left( (\hat{\mB}^{*})^{T} \hat{\mB}^{t} \right) \nonumber \\
         & + 2 \eta  \left(\frac{\max_k \hat{N}_k}{N} \right) \left( \frac{\delta \sqrt{K}}{1 - \delta} \bar{\sigma}_{\max, *}^2 + \frac{\delta K \sigma^2}{1 - \delta} \bar{\sigma}_{\max, *} \right) \nonumber \\
         & + \eta  \left(\frac{\max_k \hat{N}_k}{N} \right) \frac{\delta^2}{(1 - \delta)^{2}} \left( 2 \bar{\sigma}_{\max, *}^{2} + 2 K \sigma^{4} \right) \,,  \\
         & \le 1 - \eta K  \left(\frac{\min_k \hat{N}_k}{N} \right) \bar{\sigma}_{\min, *}^{2} E_0 
         + \eta K  \left(\frac{\max_k \hat{N}_k}{N} \right) \frac{\delta}{1 - \delta} \left( 2 \bar{\sigma}_{\max, *} + \frac{1}{2} \sigma^2  \right)^2
         \, , \label{equ: need large Nk} \\
         & \le 1 - \eta K  \left(\frac{\min_k \hat{N}_k}{N} \right) \bar{\sigma}_{\min, *}^{2} E_0 
         + \frac{\eta K}{5 }  \left(\frac{\max_k \hat{N}_k}{N} \right) \bar{\sigma}_{\min, *}^{2} E_0
         \nonumber \\
         & + \frac{\eta K}{20 }  \left(\frac{\max_k \hat{N}_k}{N} \right) \bar{\sigma}_{\min, *}^{2} E_0 \, .
    \end{align}
    where $E_0 = 1 - dist^{2} \left( \hat{\mB}^{0}, \hat{\mB}^{*} \right) \le \sigma_{\min}^{2} \left( (\hat{\mB}^{*})^{T} \hat{\mB}^{t} \right)$. 
    The Equation~\eqref{equ: need large Nk} holds when $\hat{N}_k$ is large enough that makes the following equation holds
    \begin{align}
        \frac{\delta}{1 - \delta} \le 2 \delta \le \frac{E_0}{20 (1 + \sigma^2)^2} \max \left\{ \frac{1}{\kappa^{2}}, \bar{\sigma}_{\min}^{2} \right\} \, ,
    \end{align}
    where $\kappa = \frac{\bar{\sigma}_{\max, *}}{\bar{\sigma}_{\min, *}}$, and the equation holds when $\hat{N}_k$ satisfy
    \begin{align}
        \min_k \hat{N}_k \ge \cC_0 \frac{c^3 (1 + \sigma^2)^4 \log(M)}{E_0^2} \min \left\{ \kappa^{4}, \frac{1}{\bar{\sigma}_{\min}^{4}} \right\} \, .
    \end{align}
    Then consider $A_1$, we will have
    \begin{align}
        A_1 \le dist \left( \hat{\mB}^{t}, \hat{\mB}^{*} \right) \left(  1 - \left(  \left(\frac{\min_k \hat{N}_k}{N} \right) - \frac{1}{4}  \left(\frac{\max_k \hat{N}_k}{N} \right) \right)\eta K \bar{\sigma}_{\min, *}^{2} E_0 \right) \, .
        \label{equ: A1 final}
    \end{align}
    The consider $A_2$, because $\norm{\hat{\mB}_{\perp}^{*}}_2 = 1$, and $\cE_3$ holds, we have
    \begin{align}
        A_2 & \le \eta \cC_1 \sigma^2 \frac{\sqrt{d + c}}{\sqrt{N}} \left( \sqrt{c} + \frac{\delta}{1 - \delta} dist \left( \hat{\mB},  \hat{\mB}^{*} \right) + \frac{\delta}{1 - \delta} \sigma^{2} \right) \, , \\
        & \le \eta \cC_1 \sigma^2 \frac{\sqrt{d + c}}{\sqrt{N}} \left(\sqrt{c} + \frac{1}{10} \right) \, .
        \label{equ: A2 final}
    \end{align}
    Similarly, for $A_3$, we have
    \begin{align}
        A_3 & \le \eta \cC_2 \frac{\sqrt{d+c}}{\sqrt{N}} \epsilon \, , \\
        & \le \eta \cC_2 \frac{\sqrt{d+c}}{\sqrt{N}}
        \left(  \left( \sqrt{c} + \frac{1}{\sqrt{10}} \right)^2 dist \left( \hat{\mB}^{t}, \hat{\mB}^{*} \right) 
        + \left( \frac{1}{400} + \frac{ \sqrt{c}}{20} \right) 
        \left( \sigma^2 + 1 \right)  \right) \, .
         \label{equ: A3 final}
    \end{align}
    Combining Equation~\eqref{equ: final first}, ~\eqref{equ: A1 final}, ~\eqref{equ: A2 final}, and ~\eqref{equ: A3 final}, and choose $\cC = \max \{ \cC_1, \cC_2, \cC_1 \cC_2 + \cC_2 \}$, we have
    \begin{small}
    \begin{align}
    \textstyle
        & dist \left( \hat{\mB}^{t+1}, \hat{\mB}^{*} \right) 
        \le dist \left( \hat{\mB}^{t}, \hat{\mB}^{*} \right) \nonumber \\
        & (  1 - \eta K \left(  \left(\frac{\min_k \hat{N}_k}{N} \right)^{2} - \frac{1}{4}  \left(\frac{\max_k \hat{N}_k}{N} \right)^{2} \right) \bar{\sigma}_{\min, *}^{2} E_0 \nonumber \\
        & + \eta \cC \frac{\sqrt{d+c}}{\sqrt{N}} \left( \sqrt{c} + \frac{1}{\sqrt{10}} \right)^2 
        ) \norm{(\mR^{t+1})^{-1}}_2 \nonumber \\
        & + \eta \cC \frac{\sqrt{d+c}}{\sqrt{N}} \left( 
        \sigma^2 + 1 \right) \left( 2 \sqrt{c} + \frac{1}{5} \right) \norm{(\mR^{t+1})^{-1}}_2
        \, .
        \label{equ: one step begin}
    \end{align}
    \end{small}
    Then the remaining thing is to bound $(\mR^{t+1})^{-1}$. Firstly we define
    \begin{align}
        \mS^{t} & =  \sum_{k=1}^{K} \left( \bar{\mOmega}_k^{T} \odot \mX^{T} \right) \left( \mX \odot \bar{\mOmega}_k \right)
        \left( \hat{\mB}^{t} \mtheta_k^{t} - \hat{\mB}^{*} \mtheta_k^{*} \right) (\mtheta_k^{t})^{T} \, , \\
        \mE^{t} & = \sum_{k=1}^{K} \left( \bar{\mOmega}_k^{T} \odot \mX^{T} \right) \left( \mZ \odot \mOmega_k \right) (\mtheta_k^{t})^{T} \, .
    \end{align}
    Then we have
    \begin{align}
        \mB^{t+1} = \hat{\mB}^{t} - \frac{\eta}{N} \mS^{t} + \frac{\eta}{N} \mE^{t} \, .
    \end{align}
    Because $(\mR^{t+1})^{T} \mR^{t+1} = (\mB^{t+1})^{T} \mB^{t+1}$, and we have
    \begin{align}
        (\mB^{t+1})^{T} \mB^{t+1} 
        & = (\hat{\mB}^{t})^{T} \hat{\mB}^{t} 
        - \frac{\eta}{N} \left( (\hat{\mB}^{t})^{T} \mS^{t} + (\mS^{t})^{T} \hat{\mB}^{t} \right) 
        + \frac{\eta}{N} \left( (\hat{\mB}^{t})^{T} \mE^{t} + (\mE^{t})^{T} \hat{\mB}^{t} \right) \nonumber \\
        & + \frac{\eta^2}{N^2} (\mS^{t})^{T} \mS^{t}
        - \frac{\eta^2}{N^2}  \left( (\mE^{t})^{T} \mS^{t} + (\mS^{t})^{T} \mE^{t} \right)
        + \frac{\eta^2}{N^2} (\mE^{t})^{T} \mE^{t} \, , \\
        & = \mI_c
        - \frac{\eta}{N} \left( (\hat{\mB}^{t})^{T} \mS^{t} + (\mS^{t})^{T} \hat{\mB}^{t} \right) 
        + \frac{\eta}{N} \left( (\hat{\mB}^{t})^{T} \mE^{t} + (\mE^{t})^{T} \hat{\mB}^{t} \right) \nonumber \\
        & + \frac{\eta^2}{N^2} (\mS^{t})^{T} \mS^{t}
        - \frac{\eta^2}{N^2}  \left( (\mE^{t})^{T} \mS^{t} + (\mS^{t})^{T} \mE^{t} \right)
        + \frac{\eta^2}{N^2} (\mE^{t})^{T} \mE^{t} \, .
    \end{align}
    By Weyl's inequality, we have
    \begin{align}
        & \sigma_{\min}^{2} \left( \mR^{t+1} \right) \nonumber \\
        & \ge 1 
        - \frac{\eta}{N} \lambda_{\max} \left( (\hat{\mB}^{t})^{T} \mS^{t} + (\mS^{t})^{T} \hat{\mB}^{t} \right) \nonumber \\
        & - \frac{\eta}{N} \lambda_{\max} \left( (\hat{\mB}^{t})^{T} \mE^{t} + (\mE^{t})^{T} \hat{\mB}^{t} \right)
        - \frac{\eta^2}{N^2} \lambda_{\max}  \left( (\mE^{t})^{T} \mS^{t} + (\mS^{t})^{T} \mE^{t} \right) \, .
    \end{align}
    Then we can define
    \begin{align}
        R_1 & = \frac{\eta}{N} \lambda_{\max} \left( (\hat{\mB}^{t})^{T} \mS^{t} + (\mS^{t})^{T} \hat{\mB}^{t} \right) \, , \\
        R_2 & = \frac{\eta}{N} \lambda_{\max} \left( (\hat{\mB}^{t})^{T} \mE^{t} + (\mE^{t})^{T} \hat{\mB}^{t} \right) \, , \\
        R_3 & = \frac{\eta^2}{N^2} \lambda_{\max}  \left( (\mE^{t})^{T} \mS^{t} + (\mS^{t})^{T} \mE^{t} \right) \, .
    \end{align}
    Then we have
    \begin{align}
        \sigma_{\min}^{2} \left( \mR^{t+1} \right) \ge 1 - R_1 - R_2 - R_3 \, .
        \label{equ: R split}
    \end{align}
    Then we consider to bound $R_1, R_2$, and $R_3$, respectively. Consider $R_1$ first, and we have
    \begin{small}
    \begin{align}
    \textstyle
        & R_1 \nonumber \\
        & = \frac{2 \eta}{N} \max_{\norm{\pp}_2 = 1} \pp^{T} (\hat{\mB}^{t})^{T} \mS^{t} \pp \, , \\
        & = \frac{2 \eta}{N} \max_{\norm{\pp}_2 = 1} \pp^{T} (\hat{\mB}^{t})^{T} \left( \sum_{k=1}^{K} \left( \bar{\mOmega}_k^{T} \odot \mX^{T} \right) \left( \mX \odot \bar{\mOmega}_k \right)
        \left( \hat{\mB}^{t} \mtheta_k^{t} - \hat{\mB}^{*} \mtheta_k^{*} \right) (\mtheta_k^{t})^{T} \right) \pp \, , \\
        & \le \frac{2 \eta}{N} \max_{\norm{\pp}_2 = 1} \pp^{T} (\hat{\mB}^{t})^{T} ( \sum_{k=1}^{K} ( \left( \bar{\mOmega}_k^{T} \odot \mX^{T} \right) \left( \mX \odot \bar{\mOmega}_k \right)
        \left( \hat{\mB}^{t} \mtheta_k^{t} - \hat{\mB}^{*} \mtheta_k^{*} \right)  \nonumber \\
        & - \hat{N}_k \left( \hat{\mB}^{t} \mtheta_k^{t} - \hat{\mB}^{*} \mtheta_k^{*} \right) ) (\mtheta_k^{t})^{T}  ) \pp \nonumber \\
        & + \frac{2 \eta}{N} \max_{\norm{\pp}_2 = 1} \pp^{T} (\hat{\mB}^{t})^{T} \left( \sum_{k=1}^{K} \hat{N}_k \left( \hat{\mB}^{t} \mtheta_k^{t} - \hat{\mB}^{*} \mtheta_k^{*} \right) (\mtheta_k^{t})^{T}  \right) \pp \, .
    \end{align}
    \end{small}
    Under the condition $\cE_4$, we have
    \begin{align}
        R_1 & \le 2 \eta \cC_2 \norm{\hat{\mB}^{t}}_2 \frac{\sqrt{d+c}}{\sqrt{N}} \epsilon 
         + \frac{2 \eta}{N} \max_{\norm{\pp}_2 = 1} \pp^{T} (\hat{\mB}^{t})^{T} \left( \sum_{k=1}^{K} \hat{N}_k \left( \hat{\mB}^{t} \mtheta_k^{t} - \hat{\mB}^{*} \mtheta_k^{*} \right) (\mtheta_k^{t})^{T}  \right) \pp \, , \\
        & \le \left( 2 \eta \cC_2 \frac{\sqrt{d+c}}{\sqrt{N}}  \right) \left( \left( \sqrt{c} + \frac{1}{\sqrt{10}} \right)^2 dist \left( \hat{\mB}^{t}, \hat{\mB}^{*} \right) 
        + \left( \frac{1}{400} + \frac{ \sqrt{c}}{20} \right) 
        \left( \sigma^2 + 1 \right)  \right) \nonumber \\
        & + \frac{2 \eta}{N} \max_{\norm{\pp}_2 = 1} \pp^{T} (\hat{\mB}^{t})^{T} \left( \sum_{k=1}^{K} \hat{N}_k \left( \hat{\mB}^{t} \mtheta_k^{t} - \hat{\mB}^{*} \mtheta_k^{*} \right) (\mtheta_k^{t})^{T}  \right) \pp \, , \\
        & \le 4 \eta \left( \cC_2 \frac{\sqrt{d+c}}{\sqrt{N}} \right)
        \left( (\sigma^2 + 1) \left( \sqrt{c} + \frac{1}{\sqrt{10}} \right)^2  \right) \nonumber \\
        & + \frac{2 \eta}{N} \max_{\norm{\pp}_2 = 1} \pp^{T} (\hat{\mB}^{t})^{T} \left( \sum_{k=1}^{K} \hat{N}_k \left( \hat{\mB}^{t} \mtheta_k^{t} - \hat{\mB}^{*} \mtheta_k^{*} \right) (\mtheta_k^{t})^{T}  \right) \pp
        \label{equ: R1 final}
    \end{align}
    The remaining thing is to bound $\frac{2 \eta}{N} \pp^{T} (\hat{\mB}^{t})^{T} \left( \sum_{k=1}^{K} \hat{N}_k \left( \hat{\mB}^{t} \mtheta_k^{t} - \hat{\mB}^{*} \mtheta_k^{*} \right) (\mtheta_k^{t})^{T}  \right) \pp$, and we have
    \begin{align}
        & \frac{2 \eta}{N}  \pp^{T} (\hat{\mB}^{t})^{T} \left( \sum_{k=1}^{K} \hat{N}_k \left( \hat{\mB}^{t} \mtheta_k^{t} - \hat{\mB}^{*} \mtheta_k^{*} \right) (\mtheta_k^{t})^{T}  \right) \pp \nonumber \\
        & = \frac{2 \eta}{N} tr \left[ \sum_{k=1}^{K} \hat{N}_k \left( \hat{\mB}^{t} \mtheta_k^{t} - \hat{\mB}^{*} \mtheta_k^{*} \right) (\mtheta_k^{t})^{T} \pp \pp^{T} (\hat{\mB}^{t})^{T} \right] \, , \\
        & = \frac{2 \eta}{N} tr \left[ \sum_{k=1}^{K} \hat{N}_k \left( \hat{\mB}^{t} \mtheta_k^{t} - \hat{\mB}^{*} \mtheta_k^{*} \right) \left((\hat{\mB}^{t})^{T} \hat{\mB}^{*} \mtheta_k^{*} + \mF_k^{t} + \mG_k^{t} \right)^{T} \pp \pp^{T} (\hat{\mB}^{t})^{T} \right] \, .
    \end{align}
    Then we can define
    \begin{align}
        T_1 & = \frac{2 \eta}{N} tr \left[ \sum_{k=1}^{K}\hat{N}_k \left( \hat{\mB}^{t} \mtheta_k^{t} - \hat{\mB}^{*} \mtheta_k^{*} \right) (\mtheta_k^{*})^{T} (\hat{\mB}^{*})^{T}  \hat{\mB}^{t}  \pp \pp^{T} (\hat{\mB}^{t})^{T} \right] \, , \\
        T_2 & = \frac{2 \eta}{N} tr \left[
        \sum_{k=1}^{K}\hat{N}_k \left( \hat{\mB}^{t} \mtheta_k^{t} - \hat{\mB}^{*} \mtheta_k^{*} \right) (\mF_k^{t})^{T} \pp \pp^{T} (\hat{\mB}^{t})^{T}
        \right] \, , \\
        T_3 & = \frac{2 \eta}{N} tr \left[
         \sum_{k=1}^{K}\hat{N}_k \left( \hat{\mB}^{t} \mtheta_k^{t} - \hat{\mB}^{*} \mtheta_k^{*} \right) (\mG_k^{t})^{T} \pp \pp^{T} (\hat{\mB}^{t})^{T}
        \right] \, .
    \end{align}
    Consider $T_1$ first, we have
    \begin{align}
        T_1 & = \frac{2 \eta}{N} tr [
        \sum_{k=1}^{K}\hat{N}_k \left( 
        \hat{\mB}^{t} (\hat{\mB}^{t})^{T} \hat{\mB}^{*} \mtheta_k^{*} +  \hat{\mB}^{t} \mF_k^{t} + \hat{\mB}^{t} \mG_k^{t} - \hat{\mB}^{*} \mtheta_k^{*}
        \right) \nonumber \\ 
        & (\mtheta_k^{*})^{T} (\hat{\mB}^{*})^{T}  \hat{\mB}^{t}  \pp \pp^{T} (\hat{\mB}^{t})^{T} 
        ] \, , \\
        & = \frac{2 \eta}{N} tr \left[
        \left(  \hat{\mB}^{t} (\hat{\mB}^{t})^{T} - \mI_d \right) \sum_{k=1}^{K} \hat{N}_k \hat{\mB}^{*} \mtheta_k^{*}  (\mtheta_k^{*})^{T} (\hat{\mB}^{*})^{T}  \hat{\mB}^{t}  \pp \pp^{T} (\hat{\mB}^{t})^{T}
        \right] \nonumber \\
        & + \frac{2 \eta}{N} tr \left[ \sum_{k=1}^{K} \hat{N}_k \left( \hat{\mB}^{t} \mF_{k}^{t}  (\mtheta_k^{*})^{T} \right) (\hat{\mB}^{*})^{T}  \hat{\mB}^{t}  \pp \pp^{T} (\hat{\mB}^{t})^{T}
        \right] \nonumber \\
        & + \frac{2 \eta}{N} tr \left[ \sum_{k=1}^{K} \hat{N}_k \left( \hat{\mB}^{t} \mG_{k}^{t}  (\mtheta_k^{*})^{T} \right) (\hat{\mB}^{*})^{T}  \hat{\mB}^{t}  \pp \pp^{T} (\hat{\mB}^{t})^{T}
        \right] \, , \\
        & = 2 \eta tr \left[
        (\hat{\mB}^{t})^{T} \hat{\mB}^{t} \left( (\mF^t)^{T} + (\mG^{t})^{T} \right) \mW \mTheta^{*} (\hat{\mB}^{*})^{T}  \hat{\mB}^{t}  \pp \pp^{T}
        \right] \, , \\
        & \le 2 \eta \frac{ \max_k N_k}{N}
        \left( \norm{\mF^{t}}_F + \norm{\mG^{t}}_F \right) \norm{\mTheta^{*}}_2 \, .
    \end{align}
    Because $\cE_2$ holds, we have
    \begin{align}
        T_1 & \le 2 \eta \frac{ \max_k N_k}{N} 
        \left( \frac{\delta}{(1 - \delta) \sqrt{K}} \norm{\mTheta^{*}}_2^2 
        + \sqrt{K} \frac{\delta}{1 - \delta} \sigma^2 \norm{\mTheta^{*}}_2
        \right) \, , \\
        & \le 2 \eta K \frac{ \max_k N_k}{N} \left(
        \frac{\delta}{1 - \delta} \bar{\sigma}_{\max, *}^{2} + \frac{\delta}{1 - \delta} \sigma^2 \bar{\sigma}_{\max, *}
        \right) \, , \\
        & \le \frac{\eta K}{5} \frac{ \max_k N_k}{N} E_0 \bar{\sigma}_{\min, *}^{2} \, .
        \label{equ: T1 final}
    \end{align}
    Then consider $T_2$, we have
    \begin{align}
        T_2 & = \frac{2 \eta}{N} tr \left[
        \sum_{k=1}^{K}\hat{N}_k \left( \hat{\mB}^{t} \mtheta_k^{t} - \hat{\mB}^{*} \mtheta_k^{*} \right) (\mF_k^{t})^{T} \pp \pp^{T} (\hat{\mB}^{t})^{T}
        \right] \, , \\
        & = \frac{2 \eta}{N} tr \left[
        \sum_{k=1}^{K}\hat{N}_k \left( \hat{\mB}^{t} (\hat{\mB}^{t})^{T} \hat{\mB}^{*} \mtheta_k^{*} 
        +  \hat{\mB}^{t} \mF_k^{t} 
        + \hat{\mB}^{t} \mG^{t}
        - \hat{\mB}^{*} \mtheta_k^{*} \right) (\mF_k^{t})^{T} \pp \pp^{T} (\hat{\mB}^{t})^{T}
        \right] \, , \\
        & = 2 \eta tr \left[
        (\hat{\mB}^{t})^{T} \hat{\mB}^{t} 
        \left( (\mF^{t})^{T} \mF^{t} + (\mG^{t})^{T} \mF^{t} \right) \mW \pp \pp^{T}
        \right] \, , \\
        & \le 2 \eta \frac{ \max_k N_k}{N} \left( \norm{\mF^{t}}_{F}^2 + \norm{\mF^{t}}_F \norm{\mG^{t}}_F \right) \, , \\
        & \le 2 \eta \frac{ \max_k N_k}{N} \left( 
        \frac{\delta^2}{(1 - \delta)^2} \bar{\sigma}_{\max, *}^{2} +  \frac{\delta^2}{(1 - \delta)^2} \sqrt{K} \bar{\sigma}_{\max, *} \sigma^2 \right) \, , \\
        & \le \frac{\eta K}{100} \frac{ \max_k N_k}{N} E_0 \bar{\sigma}_{\min, *}^{2} \, .
        \label{equ: T2 final}
    \end{align}
    Finally, consider $T_3$, we have
    \begin{align}
        T_3 & = \frac{2 \eta}{N} tr \left[
         \sum_{k=1}^{K}\hat{N}_k \left( \hat{\mB}^{t} \mtheta_k^{t} - \hat{\mB}^{*} \mtheta_k^{*} \right) (\mG_k^{t})^{T} \pp \pp^{T} (\hat{\mB}^{t})^{T}
        \right] \, , \\
        & = \frac{2 \eta}{N} tr \left[
        \sum_{k=1}^{K}\hat{N}_k \left( \hat{\mB}^{t} (\hat{\mB}^{t})^{T} \hat{\mB}^{*} \mtheta_k^{*} 
        +  \hat{\mB}^{t} \mF_k^{t} 
        + \hat{\mB}^{t} \mG^{t}
        - \hat{\mB}^{*} \mtheta_k^{*} \right) (\mG_k^{t})^{T} \pp \pp^{T} (\hat{\mB}^{t})^{T}
        \right] \, , \\
        & = 2 \eta tr \left[
        (\hat{\mB}^{t})^{T} \hat{\mB}^{t} 
        \left( (\mG^{t})^{T} \mG^{t} + (\mF^{t})^{T} \mG^{t} \right) \mW \pp \pp^{T}
        \right] \, , \\
        & \le 2 \eta \frac{ \max_k N_k}{N} \left( \norm{\mG^{t}}_{F}^2 + \norm{\mF^{t}}_F \norm{\mG^{t}}_F \right) \, , \\
        & \le 2 \eta \frac{ \max_k N_k}{N} \left( 
        \frac{\delta^2}{(1 - \delta)^2} K \sigma^4 +  \frac{\delta^2}{(1 - \delta)^2} \sqrt{K} \bar{\sigma}_{\max, *} \sigma^2 \right) \, , \\
        & \le \frac{\eta K}{100} \frac{ \max_k N_k}{N} E_0 \bar{\sigma}_{\min, *}^{2} \, .
        \label{equ: T3 final}
    \end{align}
    Combining Equation~\eqref{equ: T1 final}, ~\eqref{equ: T2 final}, and~\eqref{equ: T3 final}, we have
    \begin{align}
        R_1 \le 4 \eta \left( \cC_2 \frac{\sqrt{d+c}}{\sqrt{N}} \right)
        \left( (\sigma^2 + 1) \left( \sqrt{c} + \frac{1}{\sqrt{10}} \right)^2  \right)
        + \frac{11 \eta K}{50} \frac{ \max_k N_k}{N} E_0 \bar{\sigma}_{\min, *}^{2} \, .
    \end{align}
    Then consider $R_2$, because condition $\cE_3$ holds, we have
    \begin{align}
        R_2 & = \frac{2 \eta}{N} \max_{\norm{\pp}_2 = 1} \pp^{T} (\hat{\mB}^{t})^{T} \mE^{t} \pp \, , \\
        & = \frac{2 \eta}{N} \max_{\norm{\pp}_2 = 1} \pp^{T} (\hat{\mB}^{t})^{T} \left(  \sum_{k=1}^{K} \left( \bar{\mOmega}_k^{T} \odot \mX^{T} \right) \left( \mZ \odot \mOmega_k \right) (\mtheta_k^{t})^{T} \right) \pp \, , \\
        & \le 2 \eta \norm{\frac{1}{N} \sum_{k=1}^{K} \left( \bar{\mOmega}_k^{T} \odot \mX^{T} \right) \left( \mZ \odot \mOmega_k \right) (\mtheta_k^{t})^{T}}_2 \, , \\
        & \le 2 \eta \cC_1 \sigma^2 \frac{\sqrt{d+c}}{\sqrt{N}} \left( \sqrt{c} + \frac{1}{10} \right) \, .
        \label{equ: R2 final}
    \end{align}
    Then consider $R_3$, when condition $\cE_3, \cE_4, \cE_5$, and $\cE_6$ hold, we have
    \begin{align}
        R_3 & = \frac{\eta^2}{N^2} \max_{\norm{\pp}_2 = 1} \pp^{T} (\mE^{t})^{T} \mS^{t} \pp \, , \\
        & \le \frac{2 \eta^2}{N} 
        \left( \cC_1 \sigma^2 \frac{\sqrt{d+c}}{\sqrt{N}} \left( \sqrt{c} + \frac{1}{10} \right) \right)
        \left( \cC_2 \frac{\sqrt{d+c}}{\sqrt{N}} + 1 \right)\left( (\sigma^2 + 1) \left( \sqrt{c} + \frac{1}{\sqrt{10}} \right)^2  \right) \, , \\
        & \le \cC \frac{2 \eta^2 \sqrt{d+c}}{N^{3/2}} (\sigma^2 +1)^2 \left( \sqrt{c} + \frac{1}{\sqrt{10}} \right)^3  \, .
        \label{equ: R3 final}
    \end{align}
    The last equation holds when $N$ satisfy $N \ge \frac{K^2}{d + c}$.
    Then combine Equation~\eqref{equ: R split}, ~\eqref{equ: R1 final}, ~\eqref{equ: R2 final}, and~\eqref{equ: R3 final}, because $N \ge (\sqrt{c} + 1)(\sigma^2 + 1)$ holds, we have
    \begin{align}
        \sigma_{\min}^2 (\mR^{t+1}) 
        & \ge 1 - R_1 - R_2 - R_3 \, , \\
        & \ge 1 - 8 \eta \cC \left(\frac{\sqrt{d+c}}{\sqrt{N}} \right) \left( \sigma^2 + 1 \right) \left( \sqrt{c} + \frac{1}{\sqrt{10}} \right)^2
        - \frac{11 \eta K}{50} \frac{ \max_k N_k}{N} E_0 \bar{\sigma}_{\min, *}^{2}
        \, .
        \label{equ: R final}
    \end{align}
    Combining Equation~\eqref{equ: one step begin} and ~\eqref{equ: R final}, we have
    \begin{small}
    \begin{align}
        & dist \left( \hat{\mB}^{t+1}, \hat{\mB}^{*} \right) 
        \le dist \left( \hat{\mB}^{t}, \hat{\mB}^{*} \right) \nonumber \\
        & \left(  1 - \eta K \left(  \left(\frac{\min_k \hat{N}_k}{N} \right) - \frac{1}{4}  \left(\frac{\max_k \hat{N}_k}{N} \right) \right) \bar{\sigma}_{\min, *}^{2} E_0 
        + \eta \cC \frac{\sqrt{d+c}}{\sqrt{N}} \left( \sqrt{c} + \frac{1}{\sqrt{10}} \right)^2 
        \right) \nonumber \\
        & \sigma_{\min} (\mR^{t+1})^{-1} 
        \nonumber \\
        & + \eta \cC \frac{\sqrt{d+c}}{\sqrt{N}} \left( 
        \sigma^2 + 1 \right) \left( 2 \sqrt{c} + \frac{1}{5} \right)
        \sigma_{\min} (\mR^{t+1})^{-1} 
        \, .
    \end{align}
    \end{small}
    When $\frac{\max_k \hat{N}_k}{N}$ satisfy
    \begin{align}
        \frac{\max_k \hat{N}_k}{N} \ge \frac{200}{7} \cC \frac{\sqrt{d+c}}{\sqrt{N}} \frac{(\sigma^2 + 1)(\sqrt{c} + 1)^2}{K E_0 \bar{\sigma}_{\min, *}^2} \, ,
    \end{align}
    we will have
    \begin{align}
        \sigma_{\min}^2 (\mR^{t+1}) & \ge 1 - \frac{1}{2} \eta K \left(\frac{\max_k \hat{N}_k}{N} \right) \bar{\sigma}_{\min, *}^2 E_0 \, , \\
        \eta \cC \frac{\sqrt{d+c}}{\sqrt{N}} \left( \sqrt{c} + \frac{1}{\sqrt{10}} \right)^2 & \le
        \frac{7}{200} \eta K \left(\frac{\max_k \hat{N}_k}{N} \right) \bar{\sigma}_{\min, *}^2 E_0 \, .
    \end{align}
    Then we have
    \begin{small}
    \begin{align}
    \textstyle
        & dist \left( \hat{\mB}^{t+1}, \hat{\mB}^{*} \right) 
        \le dist \left( \hat{\mB}^{t}, \hat{\mB}^{*} \right) \nonumber \\
        & \left(  1 - \eta K \left(  \left(\frac{\min_k \hat{N}_k}{N} \right) - \frac{57}{200}  \left(\frac{\max_k \hat{N}_k}{N} \right) \right) \bar{\sigma}_{\min, *}^{2} E_0 \right) \nonumber \\
        & \left( 1 - \frac{1}{2} \eta K \left(\frac{\max_k \hat{N}_k}{N} \right) \bar{\sigma}_{\min, *}^2 E_0 \right)^{-1/2} \nonumber \\
        & + \left(\frac{7}{100} \eta K \left(\frac{\max_k \hat{N}_k}{N} \right) \bar{\sigma}_{\min, *}^2 E_0 \right)
        \left( 1 - \frac{1}{2} \eta K \left(\frac{\max_k \hat{N}_k}{N} \right) \bar{\sigma}_{\min, *}^2 E_0 \right)^{-1/2} \, .
    \end{align}
    \end{small}

\end{proof}

\begin{remark}
    Different from previous studies that assuming clusters to have the same number of data points~\citep{collins2021exploiting,tziotis2022straggler}, we analysed the impact of imbalanced sampled numbers among clusters in Theorem~\ref{the-convergence-appendix}. We can see that \algfednew's convergence is strongly influenced by $\max_k \hat{N}_k$ and $\min_k \hat{N}_k$. In other words, \algfednew converges faster as $\frac{\min_k \hat{N}_k}{\max_k \hat{N}_k}$ increases. This suggests that \algfednew performs better when the number of samples is evenly distributed among all $K$ underlying clusters. In contrary, if the number of samples is extremely imbalanced, the \algfednew may hard to converge to optimum.
    \looseness=-1 
\end{remark}

%% file: iclr2025_hcfl.bbl
\begin{thebibliography}{53}
\providecommand{\natexlab}[1]{#1}
\providecommand{\url}[1]{\texttt{#1}}
\expandafter\ifx\csname urlstyle\endcsname\relax
  \providecommand{\doi}[1]{doi: #1}\else
  \providecommand{\doi}{doi: \begingroup \urlstyle{rm}\Url}\fi

\bibitem[Bao et~al.(2023)Bao, Wang, Wu, and He]{bao2023optimizing}
Wenxuan Bao, Haohan Wang, Jun Wu, and Jingrui He.
\newblock Optimizing the collaboration structure in cross-silo federated
  learning.
\newblock 2023.

\bibitem[Briggs et~al.(2020)Briggs, Fan, and Andras]{briggs2020federated}
Christopher Briggs, Zhong Fan, and Peter Andras.
\newblock Federated learning with hierarchical clustering of local updates to
  improve training on non-iid data.
\newblock In \emph{2020 International Joint Conference on Neural Networks
  (IJCNN)}, pp.\  1--9. IEEE, 2020.

\bibitem[Collins et~al.(2021)Collins, Hassani, Mokhtari, and
  Shakkottai]{collins2021exploiting}
Liam Collins, Hamed Hassani, Aryan Mokhtari, and Sanjay Shakkottai.
\newblock Exploiting shared representations for personalized federated
  learning.
\newblock In \emph{International Conference on Machine Learning}, pp.\
  2089--2099. PMLR, 2021.

\bibitem[Diao et~al.(2024)Diao, Li, and He]{diao2024exploiting}
Yiqun Diao, Qinbin Li, and Bingsheng He.
\newblock Exploiting label skews in federated learning with model
  concatenation.
\newblock In \emph{Proceedings of the AAAI Conference on Artificial
  Intelligence}, volume~38, pp.\  11784--11792, 2024.

\bibitem[Ding et~al.(2023)Ding, Li, Xu, Guo, Ding, and Wu]{ding2023horizontal}
Shifei Ding, Chao Li, Xiao Xu, Lili Guo, Ling Ding, and Xindong Wu.
\newblock Horizontal federated density peaks clustering.
\newblock \emph{IEEE Transactions on Neural Networks and Learning Systems},
  2023.

\bibitem[Duan et~al.(2021{\natexlab{a}})Duan, Liu, Ji, Liu, Liang, Chen, and
  Tan]{duan2021fedgroup}
Moming Duan, Duo Liu, Xinyuan Ji, Renping Liu, Liang Liang, Xianzhang Chen, and
  Yujuan Tan.
\newblock {FedGroup}: Efficient federated learning via decomposed
  similarity-based clustering.
\newblock In \emph{2021 IEEE Intl Conf on Parallel \& Distributed Processing
  with Applications, Big Data \& Cloud Computing, Sustainable Computing \&
  Communications, Social Computing \& Networking
  (ISPA/BDCloud/SocialCom/SustainCom)}, pp.\  228--237. IEEE,
  2021{\natexlab{a}}.

\bibitem[Duan et~al.(2021{\natexlab{b}})Duan, Liu, Ji, Wu, Liang, Chen, Tan,
  and Ren]{duan2021flexible}
Moming Duan, Duo Liu, Xinyuan Ji, Yu~Wu, Liang Liang, Xianzhang Chen, Yujuan
  Tan, and Ao~Ren.
\newblock Flexible clustered federated learning for client-level data
  distribution shift.
\newblock \emph{IEEE Transactions on Parallel and Distributed Systems},
  33\penalty0 (11):\penalty0 2661--2674, 2021{\natexlab{b}}.

\bibitem[Fang \& Ye(2022)Fang and Ye]{fang2022robust}
Xiuwen Fang and Mang Ye.
\newblock Robust federated learning with noisy and heterogeneous clients.
\newblock In \emph{Proceedings of the IEEE/CVF Conference on Computer Vision
  and Pattern Recognition}, pp.\  10072--10081, 2022.

\bibitem[Gan et~al.(2021)Gan, Mathur, Isopoussu, Kawsar, Berthouze, and
  Lane]{gan2021fruda}
Shaoduo Gan, Akhil Mathur, Anton Isopoussu, Fahim Kawsar, Nadia Berthouze, and
  Nicholas Lane.
\newblock {FRuDA}: Framework for distributed adversarial domain adaptation.
\newblock \emph{IEEE Transactions on Parallel and Distributed Systems}, 2021.

\bibitem[Ghosh et~al.(2020)Ghosh, Chung, Yin, and
  Ramchandran]{ghosh2020efficient}
Avishek Ghosh, Jichan Chung, Dong Yin, and Kannan Ramchandran.
\newblock An efficient framework for clustered federated learning.
\newblock \emph{Advances in Neural Information Processing Systems},
  33:\penalty0 19586--19597, 2020.

\bibitem[Guo et~al.(2021)Guo, Lin, and Tang]{guo2021towards}
Yongxin Guo, Tao Lin, and Xiaoying Tang.
\newblock Towards federated learning on time-evolving heterogeneous data.
\newblock \emph{arXiv preprint arXiv:2112.13246}, 2021.

\bibitem[Guo et~al.(2023{\natexlab{a}})Guo, Tang, and Lin]{guo2022fedaug}
Yongxin Guo, Xiaoying Tang, and Tao Lin.
\newblock {FedBR}: Improving federated learning on heterogeneous data via local
  learning bias reduction.
\newblock 2023{\natexlab{a}}.

\bibitem[Guo et~al.(2023{\natexlab{b}})Guo, Tang, and Lin]{guo2023fedconceptem}
Yongxin Guo, Xiaoying Tang, and Tao Lin.
\newblock {FedRC}: Tackling diverse distribution shifts challenge in federated
  learning by robust clustering.
\newblock \emph{arXiv preprint arXiv:2301.12379}, 2023{\natexlab{b}}.

\bibitem[Hendrycks \& Dietterich(2019)Hendrycks and
  Dietterich]{hendrycks2019benchmarking}
Dan Hendrycks and Thomas Dietterich.
\newblock Benchmarking neural network robustness to common corruptions and
  perturbations.
\newblock \emph{arXiv preprint arXiv:1903.12261}, 2019.

\bibitem[Hsu et~al.(2019)Hsu, Qi, and Brown]{hsu2019measuring}
Tzu-Ming~Harry Hsu, Hang Qi, and Matthew Brown.
\newblock Measuring the effects of non-identical data distribution for
  federated visual classification.
\newblock \emph{arXiv preprint arXiv:1909.06335}, 2019.

\bibitem[Jain et~al.(2013)Jain, Netrapalli, and Sanghavi]{jain2013low}
Prateek Jain, Praneeth Netrapalli, and Sujay Sanghavi.
\newblock Low-rank matrix completion using alternating minimization.
\newblock In \emph{Proceedings of the forty-fifth annual ACM symposium on
  Theory of computing}, pp.\  665--674, 2013.

\bibitem[Jiang \& Lin(2023)Jiang and Lin]{jiang2023test}
Liangze Jiang and Tao Lin.
\newblock Test-time robust personalization for federated learning.
\newblock In \emph{International Conference on Learning Representations}, 2023.

\bibitem[Jothimurugesan et~al.(2023)Jothimurugesan, Hsieh, Wang, Joshi, and
  Gibbons]{jothimurugesan2022federated}
Ellango Jothimurugesan, Kevin Hsieh, Jianyu Wang, Gauri Joshi, and Phillip~B
  Gibbons.
\newblock Federated learning under distributed concept drift.
\newblock pp.\  5834--5853, 2023.

\bibitem[Karimireddy et~al.(2019)Karimireddy, Kale, Mohri, Reddi, Stich, and
  Suresh]{karimireddyscaffold2019}
Sai~Praneeth Karimireddy, Satyen Kale, Mehryar Mohri, Sashank~J. Reddi,
  Sebastian~U. Stich, and Ananda~Theertha Suresh.
\newblock {SCAFFOLD}: Stochastic controlled averaging for federated learning,
  2019.
\newblock URL \url{https://arxiv.org/abs/1910.06378}.

\bibitem[Karimireddy et~al.(2020)Karimireddy, Kale, Mohri, Reddi, Stich, and
  Suresh]{karimireddy2020scaffold}
Sai~Praneeth Karimireddy, Satyen Kale, Mehryar Mohri, Sashank Reddi, Sebastian
  Stich, and Ananda~Theertha Suresh.
\newblock {SCAFFOLD}: Stochastic controlled averaging for federated learning.
\newblock In \emph{International Conference on Machine Learning}, pp.\
  5132--5143. PMLR, 2020.

\bibitem[Ke et~al.(2022)Ke, Huang, and Liu]{ke2022quantifying}
Shuqi Ke, Chao Huang, and Xin Liu.
\newblock Quantifying the impact of label noise on federated learning.
\newblock \emph{arXiv preprint arXiv:2211.07816}, 2022.

\bibitem[Li et~al.(2021)Li, He, and Song]{li2021model}
Qinbin Li, Bingsheng He, and Dawn Song.
\newblock Model-contrastive federated learning.
\newblock In \emph{Proceedings of the IEEE/CVF Conference on Computer Vision
  and Pattern Recognition}, 2021.

\bibitem[Li et~al.(2018)Li, Sahu, Zaheer, Sanjabi, Talwalkar, and
  Smith]{li2018federated}
Tian Li, Anit~Kumar Sahu, Manzil Zaheer, Maziar Sanjabi, Ameet Talwalkar, and
  Virginia Smith.
\newblock Federated optimization in heterogeneous networks.
\newblock \emph{arXiv preprint arXiv:1812.06127}, 2018.

\bibitem[Lin et~al.(2020)Lin, Stich, Patel, and Jaggi]{lin2020dont}
Tao Lin, Sebastian~U. Stich, Kumar~Kshitij Patel, and Martin Jaggi.
\newblock Don't use large mini-batches, use local sgd.
\newblock In \emph{International Conference on Learning Representations}, 2020.
\newblock URL \url{https://openreview.net/forum?id=B1eyO1BFPr}.

\bibitem[Long et~al.(2023)Long, Xie, Shen, Zhou, Wang, and
  Jiang]{long2023multi}
Guodong Long, Ming Xie, Tao Shen, Tianyi Zhou, Xianzhi Wang, and Jing Jiang.
\newblock Multi-center federated learning: clients clustering for better
  personalization.
\newblock \emph{World Wide Web}, 26\penalty0 (1):\penalty0 481--500, 2023.

\bibitem[Ma et~al.(2022)Ma, Long, Zhou, Jiang, and Zhang]{ma2022convergence}
Jie Ma, Guodong Long, Tianyi Zhou, Jing Jiang, and Chengqi Zhang.
\newblock On the convergence of clustered federated learning.
\newblock \emph{arXiv preprint arXiv:2202.06187}, 2022.

\bibitem[Marfoq et~al.(2021)Marfoq, Neglia, Bellet, Kameni, and
  Vidal]{marfoq2021federated}
Othmane Marfoq, Giovanni Neglia, Aur{\'e}lien Bellet, Laetitia Kameni, and
  Richard Vidal.
\newblock Federated multi-task learning under a mixture of distributions.
\newblock \emph{Advances in Neural Information Processing Systems},
  34:\penalty0 15434--15447, 2021.

\bibitem[McMahan et~al.(2016)McMahan, Moore, Ramage, Hampson, and
  Arcas]{mcmahan2016communication}
H.~Brendan McMahan, Eider Moore, Daniel Ramage, Seth Hampson, and Blaise
  Agüera~y Arcas.
\newblock Communication-efficient learning of deep networks from decentralized
  data.
\newblock 2016.
\newblock \doi{10.48550/ARXIV.1602.05629}.
\newblock URL \url{https://arxiv.org/abs/1602.05629}.

\bibitem[Peng et~al.(2019)Peng, Huang, Zhu, and Saenko]{peng2019federated}
Xingchao Peng, Zijun Huang, Yizhe Zhu, and Kate Saenko.
\newblock Federated adversarial domain adaptation, 2019.
\newblock URL \url{https://arxiv.org/abs/1911.02054}.

\bibitem[Qiao et~al.(2024)Qiao, Ding, and Fan]{qiao2024federated}
Dong Qiao, Chris Ding, and Jicong Fan.
\newblock Federated spectral clustering via secure similarity reconstruction.
\newblock \emph{Advances in Neural Information Processing Systems}, 36, 2024.

\bibitem[Reddi et~al.(2021)Reddi, Charles, Zaheer, Garrett, Rush,
  Kone{\v{c}}n{\'y}, Kumar, and McMahan]{reddi2021adaptive}
Sashank~J. Reddi, Zachary Charles, Manzil Zaheer, Zachary Garrett, Keith Rush,
  Jakub Kone{\v{c}}n{\'y}, Sanjiv Kumar, and Hugh~Brendan McMahan.
\newblock Adaptive federated optimization.
\newblock In \emph{International Conference on Learning Representations}, 2021.
\newblock URL \url{https://openreview.net/forum?id=LkFG3lB13U5}.

\bibitem[Reisser et~al.(2021)Reisser, Louizos, Gavves, and
  Welling]{reisser2021federated}
Matthias Reisser, Christos Louizos, Efstratios Gavves, and Max Welling.
\newblock Federated mixture of experts.
\newblock \emph{arXiv preprint arXiv:2107.06724}, 2021.

\bibitem[Ruan \& Joe-Wong(2022)Ruan and Joe-Wong]{ruan2022fedsoft}
Yichen Ruan and Carlee Joe-Wong.
\newblock {FedSoft}: Soft clustered federated learning with proximal local
  updating.
\newblock In \emph{Proceedings of the AAAI Conference on Artificial
  Intelligence}, volume~36, pp.\  8124--8131, 2022.

\bibitem[Sandler et~al.(2018)Sandler, Howard, Zhu, Zhmoginov, and
  Chen]{sandler2018mobilenetv2}
Mark Sandler, Andrew Howard, Menglong Zhu, Andrey Zhmoginov, and Liang-Chieh
  Chen.
\newblock {MobileNetV2}: Inverted residuals and linear bottlenecks.
\newblock In \emph{Proceedings of the IEEE conference on computer vision and
  pattern recognition}, pp.\  4510--4520, 2018.

\bibitem[Sattler et~al.(2020{\natexlab{a}})Sattler, M{\"u}ller, and
  Samek]{sattler2020clustered}
Felix Sattler, Klaus-Robert M{\"u}ller, and Wojciech Samek.
\newblock Clustered federated learning: Model-agnostic distributed multitask
  optimization under privacy constraints.
\newblock \emph{IEEE transactions on neural networks and learning systems},
  32\penalty0 (8):\penalty0 3710--3722, 2020{\natexlab{a}}.

\bibitem[Sattler et~al.(2020{\natexlab{b}})Sattler, M{\"u}ller, Wiegand, and
  Samek]{sattler2020byzantine}
Felix Sattler, Klaus-Robert M{\"u}ller, Thomas Wiegand, and Wojciech Samek.
\newblock On the byzantine robustness of clustered federated learning.
\newblock In \emph{ICASSP 2020-2020 IEEE International Conference on Acoustics,
  Speech and Signal Processing (ICASSP)}, pp.\  8861--8865. IEEE,
  2020{\natexlab{b}}.

\bibitem[Shen et~al.(2021)Shen, Du, Zhao, Zhang, Ji, and Gao]{shen2021fedmm}
Yan Shen, Jian Du, Han Zhao, Benyu Zhang, Zhanghexuan Ji, and Mingchen Gao.
\newblock {FedMM}: Saddle point optimization for federated adversarial domain
  adaptation.
\newblock \emph{arXiv preprint arXiv:2110.08477}, 2021.

\bibitem[Stallmann \& Wilbik(2022)Stallmann and Wilbik]{stallmann2022towards}
Morris Stallmann and Anna Wilbik.
\newblock Towards federated clustering: A federated fuzzy $ c $-means algorithm
  (ffcm).
\newblock \emph{arXiv preprint arXiv:2201.07316}, 2022.

\bibitem[Sun et~al.(2022)Sun, Chong, and Hideya]{sun2022multi}
Yuwei Sun, Ng~Chong, and Ochiai Hideya.
\newblock Multi-source domain adaptation based on federated knowledge
  alignment.
\newblock \emph{arXiv preprint arXiv:2203.11635}, 2022.

\bibitem[Tang et~al.(2022)Tang, Zhang, Shi, He, Han, and Chu]{tang2022virtual}
Zhenheng Tang, Yonggang Zhang, Shaohuai Shi, Xin He, Bo~Han, and Xiaowen Chu.
\newblock Virtual homogeneity learning: Defending against data heterogeneity in
  federated learning.
\newblock \emph{International Conference on Machine Learning}, pp.\
  21111--21132, 2022.

\bibitem[Tziotis et~al.(2022)Tziotis, Shen, Pedarsani, Hassani, and
  Mokhtari]{tziotis2022straggler}
Isidoros Tziotis, Zebang Shen, Ramtin Pedarsani, Hamed Hassani, and Aryan
  Mokhtari.
\newblock Straggler-resilient personalized federated learning.
\newblock \emph{arXiv preprint arXiv:2206.02078}, 2022.

\bibitem[Vahidian et~al.(2023)Vahidian, Morafah, Wang, Kungurtsev, Chen, Shah,
  and Lin]{vahidian2023efficient}
Saeed Vahidian, Mahdi Morafah, Weijia Wang, Vyacheslav Kungurtsev, Chen Chen,
  Mubarak Shah, and Bill Lin.
\newblock Efficient distribution similarity identification in clustered
  federated learning via principal angles between client data subspaces.
\newblock In \emph{Proceedings of the AAAI Conference on Artificial
  Intelligence}, volume~37, pp.\  10043--10052, 2023.

\bibitem[Vershynin(2018)]{vershynin2018high}
Roman Vershynin.
\newblock \emph{High-dimensional probability: An introduction with applications
  in data science}, volume~47.
\newblock Cambridge university press, 2018.

\bibitem[Wang et~al.(2022{\natexlab{a}})Wang, Li, Wu, Zhang, Zhou, and
  Wei]{wang2022framework}
Bin Wang, Gang Li, Chao Wu, WeiShan Zhang, Jiehan Zhou, and Ye~Wei.
\newblock A framework for self-supervised federated domain adaptation.
\newblock \emph{EURASIP Journal on Wireless Communications and Networking},
  2022\penalty0 (1):\penalty0 1--17, 2022{\natexlab{a}}.

\bibitem[Wang et~al.(2020)Wang, Yurochkin, Sun, Papailiopoulos, and
  Khazaeni]{wang2020federated}
Hongyi Wang, Mikhail Yurochkin, Yuekai Sun, Dimitris Papailiopoulos, and
  Yasaman Khazaeni.
\newblock Federated learning with matched averaging.
\newblock \emph{arXiv preprint arXiv:2002.06440}, 2020.

\bibitem[Wang et~al.(2022{\natexlab{b}})Wang, Xu, Liu, Xu, Huang, and
  Zhao]{wang2022accelerating}
Zhiyuan Wang, Hongli Xu, Jianchun Liu, Yang Xu, He~Huang, and Yangming Zhao.
\newblock Accelerating federated learning with cluster construction and
  hierarchical aggregation.
\newblock \emph{IEEE Transactions on Mobile Computing}, 2022{\natexlab{b}}.

\bibitem[Wei \& Huang(2023)Wei and Huang]{wei2023edge}
Xiao-Xiang Wei and Hua Huang.
\newblock Edge devices clustering for federated visual classification: A
  feature norm based framework.
\newblock \emph{IEEE Transactions on Image Processing}, 32:\penalty0 995--1010,
  2023.

\bibitem[Wu et~al.(2022)Wu, Li, Charles, Xiao, Liu, Xu, and
  Smith]{wu2022motley}
Shanshan Wu, Tian Li, Zachary Charles, Yu~Xiao, Ziyu Liu, Zheng Xu, and
  Virginia Smith.
\newblock Motley: Benchmarking heterogeneity and personalization in federated
  learning.
\newblock \emph{arXiv preprint arXiv:2206.09262}, 2022.

\bibitem[Wu et~al.(2023)Wu, Zhang, Yu, Liu, Gu, Zhou, Chen, and
  Cheng]{wu2023personalized}
Yue Wu, Shuaicheng Zhang, Wenchao Yu, Yanchi Liu, Quanquan Gu, Dawei Zhou,
  Haifeng Chen, and Wei Cheng.
\newblock Personalized federated learning under mixture of distributions.
\newblock 2023.

\bibitem[Yan et~al.(2023)Yan, Tong, and Wang]{yan2023clustered}
Yihan Yan, Xiaojun Tong, and Shen Wang.
\newblock Clustered federated learning in heterogeneous environment.
\newblock \emph{IEEE Transactions on Neural Networks and Learning Systems},
  2023.

\bibitem[Yoshida et~al.(2019)Yoshida, Nishio, Morikura, Yamamoto, and
  Yonetani]{yoshida2019hybrid}
Naoya Yoshida, Takayuki Nishio, Masahiro Morikura, Koji Yamamoto, and Ryo
  Yonetani.
\newblock {Hybrid-FL}: Cooperative learning mechanism using non-iid data in
  wireless networks.
\newblock \emph{arXiv preprint arXiv:1905.07210}, 2019.

\bibitem[Zeng et~al.(2023)Zeng, Hu, Liu, Yu, Wang, and Xu]{zeng2023stochastic}
Dun Zeng, Xiangjing Hu, Shiyu Liu, Yue Yu, Qifan Wang, and Zenglin Xu.
\newblock Stochastic clustered federated learning.
\newblock \emph{arXiv preprint arXiv:2303.00897}, 2023.

\bibitem[Zhao et~al.(2020)Zhao, Huang, Sai, and Wu]{zhao2020cluster}
Fengpan Zhao, Yan Huang, Akshita Maradapu Vera~Venkata Sai, and Yubao Wu.
\newblock A cluster-based solution to achieve fairness in federated learning.
\newblock In \emph{2020 IEEE Intl Conf on Parallel \& Distributed Processing
  with Applications, Big Data \& Cloud Computing, Sustainable Computing \&
  Communications, Social Computing \& Networking
  (ISPA/BDCloud/SocialCom/SustainCom)}, pp.\  875--882. IEEE, 2020.

\end{thebibliography}
